\pdfoutput=1
\documentclass[11pt]{article}

\usepackage[letterpaper,margin=1in]{geometry}
\usepackage[T1]{fontenc}
\usepackage{lmodern}
\usepackage[numbers,sort&compress]{natbib}

\usepackage[utf8]{inputenc}
\usepackage{hyperref}
\usepackage{url}
\usepackage{booktabs}
\usepackage{amsfonts}
\usepackage{amssymb}
\usepackage{mathtools}
\usepackage{amsthm}
\usepackage{microtype}
\usepackage{xcolor}
\usepackage{graphicx}
\usepackage{adjustbox}
\usepackage{algorithm}
\usepackage{algorithmic}
\usepackage{float}
\usepackage{cleveref}
\hypersetup{
  hidelinks,
  hypertexnames=false,
  pdftitle={NBCLT: Parameter-Agnostic Norm-Ball Clipped Local Training},
  pdfauthor={Linxuan Pan, Shucheng Yang, Junchi Yang}
}
\allowdisplaybreaks[1]

\newtheorem{theorem}{Theorem}
\newtheorem{lemma}[theorem]{Lemma}
\newtheorem{assumption}[theorem]{Assumption}
\newtheorem{proposition}[theorem]{Proposition}

\newcommand{\grad}{\nabla}

\newcommand{\Proj}{\Pi_{\mathcal{D}}}
\newcommand{\vecop}{\operatorname{vec}}
\newcommand{\unvecop}{\operatorname{unvec}}

\DeclareMathOperator*{\argmin}{arg\,min}

\title{\LARGE\bfseries NBCLT: Parameter-Agnostic Norm-Ball\\Clipped Local Training}
\author{%
  \begin{tabular}{ccc}
    Linxuan Pan & Shucheng Yang & Junchi Yang \\[3pt]
    \small\href{mailto:panlinxuan@link.cuhk.edu.cn}{\texttt{panlinxuan@link.cuhk.edu.cn}} &
    \small\href{mailto:yangsc@mail.sim.ac.cn}{\texttt{yangsc@mail.sim.ac.cn}} &
    \small\href{mailto:yangjunchi@cuhk.edu.cn}{\texttt{yangjunchi@cuhk.edu.cn}}
  \end{tabular}%
}
\date{}

\begin{document}
\maketitle

\begin{abstract}%
Distributed learning with multiple local steps reduces communication but aggravates drift, and its practical performance often depends on carefully tuned hyperparameters under heterogeneous data and communication constraints. At the same time, recent Muon-type optimizers suggest that constraining updates within suitable norm balls provides a simple way to encode geometric priors into optimization. In this work, we propose NBCLT (\textbf{N}orm-\textbf{B}all \textbf{C}lipped \textbf{L}ocal \textbf{T}raining), a unified local-training framework for both federated learning (FL) and decentralized learning (DL). NBCLT combines a predict-and-correct local update rule with Euclidean projections onto flexible norm balls, thereby performing norm-ball clipping under flexible update geometries while retaining a simple implementation. In FL, its problem-parameter-free regime achieves $\mathcal{O}(\epsilon^{-2})$ communication complexity and $\mathcal{O}(\epsilon^{-4})$ total stochastic-gradient complexity without assuming bounded heterogeneity, when the per-round local-computation budget is scaled appropriately with the horizon. In DL, NBCLT is compatible with both static and time-varying mixing and preserves these accuracy exponents under the analogous computation scaling, up to fixed network constants and an additive initialization cost.
\end{abstract}

\section{Introduction}
\label{sec:intro}
Machine learning has achieved significant progress in domains such as computer vision, natural language processing, and scientific discovery~\citep{vaswani2017attention,brown2020language,ramesh2022hierarchical,rombach2022high,jumper2021highly}. To address requirements arising from data privacy, system topology, and communication constraints, several distributed learning paradigms have been developed. Federated Learning (FL) studies training over decentralized and typically non-IID data under a central coordinator~\citep{konevcny2016federated,mcmahan2017communication}, while Decentralized Learning (DL) removes the coordinator and relies on peer-to-peer communication over a network~\citep{lian2017can}. Throughout this paper, we use \emph{distributed learning} to encompass both settings.

Formally, a distributed system comprises $M$ nodes (e.g., clients, workers, or organizations). Node $i\in[M]$ holds a local objective $f_i(x)=\mathbb E_{\xi_i\sim\mathcal D_i}[\ell(x;\xi_i)]$, and the global model is trained by solving
\begin{equation}
\label{eq:obj}
\min_{x\in\mathcal X}\left[f(x):=\frac{1}{M}\sum_{i=1}^M f_i(x)\right].
\end{equation}

In distributed optimization, two challenges are fundamental. First, communication is often the dominant systems bottleneck, which makes multiple local updates between synchronization rounds highly desirable in both FL~\citep{mcmahan2017communication,stich2018local} and DL~\citep{hashemi2020benefits}. Second, local training amplifies drift caused by data heterogeneity or imperfect consensus. To reduce this drift, FL algorithms use control variates or momentum corrections, as in SCAFFOLD~\citep{karimireddy2020scaffold} and momentum-based FL methods~\citep{cheng2023momentum}, while DL methods rely on mechanisms such as gradient tracking or Exact Diffusion~\citep{alghunaim2024local,liu2025decentralized}. Yet their guarantees typically require well-tuned hyperparameters that depend on inaccessible quantities such as smoothness, noise, or heterogeneity levels.

\begin{table}[tbp]
\centering
\caption{Comparison in FL (nonconvex). Communication complexity counts server-client rounds; sample complexity counts total stochastic gradients across participating clients. The NBCLT entry uses $NK=\Theta(T)$ with $NK\le T$ and fixed problem/system constants, not fixed $N,K$ as accuracy increases.}
\label{tab:fl_comparison}
\small
\setlength{\tabcolsep}{4pt}
\renewcommand{\arraystretch}{1.2}
\begin{adjustbox}{max width=\textwidth}
\begin{tabular}{lccccc}
\toprule
Method & Unbounded het. & Comm. compl. & Sample compl. & Prob.-param.-free & Local steps \\
\midrule
\textbf{NBCLT-FL (ours)}
& $\checkmark$
& $\mathcal{O}(\epsilon^{-2})$
& $\mathcal{O}(\epsilon^{-4})$
& $\checkmark$ & $\checkmark$ \\
PAdaMFed~\citep{yan2025problem}
& $\checkmark$
& $\mathcal{O}(\epsilon^{-3})$
& $\mathcal{O}(\epsilon^{-4})$
& $\checkmark$ & $\checkmark$ \\
AdaMasFL~\citep{yanmomentum}
& $\checkmark$
& $\mathcal{O}(\epsilon^{-8/3})$
& $\mathcal{O}(\epsilon^{-4})$
& $\checkmark$ & $\checkmark$ \\
FedMuon~\citep{takezawa2025fedmuon}
& $\checkmark$
& $\mathcal{O}(\epsilon^{-3})$
& $\mathcal{O}(\epsilon^{-4})$
& $\times$ & $\checkmark$ \\
FedMuon~\citep{liu2025fedmuon}
& $\times$
& $\mathcal{O}(\epsilon^{-3})$
& $\mathcal{O}(\epsilon^{-4})$
& $\checkmark$ & $\checkmark$ \\
Diloco~\citep{douillard2024diloco}
& --
& -- (empirical)
& -- (empirical)
& $\times$ & $\checkmark$ \\
\bottomrule
\end{tabular}
\end{adjustbox}
\end{table}

\begin{table}[tbp]
\centering
\caption{Comparison in DL (nonconvex). ``Mixing'' indicates whether the theory covers static doubly-stochastic matrices or time-varying mixing matrices. The NBCLT entry uses $MK=\Theta(T)$ with $MK\le T$ and fixed problem/network constants; exact initialization has an additional network-dependent communication cost.}
\label{tab:dl_comparison}
\small
\setlength{\tabcolsep}{3pt}
\renewcommand{\arraystretch}{1.2}
\begin{adjustbox}{max width=\textwidth}
\begin{tabular}{lcccccc}
\toprule
Method & Mixing & Unbounded het. & Comm. compl. & Sample compl. & Prob.-param.-free & Local steps \\
\midrule
\textbf{NBCLT-DL (ours)} & Both & $\checkmark$ & $\mathcal{O}(\epsilon^{-2})$ & $\mathcal{O}(\epsilon^{-4})$ & $\checkmark$ & $\checkmark$ \\
D-NASA~\citep{li2024problem} & static & $\checkmark$ & $\mathcal{O}(\epsilon^{-4})$ & $\mathcal{O}(\epsilon^{-4})$ & $\checkmark$ & $\times$ \\
LED~\citep{alghunaim2024local} & static & $\checkmark$ & $\mathcal{O}(\epsilon^{-2})$ & $\mathcal{O}(\epsilon^{-4})$ & $\times$ & $\checkmark$ \\
GT-FT~\citep{nguyen2025graphs} & dynamic & $\checkmark$ & $\mathcal{O}(\epsilon^{-4})$ & $\mathcal{O}(\epsilon^{-4})$ & $\times$ & $\times$ \\
\bottomrule
\end{tabular}
\end{adjustbox}
\end{table}

Recent progress on problem-parameter-free optimization has reduced the burden of manual tuning, but a gap remains in both settings. In FL, methods such as PAdaMFed~\citep{yan2025problem} and AdaMasFL~\citep{yanmomentum} avoid prior knowledge of problem constants, yet their communication guarantees still fall short of the $\mathcal{O}(\epsilon^{-2})$ benchmark achieved by well-tuned local training methods. In DL, the problem-parameter-free method D-NASA~\citep{li2024problem} does not support multiple local updates, which likewise leads to suboptimal communication efficiency. This raises a first question:
\begin{quote}
\textit{Can a problem-parameter-free local-training algorithm match the $\mathcal{O}(1/\epsilon^2)$ communication complexity of well-tuned methods in both FL and DL, without sacrificing the benefits of multiple local updates?}
\end{quote}

In addition to tuning the update scale, local training also depends on the geometry used to constrain local updates. We view a local update as lying in a norm budget $\mathcal D=\{\Delta:\|\Delta\|^g\le \eta\}$, where the radius $\eta$ sets the overall scale of motion and the norm $\|\cdot\|^g$ determines the shape of admissible updates. This separates how far the algorithm moves from the geometry under which admissible updates are measured; Section~\ref{sec:prelim} makes this distinction concrete through a quadratic example.

This viewpoint is particularly relevant in distributed learning, where many clients may have highly heterogeneous local loss landscapes. Recent Muon-type optimizers make this geometric perspective explicit by constraining updates in carefully chosen norm balls~\citep{jordan2024muon,pethick2025training,boreiko2025towards,liu2025muon,shah2025practical}. Several recent works have begun integrating Muon into federated learning, but the current analyses remain limited to $\mathcal{O}(\epsilon^{-3})$ communication complexity~\citep{zhang2025provable,takezawa2025fedmuon,liu2025fedmuon}. This leads to a second question:
\begin{quote}
\textit{Can we design a distributed local-training algorithm that exploits flexible update geometry while retaining the communication and sample-complexity guarantees expected from efficient local methods?}
\end{quote}

Motivated by these two questions, we propose NBCLT (\textbf{N}orm-\textbf{B}all \textbf{C}lipped \textbf{L}ocal \textbf{T}raining). Our main contributions are as follows:
\begin{itemize}
    \item NBCLT provides a unified framework for both federated learning (FL) and decentralized learning (DL). It introduces two key local-update components: (i) Euclidean projection of tentative update variables onto flexible norm balls, i.e., norm-ball clipping under the chosen update geometry, and (ii) a novel \emph{predict-and-correct} mechanism, conceptually adjacent to extragradient and optimistic ideas~\citep{nemirovski2004prox,rakhlin2013optimization}, that stabilizes clipped local training. For spectral-norm matrix blocks, we instantiate this oracle by PolyGram-RCClip, a Gram-polynomial clipping projector.

    \item In FL, we prove a rate of order $T^{-1/2}+(NKT)^{-1/4}$ plus lower-order terms without assuming bounded heterogeneity. With $NK=\Theta(T)$ and $NK\le T$, this yields $\mathcal{O}(\epsilon^{-2})$ communication and $\mathcal{O}(\epsilon^{-4})$ total stochastic-gradient complexity for fixed problem/system constants.

    \item In DL, to the best of our knowledge, NBCLT is the first problem-parameter-free algorithm that supports multiple local updates while remaining compatible with both static and time-varying gossip matrices. Under the analogous scaling $MK=\Theta(T)$ with $MK\le T$, its accuracy exponents match those of well-tuned local-update baselines, up to fixed network constants and exact-initialization overhead.
\end{itemize}

\section{Preliminaries: Update Geometry under a Norm Budget}
\label{sec:prelim}
This section isolates the update-geometry viewpoint used throughout the paper. A norm budget $\mathcal D=\{\Delta:\|\Delta\|^g\le\eta\}$ contains two distinct design choices: the radius $\eta$, which sets the overall scale of motion, and the norm $\|\cdot\|^g$, which shapes the admissible update directions. We first illustrate this distinction on a quadratic model and then connect it to Muon-type norm-ball updates, which motivate the projection-based clipping used by NBCLT.

\subsection{A Quadratic Example}
\label{subsec:quadratic_example}
Consider the local quadratic model
\[
f(x+\Delta)-f(x)\approx \langle \nabla f(x),\Delta\rangle + \frac12\langle \Delta,\mathbf H(x)\Delta\rangle.
\]
Writing $\Delta=\eta u$ with $\|u\|^g\le 1$ separates the scale $\eta$ from the shape of the unit-budget set. To see why the shape matters, consider $q(z)=\frac12(\lambda_1 z_1^2+\lambda_2 z_2^2)$ with $\lambda_1\gg\lambda_2$ and Hessian $\mathbf H=\mathrm{diag}(\lambda_1,\lambda_2)$. Under the Euclidean budget $\|u\|_2\le 1$, $\sup_{\|u\|_2\le 1} u^\top \mathbf H u = \lambda_1$, so stability of the local quadratic model is dictated by the stiffest direction and forces the radius to scale with $1/\lambda_1$ in the worst case. In contrast, under the weighted norm $\|u\|^g:=\sqrt{\lambda_1 u_1^2+\lambda_2 u_2^2}$, we have
\[
\begin{aligned}
\sup_{\|u\|^g\le 1} u^\top \mathbf H u &= 1,\\
\sup_{\|u\|^g\le 1} -\langle g,u\rangle
&= \|g\|^{g,*}
= \sqrt{\frac{g_1^2}{\lambda_1}+\frac{g_2^2}{\lambda_2}}.
\end{aligned}
\]
Thus the norm absorbs the anisotropy: high-curvature directions become expensive, flatter directions become relatively affordable, and a single radius $\eta$ need not be dictated solely by the worst direction.

\subsection{Norm-Ball Update Geometry and Muon-type Motivation}
\label{subsec:ball}
Muon-type optimizers instantiate the same norm-budget viewpoint by selecting directions through a Linear Minimization Oracle (LMO), $\Delta^t \in \argmin_{\Delta\in\mathcal D}\langle m^t,\Delta\rangle$, over a norm ball $\mathcal D=\{\Delta:\|\Delta\|^g\le\eta\}$~\citep{jordan2024muon,pethick2025training,boreiko2025towards}. For example, the Euclidean ball yields normalized momentum directions, while a spectral-norm matrix ball yields the orthogonalized direction used in Muon~\citep{jordan2024muon}. The full uSCG template is recalled in Appendix~\ref{app:uscg_template}~\citep{pethick2025training}; here it serves only as motivation for norm balls as update-geometry priors, rather than as the algorithmic primitive of NBCLT.

For deep models, this viewpoint is naturally applied block-wise. If $\mathbf X=(\mathbf X^{(1)},\ldots,\mathbf X^{(L)})$ and each block carries a norm $\|\cdot\|_{(l)}$, define
\begin{equation}
\begin{aligned}
\|\mathbf X\|^{g} &:= \max_{l=1,\ldots,L}\|\mathbf X^{(l)}\|_{(l)},\\
\mathcal D&=\{\Delta\in\mathcal V:\|\Delta\|^g\le\eta\}.
\end{aligned}
\label{eq:blockwise_norm_budget}
\end{equation}
This covers layer-wise spectral constraints as in Muon~\citep{jordan2024muon}, heterogeneous layer geometries as in Scion~\citep{pethick2025training}, and finer matrix-block partitions such as MuonSBW~\citep{boreiko2025towards}.

NBCLT differs from LMO-based Muon-type methods. It applies Euclidean projection onto the same norm balls, so a tentative local update is modified only when it leaves the prescribed budget set. Thus the norm determines the admissible update geometry, while the algorithm implements this geometry as clipping rather than normalization.

\section{NBCLT Framework}
\label{sec:framework}
We now instantiate the norm-budget viewpoint in the distributed objective~\eqref{eq:obj}. The system contains $M$ nodes. In FL, a server broadcasts the current model and aggregates updates from participating clients; in DL, nodes replace the server step by peer-to-peer mixing. In both cases, NBCLT performs $K$ local predict-and-correct steps between communication phases, while constraining each local update variable to the norm budget $\mathcal D$.

\subsection{Local Update Design}
\label{subsec:pac}

The core NBCLT update is shared by the FL and DL variants. At communication round $t$, local step $k$, and node $i$, let $g_i^{t,k}$ denote the corrected gradient surrogate. NBCLT maintains an auxiliary update variable $\hat{\Delta}_i^{t,k}$ and applies
\begin{align*}
\hat{\Delta}^{t,k-1}_{i}
&= \Proj\!\left( \hat{\Delta}^{t,k-2}_{i} - g_i^{t,k} \right), \\
\Delta^{t,k-1}_{i}
&= \Proj\!\left( \hat{\Delta}^{t,k-1}_{i} - g_i^{t,k} \right), \\
x^{t,k}_{i}
&= x^{t,k-1}_{i} + \Delta^{t,k-1}_{i}.
\end{align*}
The two components of this update are described below.

\paragraph{Local predict-and-correct.}
The auxiliary variable $\hat{\Delta}_i^{t,k-1}$ is obtained by taking a projected step from the previous auxiliary update and can be viewed as a prediction of the next update. The actual update $\Delta_i^{t,k-1}$ then corrects this prediction by applying the same gradient surrogate once more before updating the model. This two-step mechanism is the main device that stabilizes clipped local training with multiple local steps. Equivalently, $\hat{\Delta}_i^{t,k-1}$ acts as a momentum-like update variable that accumulates past gradient information, while $\Delta_i^{t,k-1}$ combines this history with the current corrected gradient surrogate.

\paragraph{Projection oracle.}
NBCLT applies Euclidean projection $\Proj(z):=\argmin_{u\in\mathcal D}\frac12\|u-z\|_2^2$ to tentative update variables, where $\mathcal D=\{\Delta:\|\Delta\|^g\le \eta\}$. Operationally, this is clipping the update under the chosen norm. For the Euclidean ball, the projection has the familiar form
\[
\Proj(z)
=
\begin{cases}
z, & \|z\|_2\le \eta,\\[3pt]
\displaystyle \eta\,z/\|z\|_2, & \|z\|_2>\eta.
\end{cases}
\]

Matrix geometries fit the same formulation through vectorization. If $Z\in\mathbb R^{m\times n}$ and $v=\vecop(Z)$, define the induced vector norm $\|v\|_{\vecop(\mathrm{op})}:=\|\unvecop(v)\|_{\mathrm{op}}$, where $\|\cdot\|_{\mathrm{op}}$ denotes the matrix spectral norm. Since vectorization preserves Frobenius/Euclidean distances, Euclidean projection of $v$ onto the induced vector norm ball is equivalent to Frobenius projection of $Z$ onto $\mathcal D_{\mathrm{op}}=\{Y\in\mathbb R^{m\times n}:\|Y\|_{\mathrm{op}}\le \eta\}$. If $Z=U\operatorname{diag}(\sigma_1,\ldots,\sigma_r)V^\top$, then
\[
\Pi_{\mathcal D_{\mathrm{op}}}(Z)
=
U\operatorname{diag}\bigl(
\min\{\sigma_1,\eta\},\ldots,\min\{\sigma_r,\eta\}
\bigr)V^\top.
\]
This differs from LMO-based Muon-type updates~\citep{jordan2024muon,pethick2025training}, which select an update by solving a linear minimization problem over the norm ball. NBCLT instead projects a tentative update and changes it only when it leaves the prescribed budget. Plain stochastic gradient descent with gradient clipping can fail to converge without a careful coupling of the clipping threshold and the stepsize~\citep{koloskova2023revisiting}; the predict-and-correct mechanism above yields the convergence guarantees proved in Section~\ref{sec:convergence}.

\subsection{NBCLT-FL}
\label{subsec:NBCLT_fl}
NBCLT-FL is designed for federated learning and is summarized in Algorithm~\ref{alg:NBCLT-FL}. The method supports client sampling: at each communication round $t$, the server samples a subset of clients $S_t$ and broadcasts the current global model $x^t$ to the selected clients. Each participating client then performs $K$ local update steps using the predict-and-correct mechanism described above. After completing the local updates, clients transmit their accumulated model updates $\Delta_i^t$ to the server, which aggregates them to form the next global model $x^{t+1}$.

To mitigate client drift, NBCLT-FL incorporates control variates in the spirit of SCAFFOLD~\citep{karimireddy2020scaffold}. Specifically, each client maintains local control variates $\hat{c}_i^t$ and $c_i^t$, while the server maintains a global control variate $m^t$. We introduce two key modifications to the standard SCAFFOLD design: a shrinkage parameter $\beta_1$ in the calibrated gradient surrogate and an exponential moving average with parameter $\beta_2$ for the local control variate,
\[
\begin{aligned}
g^k &\leftarrow (1-\beta_1)\big(\nabla f(x^{t,k-1}_{i}, \xi^{t,k-1}_{i}) - c^{t}_{i}\big) + m^{t},\\
c^{t+1}_{i} &\leftarrow \beta_2 c^{t}_{i} + (1-\beta_2)\hat{c}^{t}_{i}.
\end{aligned}
\]
\begin{algorithm}[H]
\caption{NBCLT-FL}
\label{alg:NBCLT-FL}
\begin{algorithmic}[1]
    \STATE \textbf{Input:} Rounds $T$, clients $M$, local steps $K$, update-budget radius $\eta$, subset size $N=|S_t|$, domain $\mathcal{D}$, hyperparameters $\beta_1, \beta_2$.
    \STATE \textbf{Initialize:} $X^1 \in \mathbb{R}^d$, $\{c_i^1=\nabla f_i( X^1,\xi_i^{-1})\}_{i=1}^M$, $m^1 = \frac{1}{M} \sum_{i=1}^{M}c_i^1$.
    \FOR{$t = 1, \dots, T$}
        \STATE Server selects subset $S_t \subset [M]$ uniformly at random.
        \STATE Server broadcasts $x^{t}$ and $m^{t}$ to clients in $S_t$.
        \STATE \textbf{Client Execution (in parallel for $\boldsymbol{i \in S_t}$):}
         \STATE Initialize $\hat{\Delta}^{t,-1}_{i} \leftarrow 0$, $x^{t,0}_{i}\leftarrow x^{t}$.
        \FOR{$k = 1, \dots, K$}
            \STATE $g^{k} \leftarrow (1-\beta_1)(\nabla f(x^{t,k-1}_{i}, \xi^{t,k-1}_{i})-c^{t}_{i}) + m^{t} $
            \STATE $\hat{\Delta}^{t,k-1}_{i} \leftarrow \Proj \big( \hat{\Delta}^{t,k-2}_{i} - g^{k} \big)$
            \STATE $\Delta^{t,k-1}_{i} \leftarrow \Proj \big( \hat{\Delta}^{t,k-1}_{i} -  g^{k} \big)$
            \STATE $x^{t,k}_{i} \leftarrow x^{t,k-1}_{i} + \Delta^{t,k-1}_{i}$
        \ENDFOR
        \STATE $\hat{c}^{t}_{i} \leftarrow \frac{1}{K}\sum_{k=1}^{K}\grad f_i(x^{t,k-1}_{i}, \xi_{i}^{t,k-1})$
        \STATE $c^{t+1}_{i} \leftarrow \beta_2 c^{t}_{i}+(1-\beta_2)\hat{c}^{t}_{i}$
        \STATE Send $\Delta_{i}^{t} = x^{t,K}_{i} - x^{t,0}_{i}$ and $c^{t+1}_{i}$ to server.
        \STATE \textbf{For each client} $\boldsymbol{i \notin S_t}$: \quad  $c^{t+1}_{i} \leftarrow c^{t}_{i}$.
        \STATE \textbf{Server Aggregation:}
        \STATE Average local models: $x^{t+1} \leftarrow x^{t} + \frac{1}{N}\sum_{i\in S_t } \Delta_{i}^{t}$.
        \STATE Update global control: $m^{t+1} \leftarrow \frac{1}{M}\sum_{i=1}^M c^{t+1}_i$.
    \ENDFOR
\end{algorithmic}
\end{algorithm}

\subsection{NBCLT-DL}
\label{subsec:NBCLT_dl}

NBCLT-DL is designed for decentralized learning over peer-to-peer networks, where nodes communicate directly with their neighbors, and is summarized in Algorithm~\ref{alg:NBCLT-DL}. Unlike classical gradient-tracking methods~\citep{xu2020accelerated}, NBCLT-DL follows a design philosophy similar to NBCLT-FL: each node maintains local control variates $\hat{c}_i^t$ and $c_i^t$, along with a local estimate of the global control variate, denoted by $m_i^t$.

At each communication round, every node first performs $K$ local update steps. This is followed by a communication phase. In the absence of a central coordinator, local models and control variates cannot be directly averaged. Instead, nodes exchange information with their neighbors according to a mixing matrix $\mathbf{W}^t = [w_{ij}^t]$, which encodes the network topology. The local models are updated via neighbor averaging,
\(
x^{t+1}_{i} \leftarrow \sum_{j=1}^{M} w_{ij}^{t} x^{t,K}_{j}.
\)
For the update of the local estimate $m_i^t$, each node first incorporates its newly updated local control variate and then performs neighbor mixing, $m^{t+1}_{i} \leftarrow \sum_{j=1}^{M} w_{ij}^{t} \big(m^{t}_{j} + c^{t+1}_{j} - c^{t}_{j}\big)$. This ensures that the local estimates $m_i^t$ progressively approximate the global gradient, thereby reducing discrepancies among local models over time.




\begin{algorithm}[ht]
\caption{NBCLT-DL}
\label{alg:NBCLT-DL}
\begin{algorithmic}[1]
    \STATE \textbf{Input:} Rounds $T$, clients $M$, local steps $K$, mixing matrix $W$, domain $\mathcal{D}$, hyperparameters $ \beta_1, \beta_2$.
    \STATE \textbf{Initialize:} $x_i^1=x_{\text{init}}$ and $c_i^1=\nabla f_i(x_{\text{init}},\xi_i^{-1})$, using one stochastic-gradient query per node.
    \STATE Exactly average the $c_i^1$'s using the initialization in Appendix~\ref{app:dl-initialization}, and set $m_i^1=M^{-1}\sum_{j=1}^M c_j^1$ at every node.
    \FOR{$t = 1, \dots, T$}
        \FORALL{node $i \in [M]$ in parallel}
            \STATE Initialize $\hat{\Delta}^{t,-1}_{i} \leftarrow 0$.
            \FOR{$k = 1, \dots, K$}
                \STATE $g^{k}_i\leftarrow (1-\beta_1)(\nabla f(x^{t,k-1}_{i},\xi^{t,k-1}_{i})- c^{t}_{i}) + m^{t}_i$
                \STATE $\hat{\Delta}^{t,k-1}_{i} \leftarrow \Proj \big( \hat{\Delta}^{t,k-2}_{i} -  g^{k}_i \big)$
                \STATE $\Delta^{t,k-1}_{i} \leftarrow \Proj \big( \hat{\Delta}^{t,k-1}_{i} -  g^{k}_i \big)$
                \STATE $x^{t,k}_{i} \leftarrow x^{t,k-1}_{i} + \Delta^{t,k-1}_{i}$
            \ENDFOR
            \STATE $\hat{c}^{t}_{i} \leftarrow \frac{1}{K}\sum_{k=1}^{K}\grad f_i(x^{t,k-1}_{i}, \xi_{i}^{t,k-1})$
            \STATE $c^{t+1}_{i}\leftarrow \beta_2 c^{t}_{i}+(1-\beta_2)\hat{c}^{t}_{i}$
            \STATE \textbf{Communication:}
            \STATE Mix models: $x^{t+1}_{i} \leftarrow \sum_{j=1}^{M} w_{ij}^{t} x^{t,K}_{j}$
            \STATE Update tracker: $m^{t+1}_{i} \leftarrow \sum_{j=1}^{M} w_{ij}^{t} (m^{t}_{j} + c^{t+1}_{j} - c^{t}_{j})$
        \ENDFOR
    \ENDFOR
\end{algorithmic}
\end{algorithm}

\section{Convergence Guarantees}
\label{sec:convergence}
Before presenting the convergence for NBCLT, we first state the underlying assumptions. We assume that each local objective function $f_i$ is Lipschitz smooth, but not necessarily convex. For each $f_i$, we have access to an unbiased stochastic gradient oracle $\nabla f_i(\cdot, \xi)$.
\begin{assumption}[Smoothness]\label{assump:smooth}
Each $f_i$ is $L$-smooth: $\|\nabla f_i(x)-\nabla f_i(y)\|_2 \le L\|x-y\|_2$ for all $x,y \in \mathbb{R}^d$.
\end{assumption}

\begin{assumption}[Stochastic gradients]\label{assump:stoch}
For each client $i$, the stochastic gradient $\nabla f_i(x,\xi)$ is conditionally unbiased and has bounded
conditional variance:
\begin{align*}
\mathbb E[\nabla f_i(x,\xi)\mid x]&=\nabla f_i(x), \\
\mathbb E[\|\nabla f_i(x,\xi)-\nabla f_i(x)\|_2^2\mid x]&\le \sigma^2.
\end{align*}
Moreover, conditioned on the past, sampling noises are independent across gradient queries.
\end{assumption}

\begin{assumption}[Lower boundedness]\label{assump:lower}
The global objective $F(x):=\frac1M\sum_{i=1}^M f_i(x)$ is lower bounded:
$\forall x \in \mathbb{R}^d, F(x)\ge F^\star.$
\end{assumption}

\begin{assumption}[Norm equivalence and dual norm]\label{assump:norm}
The projection domain is a norm ball $\mathcal D=\{\Delta:\|\Delta\|^{g}\le\eta\}$ with respect to some norm \(\|\cdot\|^g\). There exist constants $\rho,\bar\rho\ge 1$ such that $\|v\|^{g}\le \rho\|v\|_2$ and $\|v\|_2\le \bar\rho\,\|v\|^{g}$ for all $v\in \mathbb{R}^d$.
The dual norm is \(\|v\|^{g,*}:=\sup_{\| u\|^{g}\le 1}\langle v,u\rangle\). Consequently, $\|\cdot\|^{g,*}$ and $\|\cdot\|_2$ are comparable up to constants: $\rho^{-1}\|v\|_2 \le \|v\|^{g,*} \le \bar\rho\,\|v\|_2$ for all $v$.
\end{assumption}
In practice, we adopt the composite norm described in Section~\ref{sec:prelim}. This assumption on the constraint set $\mathcal{D}$ can be relaxed to a closed and convex set that contains a norm ball.

\subsection{Convergence of NBCLT-FL}
\label{subsec:fl_conv}

The FL analysis admits two complementary regimes.
The main text focuses on the problem-parameter-free instantiation, where all algorithmic parameters are chosen from known system quantities such as $(N,K,T)$.
This is a fixed-horizon guarantee: $T$ is specified in advance, while unknown problem constants are not needed.
When problem-dependent quantities such as $\Delta_F$, $L$, and $\sigma$ are available, the same algorithm also admits a well-tuned guarantee with a fixed smoothing parameter and no coupling condition among $N$, $K$, and $T$; we defer this complementary result to Appendix~\ref{app:fl_welltuned}.
We now state the main problem-parameter-free guarantee.

\begin{theorem}[Problem-parameter-free convergence of NBCLT-FL]
\label{thm:fl_final}
Suppose Assumptions~\ref{assump:smooth} -- \ref{assump:norm} hold.
Run Algorithm~\ref{alg:NBCLT-FL} for $T$ communication rounds, with $K$ local updates per round and subset size $N$ (out of $M$ total clients).
Choose the hyperparameters as $\eta := N^{1/4}/(K^{3/4}T^{3/4})$, $1-\beta_1 := \eta$, and $1-\beta_2 := \sqrt{NK/(T+K)}$. Define $\Delta_F := F(x^1)-F^\star$ and $p:=N/M$.
Then for all $T\ge K$ such that $NK\le T$, we have
\begin{equation} \label{eq:rate_final}
\frac{1}{T}\sum_{t=1}^T \mathbb E\!\left[\|\nabla F(x^t)\|^{g,*}\right]
\le\;
\frac{C_1}{\sqrt{T}}
+
\frac{C_2}{(NKT)^{1/4}}
+\;C_0\,R_{\mathrm{lo}}(T,K),
\end{equation}
where
\[
C_1:=\frac{5L\bar\rho^2}{2},
\qquad
C_2:=\Delta_F+\frac{64L\bar\rho^2}{p}+32\sigma\bar\rho .
\]
The lower-order remainder term is
\begin{equation*}
\begin{aligned}
R_{\mathrm{lo}}(T,K)
={}&
\frac{N^{1/4}}{K^{3/4}T^{3/4}}
+\frac{N^{1/2}}{K^{1/2}T^{3/2}}\\
&+\frac{N^{3/4}}{K^{9/4}T^{9/4}}
+\frac{1}{N^{1/4}K^{5/4}T^{3/4}}\\
&+\frac{1}{\sqrt{NKT}},
\end{aligned}
\end{equation*}
and the constant $C_0$ depends only on fixed problem parameters (e.g., $L,\sigma,\bar\rho,p$), but not on $(T,K)$.
\end{theorem}

Theorem~\ref{thm:fl_final} is problem-parameter-free: the update-budget radius and smoothing parameters depend only on known system quantities $(N,K,T)$, not on $L$, $\sigma$, or $\Delta_F$.
Its leading terms are $\mathcal O(T^{-1/2})$ and $\mathcal O((NKT)^{-1/4})$, while $R_{\mathrm{lo}}$ is lower order.
Thus, taking $T=\Theta(\epsilon^{-2})$ and choosing $NK=\Theta(T)$ within $NK\le T$ yields $\mathcal O(\epsilon^{-2})$ communication rounds and $\mathcal O(\epsilon^{-4})$ total stochastic-gradient queries, including the additive $M$ initialization queries for a fixed client population. This communication guarantee trades additional local computation for fewer synchronization rounds; $NK\le T$ alone does not imply it. With $N,K$ fixed, the bound instead gives $\mathcal O(\epsilon^{-4})$ rounds. Constants, including their dependence on $p$, are held fixed in these accuracy exponents.
The participation size $N$ improves the second leading term through both the denominator $(NKT)^{1/4}$ and the participation ratio $p=N/M$ appearing in $C_2$; in the full-participation case $(N=M)$, the leading stochastic term scales as $\Theta((MKT)^{-1/4})$ up to problem-dependent constants.

\subsection{Effect of Inexact Projection}
\label{subsec:inexact_projectors}

The preceding FL guarantee assumes exact Euclidean projection onto $\mathcal D$.
The following theorem shows that NBCLT-FL retains its convergence rates with feasible inexact projections, provided the accumulated projection residual is sufficiently small.

\begin{theorem}[Convergence of NBCLT-FL with inexact projections]
\label{thm:fl_inexact_projection}
Suppose Assumptions~\ref{assump:smooth}--\ref{assump:norm} hold.
Run NBCLT-FL with the same hyperparameter choices as in either Theorem~\ref{thm:fl_final} or the well-tuned FL guarantee in Theorem~\ref{thm:fl_welltuned}, but replace the exact Euclidean projection calls in Algorithm~\ref{alg:NBCLT-FL} by feasible inexact projection calls.
For a projection subproblem with input $z$ and feasible output $y\in\mathcal D$, define the single-call projection residual
\begin{align}
\operatorname{res}_{\mathcal D}(z,y)
&:=
\left[
\sup_{u\in\mathcal D}\langle z-y,u-y\rangle
\right]_+,
\label{eq:main_single_call_residual}\\
&=
\left[
\eta\|z-y\|^{g,*}-\langle z-y,y\rangle
\right]_+.
\label{eq:main_single_call_residual_closed}
\end{align}
Here $\mathcal D=\{u:\|u\|^g\le \eta\}$.
For the prediction and correction projection calls at round $t$, client $i$, and local step $k$, let
\begin{align*}
\hat z_i^{t,k}&:=\hat\Delta_i^{t,k-2}-g_i^{t,k},\\
z_i^{t,k}&:=\hat\Delta_i^{t,k-1}-g_i^{t,k},
\end{align*}
and let $\hat\Delta_i^{t,k-1},\Delta_i^{t,k-1}\in\mathcal D$ be the corresponding feasible outputs.
Write $\hat r_i^{t,k}:=\operatorname{res}_{\mathcal D}(\hat z_i^{t,k},\hat\Delta_i^{t,k-1})$
and $r_i^{t,k}:=\operatorname{res}_{\mathcal D}(z_i^{t,k},\Delta_i^{t,k-1})$.
Define the averaged accumulated projection residual
\begin{equation}
\overline{\mathfrak R}_T :=
\frac1T\sum_{t=1}^T
\mathbb E\!\left[
\frac1N\sum_{i\in S_t}\sum_{k=1}^K
\big(\hat r_i^{t,k}+r_i^{t,k}\big)
\right].
\label{eq:avg_projection_residual}
\end{equation}
Then the corresponding exact-projection FL convergence bounds remain valid, up to enlarged numerical constants with the same dependence on problem parameters, with only an additional perturbation term $\mathcal O\!\left(\overline{\mathfrak R}_T/(\eta K)\right)$ on the right-hand side.
Consequently, if $\overline{\mathfrak R}_T=\mathcal O(T^{-1})$, then the inexact-projection implementation preserves the communication and sample-complexity exponents of both the problem-parameter-free and well-tuned FL regimes.
When this residual budget is distributed uniformly over the $K$ local steps, the natural single-call target is of order $1/(KT)$.
\end{theorem}

Theorem~\ref{thm:fl_inexact_projection} turns exact projection into a residual-budget condition rather than a mandatory oracle. For spectral-norm matrix blocks, this motivates an SVD-free clipping oracle for $\mathcal B_\tau=\{Y:\|Y\|_2\le\tau\}$. Let $Z\in\mathbb R^{m\times n}$ be a matrix update block and assume $m\ge n$ for notation. Choose any safe scale $s\ge\|Z\|_2$ and block radius $\tau$; this scale is easy to obtain, since $s=\|Z\|_F$ already suffices, while our certified implementation uses a deterministic tighter upper bound described in Appendix~\ref{app:polygram_table_oracle}. Set $G:=Z^\top Z/s^2$ and $\alpha_{\rm clip}:=\tau/s$. In the nontrivial case $\tau<\|Z\|_2\le s$, PolyGram-RCClip outputs
\[
Y_{d,s}:=Zp_{d,\alpha_{\rm clip}}(G),
\]
where the degree-$d$ transfer $p_{d,\alpha_{\rm clip}}$ is selected by the clipping ratio; the case $m<n$ uses the left Gram matrix and is given in Algorithm~\ref{alg:polygram_rcclip}. Appendix~\ref{app:polygram_existence} proves that for every $\alpha_{\rm clip}\in(0,1)$ there exists a feasible degree-$d$ transfer such that, for all sufficiently large $d$,
\[
\operatorname{res}_{\mathcal B_\tau}(Z,Y_{d,s})
\le
\frac{C_J}{\alpha_{\rm clip}^2 d}
\bigl(\tau\|Z\|_*+2\|Z\|_F^2\bigr),
\]
where $C_J$ is a universal Jackson-type approximation constant. This bound exposes an accuracy--compute tradeoff: for fixed $\alpha_{\rm clip}$, increasing $d$ lowers the certified residual upper bound as $1/d$, while Algorithm~\ref{alg:polygram_rcclip} uses exactly $d$ sequential multiplications by the smaller Gram matrix after it is formed. For fixed $d$ and $\alpha_{\rm clip}$, Appendix~\ref{app:polygram_lp} formulates coefficient search as a semi-infinite LP; after normalization, it depends only on $\alpha_{\rm clip}$, not on $Z,\tau,s$ separately. Hence the runtime oracle uses an offline one-dimensional table: for $d=20$ and $1000$ grid points, it stores only $21{,}000$ coefficients, below $0.2$ MB in float64. Each call computes $s$, forms $\alpha_{\rm clip}$, looks up coefficients, and evaluates the Gram polynomial. Newton--Schulz-5 Muon post-processing instead uses three matrix products per iteration and five iterations in the reference implementation, and it orthogonalizes rather than clips~\citep{jordan2024muon}.

The degree--residual bound is an existence guarantee for a sufficiently accurate, feasible polynomial on each call. Our fixed-degree table is a practical approximation: its safe scale does not by itself establish the accumulated residual budget required by Theorem~\ref{thm:fl_inexact_projection}. Appendix~\ref{app:polygram_table_oracle} separates these requirements from the experimental implementation.

\subsection{Convergence of NBCLT-DL}
\label{subsec:dl_conv}

We analyze NBCLT-DL under two network models: a static symmetric doubly stochastic mixing matrix, controlled by its spectral gap, and time-varying doubly stochastic mixing matrices satisfying finite-time consensus (FTC).
For both topologies, we use the same problem-parameter-free choice $\eta=M^{1/4}/(K^{3/4}T^{3/4})$, $1-\beta_1=\eta$, and $1-\beta_2=\sqrt{MK/(T+K)}$. The formal statements are given in Theorems~\ref{thm:dl_final_gap} and~\ref{thm:dl_final_FTC} in Appendix~\ref{app:dl-formal-guarantees}; the proofs appear later in Appendix~\ref{app:dl-analysis}.
For $T\ge MK$, both results imply
\begin{equation}
\begin{aligned}
\frac1T\sum_{t=1}^T \mathbb E\|\nabla F(\bar x^t)\|^{g,*}
\le {}&\mathcal{O}(T^{-1/2})+
\mathcal{O}\bigl((MKT)^{-1/4}\bigr)\\
&+\text{lower-order terms},
\end{aligned}
\label{eq:dl_summary_rate}
\end{equation}
where the hidden constants may depend on the network parameter and fixed problem constants, but not on $(T,K)$.
Consequently, choosing $T=\Theta(\epsilon^{-2})$ and $MK=\Theta(T)$ within $MK\le T$ gives $\mathcal O(\epsilon^{-2})$ communication rounds and $\mathcal O(\epsilon^{-4})$ total stochastic-gradient queries for fixed problem and network constants. For fixed $M,K$, the bound instead gives $\mathcal O(\epsilon^{-4})$ rounds. Exact tracker initialization adds $2h$ rounds on a precomputed static spanning tree of depth $h$, or $\tau$ rounds under FTC; the total oracle count is $M(KT+1)$, as detailed in Appendix~\ref{app:dl-initialization}.

\subsubsection{Static doubly stochastic mixing}
\label{subsec:dl_static}
Assume the mixing matrix $\mathbf{W}\in\mathbb{R}^{M\times M}$ is fixed, symmetric, and doubly stochastic: $\mathbf{W}\mathbf{1}=\mathbf{1}$ and $\mathbf{1}^\top\mathbf{W}=\mathbf{1}^\top$.
Let $\mathbf{J}:=M^{-1}\mathbf 1\mathbf 1^\top$ and define $\lambda:=\|\mathbf W-\mathbf J\|_2\in[0,1)$, equivalently the second-largest eigenvalue magnitude of $\mathbf W$.
The spectral gap $1-\lambda$ measures the contraction of disagreement under one mixing step.
Theorem~\ref{thm:dl_final_gap} gives the complete static-gossip guarantee, including the explicit constants and lower-order terms; the constants depend on the static-network parameter through $1/(1-\lambda)$, but the polynomial dependence on $(T,K,M)$ is the one summarized in~\eqref{eq:dl_summary_rate}.

\subsubsection{Time-varying mixing with finite-time consensus}
\label{subsec:dl_ftc}
Static sparse topologies can have deteriorating spectral gaps as the network grows; for example, a ring has contraction factor $1-\mathcal O(M^{-2})$, equivalently spectral gap $\mathcal O(M^{-2})$~\citep{takezawa2023beyond}, while faster static constructions such as the exponential graph improve the rate at the cost of increasing node degree~\citep{ying2021exponential}.
Time-varying communication schemes avoid part of this tradeoff by scheduling sparse mixing patterns over time while keeping the per-round communication degree small~\citep{ying2021exponential,nguyen2025graphs}.
We assume the FTC property: there exists an integer $\tau\ge1$ such that, for every $t$,
\begin{equation}
\mathbf{W}^{t+\tau-1}\mathbf{W}^{t+\tau-2}\cdots\mathbf{W}^{t}=\mathbf J,
\qquad
\mathbf J:=M^{-1}\mathbf 1\mathbf 1^\top .
\label{eq:ftc}
\end{equation}
We refer to $\tau$ as the FTC horizon.
Theorem~\ref{thm:dl_final_FTC} gives the complete FTC guarantee, including the initialization condition used to control the initial tracker error.
Its constants depend on $\tau$, while the communication and sample-complexity exponents remain those in~\eqref{eq:dl_summary_rate}.

\section{Experiments}
\label{sec:experiments}

\paragraph{Setup.}
We evaluate NBCLT-FL on CIFAR-10~\citep{krizhevsky2009learning} under Dirichlet label-skew non-IID partitions.
The 50k training examples are split across $M=100$ clients with approximately balanced client sizes, and the heterogeneity level is controlled by $\alpha_{\rm dir}\in\{0.1,0.3,0.5\}$, where smaller values indicate stronger label skew.
Each round samples $N=10$ clients, each selected client performs $K=5$ local steps with mini-batch size 50, and training runs for $T=1000$ communication rounds.
We use a CIFAR-style ResNet-18 with GELU activations and GroupNorm.
We compare against FedAvg~\citep{mcmahan2017communication}, SCAFFOLD-M~\citep{cheng2023momentum}, FedAMS~\citep{wang2022communication}, FedMuon~\citep{takezawa2025fedmuon}, and PAdaMFed~\citep{yan2025problem}, each tuned over its own hyperparameter grid. For FedMuon, we use exact orthogonalization for the LMO-based layers rather than a Newton--Schulz approximation.

We report four NBCLT variants: exact-SVD tuned/PF (SVD-T/SVD-PF) and PolyGram-RCClip tuned/PF (Poly-T/Poly-PF).
For PF, the theorem-suggested choice gives $\eta\approx0.00299$, $\beta_1\approx0.997$, and $\beta_2\approx0.777$ for $N=10,K=5,T=1000$.
Additional implementation details, hyperparameter grids, and the full nine-curve comparison are deferred to Appendix~\ref{app:experiments}.

\begin{figure}[tbp]
\centering
\includegraphics[width=\textwidth]{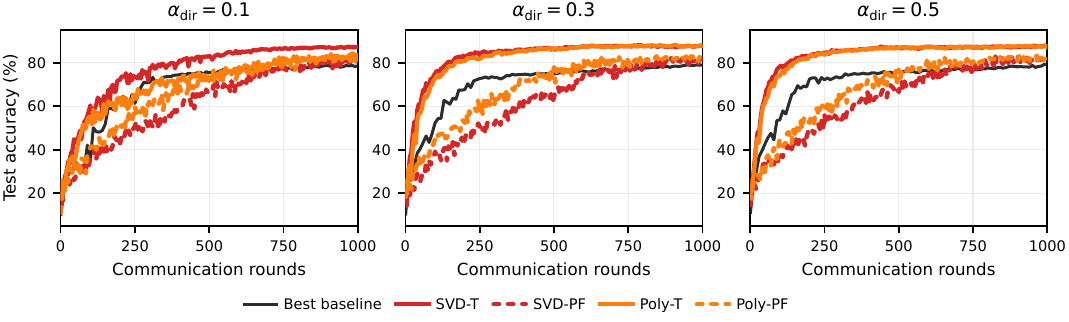}
\caption{Main CIFAR-10 FL comparison. Best baseline is the best tuned run among the five baselines, selected by average test accuracy over the last 100 rounds; the full nine-curve comparison appears in Appendix~\ref{app:experiments}.}
\label{fig:cifar10-fl}
\end{figure}

Figure~\ref{fig:cifar10-fl} shows that tuned NBCLT-FL improves over the strongest tuned baseline across all three heterogeneity levels.
For $\alpha_{\rm dir}=0.3$ and $0.5$, the degree-20 PolyGram-RCClip curves remain close to their exact-SVD counterparts in both the tuned and PF regimes.

The most heterogeneous split, $\alpha_{\rm dir}=0.1$, is more sensitive to the projector approximation, where the gap between Poly-T and SVD-T is larger.
Figure~\ref{fig:cifar10-degree-ablation} is an empirical degree ablation: increasing the polynomial degree reduces the observed gap to the exact-clipping reference in this setting. It does not measure the projection residual or verify the $1/d$ residual bound.

\begin{figure}[H]
\centering
\includegraphics[width=0.62\textwidth]{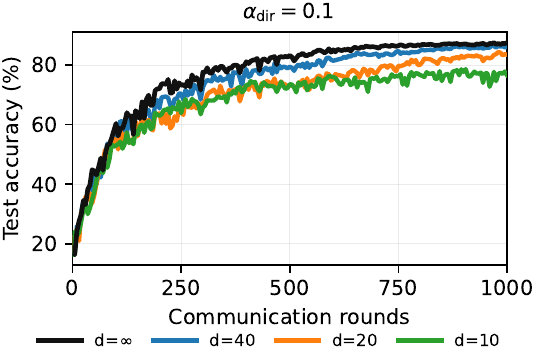}
\caption{Projector-degree ablation on $\alpha_{\rm dir}=0.1$; $d=\infty$ denotes exact spectral clipping.}
\label{fig:cifar10-degree-ablation}
\end{figure}

\section*{AI Use Statement}
Generative AI tools were used to assist with manuscript editing, literature
searches, and discussion and revision of mathematical arguments. The authors
remain responsible for all claims, proofs, experiments, and citations.

\bibliographystyle{plainnat}
\begingroup
\interlinepenalty=10000
\bibliography{references}
\endgroup
\clearpage
\appendix
\numberwithin{theorem}{section}
\numberwithin{equation}{section}
\numberwithin{algorithm}{section}
\section{Federated-learning analysis}
\label{app:fl-analysis}
\label{app:proof_fl_final_new}

\subsection{Well-tuned FL guarantee}
\label{app:fl_welltuned}

\begin{theorem}[Well-tuned convergence of NBCLT-FL]
\label{thm:fl_welltuned}
Suppose Assumptions~\ref{assump:smooth}--\ref{assump:norm} hold.
Run Algorithm~\ref{alg:NBCLT-FL} for $T$ communication rounds, with $K$ local updates per round and subset size $N$ (out of $M$ total clients).
Define $\Delta_F:=F(x^1)-F^\star$ and $p:=N/M$.
Fix any constant $\gamma\in(0,1]$ and choose
\begin{align*}
1-\beta_2 := \gamma,
\qquad
1-\beta_1 := \eta.
\end{align*}
Let
\begin{align*}
A_{\mathrm{wt}}(\gamma)
:=&\;\left(\frac{11}{2}+\frac{32\sqrt{2}}{p\gamma}\right)L\bar\rho^2 K
+\left(\frac{5}{2}+L^2\right)\bar\rho^2
+2\sigma^2
+2\bar\rho\sigma
+\frac{8\sqrt{2}\,\bar\rho\sigma\sqrt{\gamma}}{\sqrt{Kp}}.
\end{align*}
Choose the update-budget radius as
\begin{align*}
\eta := \min\left\{1,\sqrt{\frac{\Delta_F}{A_{\mathrm{wt}}(\gamma)KT}}\right\}.
\end{align*}
Then for any $T\ge 1$, we have
\begin{align}
\frac{1}{T}\sum_{t=1}^T \mathbb E\!\left[\|\nabla F(x^t)\|^{g,*}\right]
\;&\le\;
2\sqrt{\frac{A_{\mathrm{wt}}(\gamma)\Delta_F}{KT}}
+\frac{\Delta_F}{KT}
+\frac{16\bar\rho\sigma}{p\gamma T}
+\frac{8\bar\rho\sigma}{\sqrt{N\gamma T}}
+\frac{8\bar\rho\sigma\sqrt{\gamma}}{\sqrt{KM}}.
\label{eq:rate_fl_welltuned}
\end{align}
\end{theorem}

This result is not problem-parameter-free, since its update-budget radius depends on $\Delta_F$, $L$, $\sigma$, and $p$.
Its role is complementary: when these quantities are available, NBCLT-FL admits a well-tuned regime with no auxiliary condition on the relative sizes of $N$, $K$, and $T$.
If $\gamma$ is restricted to a fixed interval $[\underline\gamma,1]\subset(0,1]$, then all $\gamma$-dependence is absorbed into constants, yielding the usual $\mathcal O(\epsilon^{-2})$ communication and $\mathcal O(\epsilon^{-4})$ sample-complexity exponents under the standard tuned choice of $(T,K)$.

\subsection{Notation and preliminaries}
Let $F(x)=\frac{1}{M}\sum_{i=1}^M f_i(x)$ and $p:=N/M$.
In round $t$, the server samples $S_t\subset[M]$ uniformly at random with $|S_t|=N$.

\paragraph{Local iterates.}
For each active client $i\in S_t$, let $x_i^{t,0}=x^t$.
Define local increments
\[
\Delta_i^{t,k-1}:=x_i^{t,k}-x_i^{t,k-1},
\qquad
\Delta_i^t:=\sum_{k=1}^K\Delta_i^{t,k-1}.
\]
The server update is
\begin{equation}
 x^{t+1}-x^t = \frac{1}{N}\sum_{i\in S_t}\Delta_i^t.
\label{eq:server_update_app_new}
\end{equation}

\paragraph{Projection operator.}
Throughout this appendix, $\Proj(\cdot)$ denotes the \emph{Euclidean projection} onto $\mathcal D$:
\[
\Proj(z):=\arg\min_{u\in\mathcal D}\|u-z\|_2.
\]
Recall $\mathcal D=\{\Delta:\|\Delta\|^{g}\le \eta\}$.

\paragraph{Norm equivalence.}
By Assumption~\ref{assump:norm},
in finite-dimensional spaces, for any $\Delta\in\mathcal D$,
\begin{equation}
\|\Delta\|_2 \le \bar\rho\,\|\Delta\|^{g}\le \bar\rho\eta.
\label{eq:D_bound_new}
\end{equation}
In particular, $\|\Delta_i^{t,k-1}\|_2\le \bar\rho\eta$ and
\begin{equation}
\|x^{t+1}-x^t\|_2
\le
\frac{1}{N}\sum_{i\in S_t}\sum_{k=1}^K\|\Delta_i^{t,k-1}\|_2
\le K\bar\rho\eta.
\label{eq:move_round_new}
\end{equation}

\paragraph{Stochastic gradients.}
We write the stochastic gradient queried at client $i$ as $\nabla f_i(X,\xi)$, where $\xi$ denotes the (mini-)batch sampling randomness.
We assume \emph{conditional unbiasedness} and \emph{bounded conditional variance}:
for all $x$ and all $i\in[M]$,
\begin{equation}
\mathbb E\!\left[\nabla f_i(x,\xi)\mid x\right]=\nabla f_i(x),
\qquad
\mathbb E\!\left[\big\|\nabla f_i(x,\xi)-\nabla f_i(x)\big\|_2^2 \mid x\right]\le \sigma^2.
\label{eq:bv_new}
\end{equation}

\paragraph{Filtration.}
Let $\{\mathcal F_{t,k}\}$ be the natural filtration generated by all randomness revealed
up to (right before) drawing the mini-batch $\xi_i^{t,k-1}$ used to query
$\nabla f_i(x_i^{t,k-1},\xi_i^{t,k-1})$ in round $t$.
Conditioned on $\mathcal F_{t,k}$, the fresh sample $\xi_i^{t,k-1}$ is independent of the past.

\subsection{Auxiliary lemmas}

\begin{lemma}[$L$-smoothness descent]
\label{lem:smooth_app_new}
If $F$ is $L$-smooth, then for all $t$,
\begin{align}
F(x^{t+1})
&\le
F(x^t)
+\langle \nabla F(x^t), x^{t+1}-x^t\rangle
+\frac{L}{2}\|x^{t+1}-x^t\|_2^2 \notag\\
&\le
F(x^t)
+\langle \nabla F(x^t), x^{t+1}-x^t\rangle
+\frac{L}{2}(K\bar\rho\eta)^2.
\label{eq:smooth_app_new}
\end{align}
\end{lemma}

\begin{lemma}[Support function of $\mathcal D$]
\label{lem:support_app_new}
Let $\|\cdot\|^{g,*}$ be the dual norm induced by $\|\cdot\|^{g}$, and define
\[
 u_t\in\arg\min_{u\in\mathcal D}\langle \nabla F(x^t), u\rangle.
\]
Then
\begin{equation}
\langle \nabla F(x^t), u_t\rangle = -\eta\|\nabla F(x^t)\|^{g,*}.
\label{eq:support_app_new}
\end{equation}
\end{lemma}

\begin{lemma}[Projection three-point inequality]
\label{lem:proj_app_new}
Let $\Proj(\cdot)$ be the Euclidean projection onto a nonempty closed convex set $\mathcal D$.
Let $y=\Proj(x-g)$. Then for any $u\in\mathcal D$,
\begin{equation}
\langle g, y-u\rangle
\le
\frac{1}{2}\Big(\|u-x\|_2^2-\|u-y\|_2^2-\|y-x\|_2^2\Big).
\label{eq:proj_three_app_new}
\end{equation}
\end{lemma}
\begin{proof}
Optimality of Euclidean projection implies
$\langle (x-g)-y,\,u-y\rangle\le 0$ for all $u\in\mathcal D$.
Rearranging gives
$\langle g, y-u\rangle \le \langle x-y, u-y\rangle$.
Finally use
$2\langle x-y, u-y\rangle
=\|u-x\|_2^2-\|u-y\|_2^2-\|y-x\|_2^2$.
\end{proof}

\begin{lemma}[One-gradient predict-and-correct inequality]
\label{lem:pc_app_new}
Fix $(t,i)$ and set $\hat\Delta_i^{t,-1}=0$.
For $k=1,\dots,K$, Algorithm~\ref{alg:NBCLT-FL} generates
\[
\hat\Delta_i^{t,k-1}=\Proj\!\big(\hat\Delta_i^{t,k-2}-g_i^{t,k}\big),
\qquad
\Delta_i^{t,k-1}=\Proj\!\big(\hat\Delta_i^{t,k-1}-g_i^{t,k}\big).
\]
Then for any $u\in\mathcal D$,
\begin{align}
\sum_{k=1}^K \langle g_i^{t,k}, \Delta_i^{t,k-1}-u\rangle
\le
\frac{\|u\|_2^2}{2}
+\frac{1}{2}\sum_{k=2}^K \|g_i^{t,k}-g_i^{t,k-1}\|_2^2
+2K(\bar\rho\eta)^2.
\label{eq:pc_key_new}
\end{align}
\end{lemma}
\begin{proof}
Apply Lemma~\ref{lem:proj_app_new} to
$\Delta_i^{t,k-1}=\Proj(\hat\Delta_i^{t,k-1}-g_i^{t,k})$ with $x=\hat\Delta_i^{t,k-1}$:
\[
\langle g_i^{t,k}, \Delta_i^{t,k-1}-u\rangle
\le
\frac{1}{2}\Big(\|u-\hat\Delta_i^{t,k-1}\|_2^2-\|u-\Delta_i^{t,k-1}\|_2^2\Big).
\]
For $k=1$, $\hat\Delta_i^{t,0}=\Proj(-g_i^{t,1})\in\mathcal D$, hence
$\|u-\hat\Delta_i^{t,0}\|_2^2\le (2\bar\rho\eta)^2$.

For $k\ge 2$, by non-expansiveness of Euclidean projection,
\[
\|\hat\Delta_i^{t,k-1}-\Delta_i^{t,k-2}\|_2
=
\Big\|\Proj(\hat\Delta_i^{t,k-2}-g_i^{t,k})
-\Proj(\hat\Delta_i^{t,k-2}-g_i^{t,k-1})\Big\|_2
\le
\|g_i^{t,k}-g_i^{t,k-1}\|_2.
\]
Thus
\begin{equation}
\begin{aligned}
\|u-\hat\Delta_i^{t,k-1}\|_2^2
&\le 2\|u-\Delta_i^{t,k-2}\|_2^2 + 2\|\hat\Delta_i^{t,k-1}-\Delta_i^{t,k-2}\|_2^2 \\
&\le 2\|u-\Delta_i^{t,k-2}\|_2^2 + 2\|g_i^{t,k}-g_i^{t,k-1}\|_2^2.
\end{aligned}
\label{eq:bound_u_delta}
\end{equation}
Summing the inequality \eqref{eq:bound_u_delta} over $k=1,\dots,K$, and treating the $k=1$ case separately, we have:
\begin{align}
\sum_{k=1}^K \langle g_i^{t,k}, \Delta_i^{t,k-1}-u\rangle
&\le
\underbrace{\frac{1}{2}\|u-\hat\Delta_i^{t,0}\|_2^2 - \frac{1}{2}\|u-\Delta_i^{t,0}\|_2^2}_{k=1}
+ \sum_{k=2}^K \frac{1}{2}\Big(\underbrace{\|u-\hat\Delta_i^{t,k-1}\|_2^2}_{\text{substitute } \eqref{eq:bound_u_delta}} - \|u-\Delta_i^{t,k-1}\|_2^2 \Big) \notag \\
&\le
\frac{1}{2}\|u-\hat\Delta_i^{t,0}\|_2^2 - \frac{1}{2}\|u-\Delta_i^{t,0}\|_2^2 \notag \\
&\quad + \sum_{k=2}^K \Big( \|u-\Delta_i^{t,k-2}\|_2^2 + \|g_i^{t,k}-g_i^{t,k-1}\|_2^2 - \frac{1}{2}\|u-\Delta_i^{t,k-1}\|_2^2 \Big). \label{eq:sum_expansion}
\end{align}
Rearranging the distance terms in \eqref{eq:sum_expansion} reveals a telescoping sum:
\begin{align*}
&\quad -\frac{1}{2}\|u-\Delta_i^{t,0}\|_2^2 + \sum_{k=2}^K \|u-\Delta_i^{t,k-2}\|_2^2 - \sum_{k=2}^K \frac{1}{2}\|u-\Delta_i^{t,k-1}\|_2^2 \\
&= -\frac{1}{2}\|u-\Delta_i^{t,0}\|_2^2 + \Big(\|u-\Delta_i^{t,0}\|_2^2 + \sum_{j=1}^{K-2}\|u-\Delta_i^{t,j}\|_2^2 \Big) - \Big(\sum_{j=1}^{K-2}\frac{1}{2}\|u-\Delta_i^{t,j}\|_2^2 + \frac{1}{2}\|u-\Delta_i^{t,K-1}\|_2^2 \Big) \\
&= \frac{1}{2}\|u-\Delta_i^{t,0}\|_2^2 + \sum_{j=1}^{K-2}\frac{1}{2}\|u-\Delta_i^{t,j}\|_2^2 - \frac{1}{2}\|u-\Delta_i^{t,K-1}\|_2^2 \\
&\le \sum_{j=0}^{K-2} \frac{1}{2}(2\bar\rho\eta)^2 \le 2K(\bar\rho\eta)^2,
\end{align*}
where we used the diameter bound $\|u-\Delta_i^{t,j}\|_2 \le 2\bar\rho\eta$ and dropped the final negative term.
Substituting this back into \eqref{eq:sum_expansion}, and using the bound $\|u-\hat\Delta_i^{t,0}\|_2^2 \le \|u\|_2^2 + (2\bar\rho\eta)^2$, we obtain \eqref{eq:pc_key_new}.
\end{proof}

\begin{lemma}[Single-call approximate projection certificate]
\label{lem:single_call_projection_residual}
Let $\mathcal D=\{u:\|u\|^g\le \eta\}$ and let $z=x-g$.
Suppose an inexact projection call returns $y\in\mathcal D$ with residual
\begin{align}
\operatorname{res}_{\mathcal D}(z,y)
:=
\left[\sup_{u\in\mathcal D}\langle z-y,u-y\rangle\right]_+
\le \varepsilon.
\label{eq:single_call_residual_def}
\end{align}
Then, for every $u\in\mathcal D$,
\begin{align}
\langle g,y-u\rangle
\le
\frac12\Big(\|u-x\|_2^2-\|u-y\|_2^2-\|y-x\|_2^2\Big)
+\varepsilon.
\label{eq:single_call_approx_three_point}
\end{align}
Moreover, if $p=\Proj(z)$ is the exact Euclidean projection of $z$ onto $\mathcal D$, then
\begin{align}
\|y-p\|_2^2\le \varepsilon.
\label{eq:single_call_close_exact_projection}
\end{align}
Consequently, if $y_a,y_b$ are inexact projections of $z_a,z_b$ with residuals at most $\varepsilon_a,\varepsilon_b$, respectively, then
\begin{align}
\|y_a-y_b\|_2
\le
\|z_a-z_b\|_2+\sqrt{\varepsilon_a}+\sqrt{\varepsilon_b}.
\label{eq:approx_nonexpansive}
\end{align}
For the norm ball $\mathcal D=\{u:\|u\|^g\le \eta\}$, the residual in \eqref{eq:single_call_residual_def} can be evaluated as
\begin{align}
\operatorname{res}_{\mathcal D}(z,y)
=
\left[
\eta\|z-y\|^{g,*}-\langle z-y,y\rangle
\right]_+.
\label{eq:single_call_residual_closed_form}
\end{align}
\end{lemma}
\begin{proof}
The residual condition gives
$\langle z-y,u-y\rangle\le \varepsilon$ for all $u\in\mathcal D$.
Since $z=x-g$, this implies
\[
\langle (x-y)-g,u-y\rangle\le \varepsilon,
\]
and hence
\[
\langle g,y-u\rangle
\le
\langle x-y,u-y\rangle+\varepsilon.
\]
Using the identity
$2\langle x-y,u-y\rangle=\|u-x\|_2^2-\|u-y\|_2^2-\|y-x\|_2^2$
gives \eqref{eq:single_call_approx_three_point}.

Let $p=\Proj(z)$.
The exact projection optimality condition gives
$\langle z-p,y-p\rangle\le 0$ because $y\in\mathcal D$.
Taking $u=p$ in the approximate condition gives
$\langle z-y,p-y\rangle\le \varepsilon$, or equivalently
$\langle z-y,y-p\rangle\ge -\varepsilon$.
Therefore,
\[
\|y-p\|_2^2
=
\langle y-p,y-p\rangle
=
\langle z-p,y-p\rangle-
\langle z-y,y-p\rangle
\le \varepsilon,
\]
which proves \eqref{eq:single_call_close_exact_projection}.
The approximate non-expansiveness bound \eqref{eq:approx_nonexpansive} follows by the triangle inequality and the non-expansiveness of the exact projection:
\[
\|y_a-y_b\|_2
\le
\|y_a-\Proj(z_a)\|_2
+
\|\Proj(z_a)-\Proj(z_b)\|_2
+
\|\Proj(z_b)-y_b\|_2
\le
\sqrt{\varepsilon_a}+\|z_a-z_b\|_2+
\sqrt{\varepsilon_b}.
\]
Finally,
\[
\sup_{u\in\mathcal D}\langle z-y,u-y\rangle
=
\sup_{\|u\|^g\le \eta}\langle z-y,u\rangle-
\langle z-y,y\rangle
=
\eta\|z-y\|^{g,*}-\langle z-y,y\rangle,
\]
which proves \eqref{eq:single_call_residual_closed_form} after applying the positive part.
\end{proof}

\begin{lemma}[Inexact predict-and-correct from single-call residuals]
\label{lem:inexact_pc_from_single_call}
Fix $(t,i)$ and set $\hat\Delta_i^{t,-1}=0$.
Suppose the prediction and correction calls in the local trajectory return feasible points
$\hat\Delta_i^{t,k-1},\Delta_i^{t,k-1}\in\mathcal D$ for $k=1,\ldots,K$.
Let
\begin{align*}
\hat z_i^{t,k}
:=
\hat\Delta_i^{t,k-2}-g_i^{t,k},
\qquad
z_i^{t,k}
:=
\hat\Delta_i^{t,k-1}-g_i^{t,k},
\end{align*}
and define
\begin{align*}
\hat\varepsilon_i^{t,k}
:=
\left[
\sup_{v\in\mathcal D}
\left\langle \hat z_i^{t,k}-\hat\Delta_i^{t,k-1},v-\hat\Delta_i^{t,k-1}\right\rangle
\right]_+,
\qquad
\varepsilon_i^{t,k}
:=
\left[
\sup_{v\in\mathcal D}
\left\langle z_i^{t,k}-\Delta_i^{t,k-1},v-\Delta_i^{t,k-1}\right\rangle
\right]_+.
\end{align*}
Then for every $u\in\mathcal D$,
\begin{align}
\sum_{k=1}^K
\left\langle g_i^{t,k},\Delta_i^{t,k-1}-u\right\rangle
\le\;&
\frac{\|u\|_2^2}{2}
+2\sum_{k=2}^K
\|g_i^{t,k}-g_i^{t,k-1}\|_2^2
+2K(\bar\rho\eta)^2
+5\sum_{k=1}^K
\left(\hat\varepsilon_i^{t,k}+\varepsilon_i^{t,k}\right).
\label{eq:inexact_pc_single_call_bound}
\end{align}
\end{lemma}
\begin{proof}
For the correction call at step $k$, Lemma~\ref{lem:single_call_projection_residual} gives, with $x=\hat\Delta_i^{t,k-1}$ and $y=\Delta_i^{t,k-1}$,
\begin{align}
\left\langle g_i^{t,k},\Delta_i^{t,k-1}-u\right\rangle
\le
\frac12\left(
\|u-\hat\Delta_i^{t,k-1}\|_2^2
-
\|u-\Delta_i^{t,k-1}\|_2^2
\right)
+
\varepsilon_i^{t,k},
\label{eq:inexact_correction_three_point}
\end{align}
where we dropped the nonpositive term $-\|\Delta_i^{t,k-1}-\hat\Delta_i^{t,k-1}\|_2^2/2$.
For $k=1$, feasibility gives $\|u-\hat\Delta_i^{t,0}\|_2\le 2\bar\rho\eta$.
For $k\ge 2$, apply the approximate non-expansiveness bound in Lemma~\ref{lem:single_call_projection_residual} to the prediction call at step $k$ and the correction call at step $k-1$.
Their inputs differ by $g_i^{t,k}-g_i^{t,k-1}$, and hence
\begin{align}
\|\hat\Delta_i^{t,k-1}-\Delta_i^{t,k-2}\|_2
\le
\|g_i^{t,k}-g_i^{t,k-1}\|_2
+
\sqrt{\hat\varepsilon_i^{t,k}}
+
\sqrt{\varepsilon_i^{t,k-1}}.
\label{eq:inexact_pc_approx_nonexp_step}
\end{align}
Using $(a+b)^2\le 2a^2+2b^2$ and $(\sqrt a+\sqrt b)^2\le 2(a+b)$, we get
\begin{align}
\|u-\hat\Delta_i^{t,k-1}\|_2^2
&\le
2\|u-\Delta_i^{t,k-2}\|_2^2
+2\|\hat\Delta_i^{t,k-1}-\Delta_i^{t,k-2}\|_2^2
\notag\\
&\le
2\|u-\Delta_i^{t,k-2}\|_2^2
+4\|g_i^{t,k}-g_i^{t,k-1}\|_2^2
+8\left(\hat\varepsilon_i^{t,k}+\varepsilon_i^{t,k-1}\right).
\label{eq:inexact_u_hat_delta_bound}
\end{align}
Substituting \eqref{eq:inexact_u_hat_delta_bound} into \eqref{eq:inexact_correction_three_point} for $k\ge 2$, and treating $k=1$ separately, yields the same distance telescoping as in Lemma~\ref{lem:pc_app_new}, except that the gradient-variation term is multiplied by a universal constant and the residual terms are added.
The distance terms are bounded by $2K(\bar\rho\eta)^2$ using the diameter of $\mathcal D$.
The residual contribution is bounded by
\[
\sum_{k=1}^K\varepsilon_i^{t,k}
+4\sum_{k=2}^K\left(\hat\varepsilon_i^{t,k}+\varepsilon_i^{t,k-1}\right)
\le
5\sum_{k=1}^K\left(\hat\varepsilon_i^{t,k}+\varepsilon_i^{t,k}\right),
\]
which proves \eqref{eq:inexact_pc_single_call_bound}.
\end{proof}

\begin{lemma}[Bounding gradient variation]

\label{lem:gvar_app_new}
In Algorithm~\ref{alg:NBCLT-FL}, for any $(t,i,k\ge 2)$,
\begin{align}
\mathbb E\Big[\|g_i^{t,k}-g_i^{t,k-1}\|_2^2\Big]
\le
2(1-\beta_1)^2 L^2 \,\mathbb E\|x_i^{t,k-1}-x_i^{t,k-2}\|_2^2
+4(1-\beta_1)^2\sigma^2.
\label{eq:gvar_app_new}
\end{align}
\end{lemma}

\begin{proof}
Fix $(t,i,k\ge 2)$. Recall the client-side gradient signal in Algorithm~\ref{alg:NBCLT-FL}:
\[
 g_i^{t,k}
=(1-\beta_1)\big(\nabla f_i(x_i^{t,k-1},\xi_i^{t,k-1})-c_i^t\big)+m^t.
\]
Since $c_i^t$ and $m^t$ are fixed within round $t$, we have
\[
 g_i^{t,k}-g_i^{t,k-1}
=(1-\beta_1)\Big(\nabla f_i(x_i^{t,k-1},\xi_i^{t,k-1})
-\nabla f_i(x_i^{t,k-2},\xi_i^{t,k-2})\Big).
\]

Define the stochastic gradient noise
\[
\zeta_i^{t,k-1}
:=\nabla f_i(x_i^{t,k-1},\xi_i^{t,k-1})-\nabla f_i(x_i^{t,k-1}),
\qquad
\zeta_i^{t,k-2}
:=\nabla f_i(x_i^{t,k-2},\xi_i^{t,k-2})-\nabla f_i(x_i^{t,k-2}).
\]
Then
\[
\nabla f_i(x_i^{t,k-1},\xi_i^{t,k-1})
-\nabla f_i(x_i^{t,k-2},\xi_i^{t,k-2})
=
\big(\nabla f_i(x_i^{t,k-1})-\nabla f_i(x_i^{t,k-2})\big)
+\big(\zeta_i^{t,k-1}-\zeta_i^{t,k-2}\big).
\]
Using $\|a+b\|^2\le 2\|a\|^2+2\|b\|^2$, we obtain
\begin{align*}
\|g_i^{t,k}-g_i^{t,k-1}\|_2^2
&=(1-\beta_1)^2\Big\|
\big(\nabla f_i(x_i^{t,k-1})-\nabla f_i(x_i^{t,k-2})\big)
+\big(\zeta_i^{t,k-1}-\zeta_i^{t,k-2}\big)
\Big\|_2^2\\
&\le
2(1-\beta_1)^2\|\nabla f_i(x_i^{t,k-1})-\nabla f_i(x_i^{t,k-2})\|_2^2
+2(1-\beta_1)^2\|\zeta_i^{t,k-1}-\zeta_i^{t,k-2}\|_2^2.
\end{align*}

For the first term, by $L$-smoothness of $f_i$,
\[
\|\nabla f_i(x_i^{t,k-1})-\nabla f_i(x_i^{t,k-2})\|_2
\le
L\|x_i^{t,k-1}-x_i^{t,k-2}\|_2.
\]

For the noise term, conditioned on $\mathcal F_{t,k}$ (right before drawing $\xi_i^{t,k-1}$),
$\zeta_i^{t,k-1}$ is independent of the past and satisfies
$\mathbb E[\|\zeta_i^{t,k-1}\|_2^2\mid\mathcal F_{t,k}]\le \sigma^2$ by \eqref{eq:bv_new},
while $\zeta_i^{t,k-2}$ is $\mathcal F_{t,k}$-measurable. Hence
\begin{align*}
\mathbb E\big[\|\zeta_i^{t,k-1}-\zeta_i^{t,k-2}\|_2^2 \mid \mathcal F_{t,k}\big]
&\le
2\mathbb E\big[\|\zeta_i^{t,k-1}\|_2^2 \mid \mathcal F_{t,k}\big]
+2\|\zeta_i^{t,k-2}\|_2^2\\
&\le 2\sigma^2+2\|\zeta_i^{t,k-2}\|_2^2.
\end{align*}
Taking expectation again and using $\mathbb E\|\zeta_i^{t,k-2}\|_2^2\le \sigma^2$ yields
\[
\mathbb E\big[\|\zeta_i^{t,k-1}-\zeta_i^{t,k-2}\|_2^2\big]\le 4\sigma^2.
\]

Combining the above bounds gives \eqref{eq:gvar_app_new}.
\end{proof}

\subsection{Descent decomposition}

Define the ideal direction $u_t\in\arg\min_{u\in\mathcal D}\langle \nabla F(x^t), u\rangle$.
Using Lemma~\ref{lem:smooth_app_new} and the server update \eqref{eq:server_update_app_new},
we have
\begin{align}
F(x^{t+1})
&\le
F(x^t)
+\frac{1}{N}\sum_{i\in S_t}\sum_{k=1}^K \langle \nabla F(x^t), \Delta_i^{t,k-1}\rangle
+\frac{L}{2}(K\bar\rho\eta)^2 \notag\\
&=
F(x^t)
+K\langle \nabla F(x^t), u_t\rangle
+\underbrace{\frac{1}{N}\sum_{i\in S_t}\sum_{k=1}^K \langle \nabla F(x^t), \Delta_i^{t,k-1}-u_t\rangle}_{=:E_t}
+\frac{L}{2}(K\bar\rho\eta)^2.
\label{eq:descent_start_new}
\end{align}
By Lemma~\ref{lem:support_app_new}, $K\langle \nabla F(x^t), u_t\rangle=-\eta K\|\nabla F(x^t)\|^{g,*}$.
Rearranging \eqref{eq:descent_start_new} yields
\begin{equation}
\eta K\|\nabla F(x^t)\|^{g,*}
\le
F(x^t)-F(x^{t+1})
+\frac{L}{2}(K\bar\rho\eta)^2
+E_t.
\label{eq:descent_core_new}
\end{equation}

\subsection{Bounding \texorpdfstring{$E_t$}{E\_t} via predict-and-correct}
For convenience, set $x^0:=x^1$, and define the tracking error
\[
B_t := \frac{1}{M}\sum_{i=1}^M \mathbb E\big[\|c_i^t-\nabla f_i(x^{t-1})\|_2^2\big].
\]
Recall the client-side gradient signal in Algorithm~\ref{alg:NBCLT-FL}:
\[
 g_i^{t,k}
=(1-\beta_1)\big(\nabla f_i(x_i^{t,k-1},\xi_i^{t,k-1})-c_i^t\big)+m^t.
\]
Using $m^t=\frac1M\sum_{j=1}^M c_j^t$ and the identity
\[
\nabla F(x^t)
= m^t
+\frac1M\sum_{j=1}^M\big(\nabla f_j(x^t)-c_j^t\big),
\]
we decompose, for each $(i,k)$,
\begin{align*}
\langle \nabla F(x^t),\Delta_i^{t,k-1}-u_t\rangle
&=
\langle m^t,\Delta_i^{t,k-1}-u_t\rangle
+\Big\langle \frac1M\sum_{j=1}^M(\nabla f_j(x^t)-c_j^t),\,\Delta_i^{t,k-1}-u_t\Big\rangle\\
&=
\langle g_i^{t,k},\Delta_i^{t,k-1}-u_t\rangle
-(1-\beta_1)\langle \nabla f_i(x_i^{t,k-1},\xi_i^{t,k-1})-c_i^t,\,\Delta_i^{t,k-1}-u_t\rangle\\
&\quad
+\Big\langle \frac1M\sum_{j=1}^M(\nabla f_j(x^t)-c_j^t),\,\Delta_i^{t,k-1}-u_t\Big\rangle.
\end{align*}
Summing over $i\in S_t,k\in[K]$ and dividing by $N$, we obtain
\begin{equation}
E_t \;=\; E_t^{(a)} + E_t^{(b)} + E_t^{(c)},
\label{eq:Et_split_new}
\end{equation}
where
\begin{align}
E_t^{(a)}
&:=\frac{1}{N}\sum_{i\in S_t}\sum_{k=1}^K \langle g_i^{t,k},\Delta_i^{t,k-1}-u_t\rangle, \label{eq:Et_a_new}\\
E_t^{(b)}
&:=-(1-\beta_1)\cdot \frac{1}{N}\sum_{i\in S_t}\sum_{k=1}^K
\langle \nabla f_i(x_i^{t,k-1},\xi_i^{t,k-1})-c_i^t,\,\Delta_i^{t,k-1}-u_t\rangle,\label{eq:Et_b_new}\\
E_t^{(c)}
&:=\frac{1}{N}\sum_{i\in S_t}\sum_{k=1}^K
\Big\langle \frac1M\sum_{j=1}^M(\nabla f_j(x^t)-c_j^t),\,\Delta_i^{t,k-1}-u_t\Big\rangle. \label{eq:Et_c_new}
\end{align}

\paragraph{Step 1: bound $E_t^{(a)}$ (predict-and-correct).}
Applying Lemma~\ref{lem:pc_app_new} with $u=u_t$ and using $\|u_t\|_2\le \bar\rho\eta$, we get
\begin{equation}
E_t^{(a)}
\le
\frac{(\bar\rho\eta)^2}{2}
+\frac{1}{2N}\sum_{i\in S_t}\sum_{k=2}^K \|g_i^{t,k}-g_i^{t,k-1}\|_2^2
+2K(\bar\rho\eta)^2.
\label{eq:Et_a_bound_new}
\end{equation}
Taking expectation and using Lemma~\ref{lem:gvar_app_new} together with
$\|x_i^{t,k-1}-x_i^{t,k-2}\|_2=\|\Delta_i^{t,k-2}\|_2\le \bar\rho\eta$ yields
\begin{align}
\mathbb E[E_t^{(a)}]
\le
\Big(\frac{1}{2}+2K\Big)(\bar\rho\eta)^2
+\frac{K-1}{2}\Big(2(1-\beta_1)^2 L^2 (\bar\rho\eta)^2 +4(1-\beta_1)^2\sigma^2\Big).
\label{eq:Et_a_bound2_new}
\end{align}
Dividing by $\eta K$ gives
\begin{align}
\frac{\mathbb E[E_t^{(a)}]}{\eta K}
&\le
\Big(2+\frac{1}{2K}\Big)\bar\rho^2\eta
+(1-\beta_1)^2L^2\bar\rho^2\eta
+\frac{2(1-\beta_1)^2\sigma^2}{\eta}.
\label{eq:Et_a_over_etaK_unit}
\end{align}

\paragraph{Step 2: bound $E_t^{(b)}$ (tracking + noise).}
Define the stochastic gradient noise for each query $(t,i,k)$:
\[
\zeta_i^{t,k-1}
:=\nabla f_i(x_i^{t,k-1},\xi_i^{t,k-1})-\nabla f_i(x_i^{t,k-1}),
\qquad
\mathbb E[\zeta_i^{t,k-1}\mid \mathcal F_{t,k}]=0,\ 
\mathbb E[\|\zeta_i^{t,k-1}\|_2^2\mid \mathcal F_{t,k}]\le \sigma^2.
\]
Then
\[
\nabla f_i(x_i^{t,k-1},\xi_i^{t,k-1})-c_i^t
=
\big(\nabla f_i(x_i^{t,k-1})-c_i^t\big)+\zeta_i^{t,k-1}.
\]
Since $\Delta_i^{t,k-1}\in\mathcal D$ and $u_t\in\mathcal D$, by \eqref{eq:D_bound_new},
\begin{equation}
\|\Delta_i^{t,k-1}-u_t\|_2\le \|\Delta_i^{t,k-1}\|_2+\|u_t\|_2\le 2\bar\rho\eta.
\label{eq:Du_bound_step2_new}
\end{equation}
Using Cauchy--Schwarz, \eqref{eq:Et_b_new}, and \eqref{eq:Du_bound_step2_new}, we obtain
\begin{align}
\mathbb E[E_t^{(b)}]
&\le
(1-\beta_1)\cdot \frac{2\bar\rho\eta}{N}\sum_{i\in S_t}\sum_{k=1}^K
\mathbb E\big[\|\nabla f_i(x_i^{t,k-1})-c_i^t\|_2\big]
\;+
(1-\beta_1)\cdot \frac{2\bar\rho\eta}{N}
\sum_{i\in S_t}\sum_{k=1}^K\mathbb E\|\zeta_i^{t,k-1}\|_2.
\label{eq:Etb_split_new}
\end{align}

\paragraph{Noise term in \eqref{eq:Etb_split_new}.}
The projected multiplier $\Delta_i^{t,k-1}-u_t$ depends on the current query noise, so martingale cancellation cannot be applied to these weighted inner products. Instead, apply Cauchy--Schwarz separately to each query. Conditional Jensen and Assumption~\ref{assump:stoch} give $\mathbb E\|\zeta_i^{t,k-1}\|_2\le\sigma$. Together with the pathwise diameter bound, this yields
\[
-\frac{1-\beta_1}{N}\sum_{i\in S_t}\sum_{k=1}^K
\mathbb E\langle\zeta_i^{t,k-1},\Delta_i^{t,k-1}-u_t\rangle
\le 2\bar\rho\eta(1-\beta_1)K\sigma.
\]
Thus \eqref{eq:Etb_split_new} has the normalized noise contribution
\begin{equation}
\frac{2\bar\rho\eta(1-\beta_1)K\sigma}{\eta K}
=
2\bar\rho(1-\beta_1)\sigma.
\label{eq:Etb_noise_norm_new}
\end{equation}
This argument uses only feasibility and the conditional second-moment bound; it does not require the projected multiplier to be independent of the current noise.

\paragraph{Tracking + drift term in \eqref{eq:Etb_split_new}.}
For the first term in \eqref{eq:Etb_split_new}, use triangle inequality and smoothness:
\[
\|\nabla f_i(x_i^{t,k-1})-c_i^t\|_2
\le
\|\nabla f_i(x^{t-1})-c_i^t\|_2
+
L\|x_i^{t,k-1}-x^{t-1}\|_2.
\]
Moreover, by \eqref{eq:move_round_new},
\[
\|x_i^{t,k-1}-x^{t-1}\|_2
\le
\|x^t-x^{t-1}\|_2+\|x_i^{t,k-1}-x^t\|_2
\le
K\bar\rho\eta+(k-1)\bar\rho\eta
=
(K+k-1)\bar\rho\eta.
\]
Hence
\[
\sum_{k=1}^K \|x_i^{t,k-1}-x^{t-1}\|_2
\le
\bar\rho\eta\sum_{k=1}^K(K+k-1)
=
\bar\rho\eta\cdot \frac{K(3K-1)}{2}.
\]
Therefore,
\begin{align}
\frac{1}{N}\sum_{i\in S_t}\sum_{k=1}^K
\mathbb E\big[\|\nabla f_i(x_i^{t,k-1})-c_i^t\|_2\big]
&\le
K\cdot \mathbb E\Big[\frac{1}{N}\sum_{i\in S_t}\|\nabla f_i(x^{t-1})-c_i^t\|_2\Big]
+\frac{L\bar\rho\eta}{2}\,K(3K-1).
\label{eq:Etb_track_reduce_new}
\end{align}
Finally, by Jensen, Cauchy--Schwarz, and the definition of $B_t$,
\begin{align}
\mathbb E\Big[\frac{1}{N}\sum_{i\in S_t}\|\nabla f_i(x^{t-1})-c_i^t\|_2\Big]
&\le
\sqrt{\mathbb E\Big[\Big(\frac{1}{N}\sum_{i\in S_t}\|\nabla f_i(x^{t-1})-c_i^t\|_2\Big)^2\Big]} \notag\\
&\le
\sqrt{\mathbb E\Big[\frac{1}{N}\sum_{i\in S_t}\|\nabla f_i(x^{t-1})-c_i^t\|_2^2\Big]}
=
\sqrt{B_t}.
\label{eq:Etb_sample_to_Bt_new}
\end{align}
Combining \eqref{eq:Etb_split_new}--\eqref{eq:Etb_sample_to_Bt_new} and dividing by $\eta K$, we conclude
\begin{equation}
\frac{\mathbb E[E_t^{(b)}]}{\eta K}
\le
2\bar\rho(1-\beta_1)\sqrt{B_t}
\;+
\Big(3-\frac1K\Big)\bar\rho^2L(1-\beta_1)(K\eta)
\;+
2\bar\rho(1-\beta_1)\sigma.
\label{eq:Etb_over_etaK_new}
\end{equation}

\paragraph{Step 3: bound $E_t^{(c)}$ (global control mismatch).}
Decompose
\[
\frac1M\sum_{j=1}^M(\nabla f_j(x^t)-c_j^t)
=
\frac1M\sum_{j=1}^M(\nabla f_j(x^t)-\nabla f_j(x^{t-1}))
+
\frac1M\sum_{j=1}^M(\nabla f_j(x^{t-1})-c_j^t).
\]
Accordingly write $E_t^{(c)}=E_{t,1}^{(c)}+E_{t,2}^{(c)}$, where
\begin{align*}
E_{t,1}^{(c)}
&:=
\frac{1}{N}\sum_{i\in S_t}\sum_{k=1}^K
\Big\langle \frac1M\sum_{j=1}^M(\nabla f_j(x^t)-\nabla f_j(x^{t-1})),\,\Delta_i^{t,k-1}-u_t\Big\rangle,\\
E_{t,2}^{(c)}
&:=
\frac{1}{N}\sum_{i\in S_t}\sum_{k=1}^K
\Big\langle \frac1M\sum_{j=1}^M(\nabla f_j(x^{t-1})-c_j^t),\,\Delta_i^{t,k-1}-u_t\Big\rangle.
\end{align*}

For $E_{t,1}^{(c)}$, by $L$-smoothness, \eqref{eq:move_round_new}, and \eqref{eq:Du_bound_step2_new},
\[
\Big\|\frac1M\sum_{j=1}^M(\nabla f_j(x^t)-\nabla f_j(x^{t-1}))\Big\|_2
\le
L\|x^t-x^{t-1}\|_2
\le
LK\bar\rho\eta.
\]
Therefore,
\[
|E_{t,1}^{(c)}|
\le
\frac{1}{N}\sum_{i\in S_t}\sum_{k=1}^K
LK\bar\rho\eta \cdot \|\Delta_i^{t,k-1}-u_t\|_2
\le
2L\bar\rho^2K^2\eta^2,
\]
and hence
\begin{equation}
\frac{\mathbb E[E_{t,1}^{(c)}]}{\eta K}
\le
2L\bar\rho^2(K\eta).
\label{eq:Etc1_over_etaK_new}
\end{equation}

To sharpen the stochastic part of $E_{t,2}^{(c)}$, we split the control error into a bias part and a noise part.
Let $\gamma:=1-\beta_2$ and define, on an enlarged probability space,
\[
h_i^t:=\nabla f_i(x^{t-1}),
\qquad
e_i^t:=c_i^t-h_i^t,
\qquad
d_i^t:=\nabla f_i(x^t)-\nabla f_i(x^{t-1}).
\]
For every client $i\in[M]$, let
\[
b_i^t:=\frac1K\sum_{k=1}^K\big(\nabla f_i(x_i^{t,k-1})-\nabla f_i(x^{t-1})\big),
\qquad
n_i^t:=\frac1K\sum_{k=1}^K\zeta_i^{t,k-1},
\]
where inactive clients are equipped with virtual local trajectories and virtual stochastic gradients.
Now set
\[
r_i^1:=0,
\qquad
s_i^1:=c_i^1-\nabla f_i(x^0),
\]
and for $t\ge 1$,
\begin{align}
r_i^{t+1}
&:=(1-\gamma I_i^t)r_i^t+\gamma I_i^t b_i^t-d_i^t,
\label{eq:r_rec_new}\\
s_i^{t+1}
&:=(1-\gamma I_i^t)s_i^t+\gamma I_i^t n_i^t,
\label{eq:s_rec_new}
\end{align}
where $I_i^t:=\mathbf 1\{i\in S_t\}$. By construction,
\[
e_i^t=r_i^t+s_i^t
\qquad\text{for all } i,t.
\]
Define the averaged bias and averaged noise as
\[
r^t:=\frac1M\sum_{i=1}^M r_i^t,
\qquad
v^t:=\frac1M\sum_{i=1}^M s_i^t,
\]
and their second moments as
\[
R_t:=\frac1M\sum_{i=1}^M \mathbb E\|r_i^t\|_2^2,
\qquad
V_t:=\mathbb E\|v^t\|_2^2.
\]
Then
\[
\frac1M\sum_{j=1}^M(\nabla f_j(x^{t-1})-c_j^t)
=-\frac1M\sum_{j=1}^M e_j^t
=-(r^t+v^t).
\]
Hence
\[
E_{t,2}^{(c)}
=
-\frac{1}{N}\sum_{i\in S_t}\sum_{k=1}^K
\langle r^t+v^t,\Delta_i^{t,k-1}-u_t\rangle.
\]
Using Cauchy--Schwarz and \eqref{eq:Du_bound_step2_new},
\[
|E_{t,2}^{(c)}|
\le
\frac{1}{N}\sum_{i\in S_t}\sum_{k=1}^K
\big(\|r^t\|_2+\|v^t\|_2\big)\cdot \|\Delta_i^{t,k-1}-u_t\|_2
\le
2\bar\rho\eta K\big(\|r^t\|_2+\|v^t\|_2\big).
\]
Dividing by $\eta K$ and taking expectation gives
\begin{equation}
\frac{\mathbb E[E_{t,2}^{(c)}]}{\eta K}
\le
2\bar\rho\,\mathbb E\|r^t\|_2
+2\bar\rho\,\mathbb E\|v^t\|_2
\le
2\bar\rho\sqrt{R_t}
+2\bar\rho\sqrt{V_t},
\label{eq:Etc2_over_etaK_new}
\end{equation}
where we used
$\|r^t\|_2^2\le \frac1M\sum_i\|r_i^t\|_2^2$ and Jensen's inequality.
Combining \eqref{eq:Etc1_over_etaK_new} and \eqref{eq:Etc2_over_etaK_new}, we conclude
\begin{equation}
\frac{\mathbb E[E_t^{(c)}]}{\eta K}
\le
2L\bar\rho^2(K\eta) + 2\bar\rho\sqrt{R_t} + 2\bar\rho\sqrt{V_t}.
\label{eq:Etc_total_over_etaK_new}
\end{equation}

\subsection{Control-variates recursion and averaging bounds}
We initialize the client control variates using a single stochastic gradient query per client:
for each $i\in[M]$, draw an independent mini-batch $\xi_i^{0}$ and set
\[
 c_i^{1} := \nabla f_i(x^1,\xi_i^{0}),
\qquad
 m^{1}:=\frac{1}{M}\sum_{i=1}^M c_i^{1}.
\]
This initialization requires only one mini-batch per client and will be sufficient for our rates.

\begin{lemma}[One-batch initialization bound]
\label{lem:B1_init}
Under \eqref{eq:bv_new}, the above initialization satisfies
\[
B_1 := \frac{1}{M}\sum_{i=1}^M \mathbb E\big[\|c_i^1-\nabla f_i(x^0)\|_2^2\big]
\le \sigma^2.
\]
\end{lemma}
\begin{proof}
Since $x^0=x^1$ by convention,
\[
 c_i^1-\nabla f_i(x^0)
=\nabla f_i(x^1,\xi_i^{0})-\nabla f_i(x^1).
\]
Applying the bounded conditional variance in \eqref{eq:bv_new} and taking expectation yields
$\mathbb E\|c_i^1-\nabla f_i(x^0)\|_2^2\le \sigma^2$ for each $i$.
Averaging over $i=1,\dots,M$ gives $B_1\le \sigma^2$.
\end{proof}

\begin{lemma}[Clientwise control recursion and average $\sqrt{B_t}$ bound]
\label{lem:Bt_app_new}
Assume each $f_i$ is $L$-smooth and \eqref{eq:bv_new} holds.
Let $p:=\frac{N}{M}$ and
\[
B_t:=\frac1M\sum_{i=1}^M \mathbb E\big\|c_i^t-\nabla f_i(x^{t-1})\big\|_2^2.
\]
Then for all $t\ge 1$,
\begin{equation}
B_{t+1}
\le
\Big(1-\frac{p(1-\beta_2)}{4}\Big)B_t
\;+
\frac{32L^2(K\bar\rho\eta)^2}{p(1-\beta_2)}
\;+
\frac{8(1-\beta_2)^2\sigma^2}{K}.
\label{eq:Bt_rec_new}
\end{equation}
Consequently, for any $T\ge 1$,
\begin{equation}
\frac{1}{T}\sum_{t=1}^T \sqrt{B_t}
\le
 \frac{8\sqrt{2}LK\bar\rho\eta}{p(1-\beta_2)}
 + \frac{4\sqrt{2}\sigma\sqrt{1-\beta_2}}{\sqrt{Kp}}
 + \frac{8\sigma}{p(1-\beta_2)T}.
\label{eq:avg_sqrtBt_new}
\end{equation}
\end{lemma}

\begin{proof}
Fix $t\ge 1$ and let $\gamma:=1-\beta_2$.
Define
\[
h_i^t:=\nabla f_i(x^{t-1}),
\qquad
e_i^t:=c_i^t-h_i^t,
\qquad
d_i^t:=\nabla f_i(x^t)-\nabla f_i(x^{t-1}),
\]
so that $B_t=\frac1M\sum_{i=1}^M \mathbb E\|e_i^t\|_2^2$.
Let $I_i^t:=\mathbf 1\{i\in S_t\}$, so $\mathbb E[I_i^t]=p$ and $(I_i^t)^2=I_i^t$.

\paragraph{Step 1: error dynamics.}
The control update is
\[
c_i^{t+1}=c_i^t + I_i^t \gamma\big(\hat c_i^t-c_i^t\big),
\qquad
\hat c_i^t=\frac1K\sum_{k=1}^K \nabla f_i(x_i^{t,k-1},\xi_i^{t,k-1}).
\]
Subtracting $\nabla f_i(x^t)=h_i^t+d_i^t$, we obtain
\begin{equation}
e_i^{t+1}
=
\big(1-\gamma I_i^t\big)e_i^t
+
\gamma I_i^t\big(\hat c_i^t-h_i^t\big)
-d_i^t.
\label{eq:error_dyn_pf_fix}
\end{equation}

\paragraph{Step 2: decompose $\hat c_i^t-h_i^t$.}
On the event $\{I_i^t=1\}$, let
\[
\zeta_i^{t,k-1}:=\nabla f_i(x_i^{t,k-1},\xi_i^{t,k-1})-\nabla f_i(x_i^{t,k-1}),
\qquad
n_i^t:=\frac1K\sum_{k=1}^K \zeta_i^{t,k-1},
\]
and
\[
b_i^t:=\frac1K\sum_{k=1}^K\big(\nabla f_i(x_i^{t,k-1})-\nabla f_i(x^{t-1})\big).
\]
Then
\[
\hat c_i^t-h_i^t = b_i^t+n_i^t.
\]
By the martingale-difference property of $\{\zeta_i^{t,k-1}\}_{k=1}^K$,
\[
\mathbb E[n_i^t\mid \mathcal F_{t,1}, I_i^t=1]=0,
\qquad
\mathbb E[\|n_i^t\|_2^2\mid I_i^t=1]\le \frac{\sigma^2}{K}.
\]
Moreover, by $L$-smoothness and
\[
\|x_i^{t,k-1}-x^{t-1}\|_2
\le
\|x^t-x^{t-1}\|_2+\|x_i^{t,k-1}-x^t\|_2
\le
(K+k-1)\bar\rho\eta
\le
2K\bar\rho\eta,
\]
we have
\begin{equation}
\|b_i^t\|_2^2
\le
\frac1K\sum_{k=1}^K L^2\|x_i^{t,k-1}-x^{t-1}\|_2^2
\le
4L^2(K\bar\rho\eta)^2.
\label{eq:bi_bound_pf_fix}
\end{equation}

\paragraph{Step 3: active case $I_i^t=1$.}
Condition on $\mathcal F_{t,1}$ and the event $\{I_i^t=1\}$.
Then, by \eqref{eq:error_dyn_pf_fix},
\[
e_i^{t+1}
=
\beta_2 e_i^t+\gamma b_i^t+\gamma n_i^t-d_i^t.
\]
Since $e_i^t$ and $d_i^t$ are $\mathcal F_{t,1}$-measurable and $\mathbb E[n_i^t\mid\mathcal F_{t,1}]=0$,
\begin{align}
\mathbb E\!\left[\|e_i^{t+1}\|_2^2 \mid \mathcal F_{t,1}, I_i^t=1\right]
&\le
\Big(1-\frac{\gamma}{2}\Big)\|e_i^t\|_2^2
+3\gamma\,\mathbb E\!\left[\|b_i^t\|_2^2 \mid \mathcal F_{t,1}, I_i^t=1\right]
\notag\\
&\quad
+2\gamma^2\,\mathbb E\!\left[\|n_i^t\|_2^2 \mid \mathcal F_{t,1}, I_i^t=1\right]
+\frac{3}{\gamma}\|d_i^t\|_2^2.
\label{eq:case1_bound}
\end{align}

\paragraph{Step 4: inactive case $I_i^t=0$, then average over $I_i^t$.}
If $I_i^t=0$, then $c_i^{t+1}=c_i^t$, so by \eqref{eq:error_dyn_pf_fix},
\[
e_i^{t+1}=e_i^t-d_i^t.
\]
Using Young's inequality with parameter $p\gamma/4$, we get
\begin{equation}
\|e_i^{t+1}\|_2^2
=
\|e_i^t-d_i^t\|_2^2
\le
\Big(1+\frac{p\gamma}{4}\Big)\|e_i^t\|_2^2
+\Big(1+\frac{4}{p\gamma}\Big)\|d_i^t\|_2^2.
\label{eq:case0_bound}
\end{equation}

Now take expectation over $I_i^t$ and use $\mathbb P(I_i^t=1)=p$.
Applying \eqref{eq:case1_bound}, \eqref{eq:case0_bound}, and
$\mathbb E\|n_i^t\|_2^2\le \sigma^2/K$, we obtain
\begin{align}
\mathbb E\!\left[\|e_i^{t+1}\|_2^2\right]
&\le
\Big(1-\frac{p\gamma}{4}\Big)\mathbb E\|e_i^t\|_2^2
+3\,\mathbb E[\|b_i^t\|_2^2\mid I_i^t=1]
+\frac{8\gamma^2\sigma^2}{K}
+\frac{8}{p\gamma}\,\mathbb E\|d_i^t\|_2^2,
\label{eq:pre_additive_pf_fix}
\end{align}
where we used $p\gamma\le 1$.

\paragraph{Step 5: bound $b_i^t$ and $d_i^t$, and average over $i$.}
By \eqref{eq:bi_bound_pf_fix},
\[
\mathbb E[\|b_i^t\|_2^2\mid I_i^t=1]\le 4L^2(K\bar\rho\eta)^2.
\]
Also by $L$-smoothness and \eqref{eq:move_round_new},
\[
\|d_i^t\|_2^2
=
\|\nabla f_i(x^t)-\nabla f_i(x^{t-1})\|_2^2
\le
L^2\|x^t-x^{t-1}\|_2^2
\le
L^2(K\bar\rho\eta)^2.
\]
Substituting these bounds into \eqref{eq:pre_additive_pf_fix} and using $p\gamma\le 1$ yields
\[
\mathbb E\|e_i^{t+1}\|_2^2
\le
\Big(1-\frac{p\gamma}{4}\Big)\mathbb E\|e_i^t\|_2^2
+\frac{32L^2(K\bar\rho\eta)^2}{p\gamma}
+\frac{8\gamma^2\sigma^2}{K}.
\]
Averaging over $i=1,\dots,M$ gives \eqref{eq:Bt_rec_new}.

\paragraph{Step 6: average $\sqrt{B_t}$.}
Let $\alpha:=\frac{p\gamma}{4}$ and
\[
C:=\frac{32L^2(K\bar\rho\eta)^2}{p\gamma}+\frac{8\gamma^2\sigma^2}{K}.
\]
Then $B_{t+1}\le (1-\alpha)B_t+C$ implies
$B_t\le (1-\alpha)^{t-1}B_1+\frac{C}{\alpha}$.
Thus $\sqrt{B_t}\le (1-\alpha)^{\frac{t-1}{2}}\sqrt{B_1}+\sqrt{\frac{C}{\alpha}}$ and
\[
\frac1T\sum_{t=1}^T\sqrt{B_t}
\le
\sqrt{\frac{C}{\alpha}}+\frac{2\sqrt{B_1}}{\alpha T}.
\]
Using $B_1\le \sigma^2$ (one-batch init) gives $\sqrt{B_1}\le \sigma$.
Moreover,
\[
\sqrt{\frac{C}{\alpha}}
\le
\sqrt{\frac{32L^2(K\bar\rho\eta)^2}{p\gamma\cdot \alpha}}
+\sqrt{\frac{8\gamma^2\sigma^2/K}{\alpha}}
=
\frac{8\sqrt{2}LK\bar\rho\eta}{p\gamma}
+\frac{4\sqrt{2}\sigma\sqrt{\gamma}}{\sqrt{Kp}}.
\]
Finally, $\frac{2\sqrt{B_1}}{\alpha T}\le \frac{8\sigma}{p\gamma T}$.
Recalling $\gamma=1-\beta_2$ proves \eqref{eq:avg_sqrtBt_new}.
\end{proof}

\begin{lemma}[Bias recursion and average $\sqrt{R_t}$ bound]
\label{lem:Rt_app_new}
Assume each $f_i$ is $L$-smooth.
Let $R_t:=\frac1M\sum_{i=1}^M \mathbb E\|r_i^t\|_2^2$, where $\{r_i^t\}$ is defined in \eqref{eq:r_rec_new}.
Then for all $t\ge 1$,
\begin{equation}
R_{t+1}
\le
\Big(1-\frac{p(1-\beta_2)}{4}\Big)R_t
+\frac{32L^2(K\bar\rho\eta)^2}{p(1-\beta_2)}.
\label{eq:Rt_rec_new}
\end{equation}
Consequently,
\begin{equation}
\frac{1}{T}\sum_{t=1}^T \sqrt{R_t}
\le
\frac{8\sqrt{2}LK\bar\rho\eta}{p(1-\beta_2)}.
\label{eq:avg_sqrtRt_new}
\end{equation}
\end{lemma}
\begin{proof}
Fix $t\ge 1$ and let $\gamma:=1-\beta_2$.
From \eqref{eq:r_rec_new}, on the event $\{I_i^t=1\}$ we have
\[
r_i^{t+1}=\beta_2 r_i^t+\gamma b_i^t-d_i^t,
\]
while on the event $\{I_i^t=0\}$ we have
\[
r_i^{t+1}=r_i^t-d_i^t.
\]
The proof now follows exactly the same steps as in Lemma~\ref{lem:Bt_app_new}, except that the stochastic term $n_i^t$ is absent.
Using \eqref{eq:bi_bound_pf_fix} and the bound
$\|d_i^t\|_2^2\le L^2(K\bar\rho\eta)^2$, one obtains
\[
\mathbb E\|r_i^{t+1}\|_2^2
\le
\Big(1-\frac{p\gamma}{4}\Big)\mathbb E\|r_i^t\|_2^2
+\frac{32L^2(K\bar\rho\eta)^2}{p\gamma}.
\]
Averaging over $i$ proves \eqref{eq:Rt_rec_new}.
Since $r_i^1=0$ for all $i$, we have $R_1=0$.
Applying the same geometric-series argument as in Step 6 of Lemma~\ref{lem:Bt_app_new} yields
\eqref{eq:avg_sqrtRt_new}.
\end{proof}

\begin{lemma}[Server-average noise recursion and average $\sqrt{V_t}$ bound]
\label{lem:Vt_app_new}
Assume \eqref{eq:bv_new} holds.
Let
\[
Q_t:=\frac1M\sum_{i=1}^M \mathbb E\|s_i^t\|_2^2,
\qquad
v^t:=\frac1M\sum_{i=1}^M s_i^t,
\qquad
V_t:=\mathbb E\|v^t\|_2^2,
\]
where $\{s_i^t\}$ is defined in \eqref{eq:s_rec_new}.
Then, for all $t\ge 1$,
\begin{align}
Q_{t+1}
&\le
\big(1-p(1-\beta_2)\big)Q_t
+\frac{p(1-\beta_2)^2\sigma^2}{K},
\label{eq:St_rec_new}\\
V_{t+1}
&\le
\Big(1-\frac{p(1-\beta_2)}{2}\Big)V_t
+\frac{p(1-\beta_2)^2}{M}Q_t
+\frac{p(1-\beta_2)^2\sigma^2}{KM}.
\label{eq:Vt_rec_new}
\end{align}
Moreover,
\begin{equation}
V_1
=\mathbb E\Big\|\frac1M\sum_{i=1}^M \big(c_i^1-\nabla f_i(x^0)\big)\Big\|_2^2
\le \frac{\sigma^2}{M}.
\label{eq:V1_init_new}
\end{equation}
Consequently, for any $T\ge 1$,
\begin{equation}
\frac{1}{T}\sum_{t=1}^T \sqrt{V_t}
\le
\frac{4\sigma}{\sqrt{N(1-\beta_2)T}}
+\frac{4\sigma\sqrt{1-\beta_2}}{\sqrt{KM}}.
\label{eq:avg_sqrtVt_new}
\end{equation}
\end{lemma}
\begin{proof}
Fix $t\ge 1$ and let $\gamma:=1-\beta_2$.

\paragraph{Step 1: recursion for $Q_t$.}
From \eqref{eq:s_rec_new}, on the event $\{I_i^t=1\}$ we have
\[
s_i^{t+1}=\beta_2 s_i^t+\gamma n_i^t,
\]
while on the event $\{I_i^t=0\}$ we have $s_i^{t+1}=s_i^t$.
Conditioned on $\mathcal F_{t,1}$ and $I_i^t=1$, the cross term vanishes because
$\mathbb E[n_i^t\mid \mathcal F_{t,1}, I_i^t=1]=0$. Hence
\[
\mathbb E\big[\|s_i^{t+1}\|_2^2\mid \mathcal F_{t,1}, I_i^t=1\big]
=\beta_2^2\|s_i^t\|_2^2+\gamma^2\mathbb E\big[\|n_i^t\|_2^2\mid I_i^t=1\big]
\le (1-\gamma)\|s_i^t\|_2^2+\frac{\gamma^2\sigma^2}{K},
\]
where we used $\beta_2^2\le \beta_2=1-\gamma$.
Taking expectation over $I_i^t$ yields
\[
\mathbb E\|s_i^{t+1}\|_2^2
\le
\big(1-p\gamma\big)\mathbb E\|s_i^t\|_2^2
+\frac{p\gamma^2\sigma^2}{K}.
\]
Averaging over $i=1,\dots,M$ proves \eqref{eq:St_rec_new}.

\paragraph{Step 2: recursion for $V_t$.}
Averaging \eqref{eq:s_rec_new} over $i$ gives
\[
v^{t+1}
=
(1-p\gamma)v^t
+\underbrace{\frac{\gamma}{M}\sum_{i=1}^M (I_i^t-p)s_i^t}_{=:~\delta_t}
+\underbrace{\frac{\gamma}{M}\sum_{i=1}^M I_i^t n_i^t}_{=:~\xi_t}.
\]
Conditioned on the past, $\mathbb E[\delta_t\mid \mathcal F_{t,1}]=0$ and
$\mathbb E[\xi_t\mid \mathcal F_{t,1},\{I_i^t\}_{i=1}^M]=0$.
Moreover, the cross terms vanish after taking conditional expectation.
Therefore,
\begin{align}
\mathbb E\big[\|v^{t+1}\|_2^2\mid \mathcal F_{t,1}\big]
&=
(1-p\gamma)^2\|v^t\|_2^2
+\mathbb E\big[\|\delta_t\|_2^2\mid \mathcal F_{t,1}\big]
+\mathbb E\big[\|\xi_t\|_2^2\mid \mathcal F_{t,1}\big].
\label{eq:vt_expand_new}
\end{align}
For the sampling fluctuation term, using the fact that for uniform sampling without replacement the cross-covariances satisfy
$\mathbb E[(I_i^t-p)(I_j^t-p)]\le 0$ for $i\neq j$, we have
\begin{align}
\mathbb E\big[\|\delta_t\|_2^2\mid \mathcal F_{t,1}\big]
&\le
\frac{\gamma^2}{M^2}\sum_{i=1}^M \mathbb E\big[(I_i^t-p)^2\big]\,\|s_i^t\|_2^2
\le
\frac{p\gamma^2}{M}\cdot \frac1M\sum_{i=1}^M \|s_i^t\|_2^2.
\label{eq:delta_var_new}
\end{align}
For the averaged stochastic noise term, conditioned on the sampled set and the past,
$\{n_i^t\}_{i\in S_t}$ are zero-mean and independent across clients. Hence
\begin{align}
\mathbb E\big[\|\xi_t\|_2^2\mid \mathcal F_{t,1}\big]
&\le
\frac{\gamma^2}{M^2}\sum_{i=1}^M \mathbb E\big[I_i^t\|n_i^t\|_2^2\mid \mathcal F_{t,1}\big]
\le
\frac{p\gamma^2\sigma^2}{KM}.
\label{eq:xi_var_new}
\end{align}
Substituting \eqref{eq:delta_var_new} and \eqref{eq:xi_var_new} into \eqref{eq:vt_expand_new}, then taking expectation, yields
\[
V_{t+1}
\le
(1-p\gamma)^2V_t
+\frac{p\gamma^2}{M}Q_t
+\frac{p\gamma^2\sigma^2}{KM}.
\]
Using $(1-a)^2\le 1-a/2$ for $a\in[0,1]$ gives \eqref{eq:Vt_rec_new}.

\paragraph{Step 3: average $\sqrt{V_t}$.}
From \eqref{eq:St_rec_new}, summing the recursion for $Q_t$ over $t$ gives
\[
\frac{1}{T}\sum_{t=1}^T Q_t
\le
\frac{Q_1}{p\gamma T} + \frac{\gamma\sigma^2}{K}.
\]
Since $Q_1\le \sigma^2$ by the one-batch initialization, we obtain
\begin{equation}
\frac{1}{T}\sum_{t=1}^T Q_t
\le
\frac{\sigma^2}{p\gamma T} + \frac{\gamma\sigma^2}{K}.
\label{eq:avg_St_new}
\end{equation}
Now sum \eqref{eq:Vt_rec_new} over $t=1,\dots,T$. Since
$V_{t+1}-V_t\le -\frac{p\gamma}{2}V_t + \frac{p\gamma^2}{M}Q_t + \frac{p\gamma^2\sigma^2}{KM}$,
we get
\[
\frac{p\gamma}{2}\sum_{t=1}^T V_t
\le
V_1 + \frac{p\gamma^2}{M}\sum_{t=1}^T Q_t + \frac{Tp\gamma^2\sigma^2}{KM}.
\]
Dividing by $T$ and using \eqref{eq:avg_St_new} yields
\begin{align*}
\frac{1}{T}\sum_{t=1}^T V_t
&\le
\frac{2V_1}{p\gamma T}
+\frac{2\gamma}{M}\Big(\frac{\sigma^2}{p\gamma T} + \frac{\gamma\sigma^2}{K}\Big)
+\frac{2\gamma\sigma^2}{KM}\\
&\le
\frac{2\sigma^2/M}{p\gamma T}
+\frac{2\sigma^2}{MpT}
+\frac{2\gamma^2\sigma^2}{KM}
+\frac{2\gamma\sigma^2}{KM}\\
&\le
\frac{4\sigma^2}{N\gamma T}
+\frac{4\gamma\sigma^2}{KM},
\end{align*}
where we used $Mp=N$, $\gamma\le 1$, and \eqref{eq:V1_init_new}.
Finally, Jensen's inequality gives
\[
\frac{1}{T}\sum_{t=1}^T \sqrt{V_t}
\le
\sqrt{\frac{1}{T}\sum_{t=1}^T V_t}
\le
\frac{4\sigma}{\sqrt{N\gamma T}} + \frac{4\sigma\sqrt{\gamma}}{\sqrt{KM}},
\]
which proves \eqref{eq:avg_sqrtVt_new}.
\end{proof}

\subsection{A generic FL bound shared by both instantiations}
\label{subsec:fl_generic_bound}

Let
\[
\mathcal G_T:=\frac{1}{T}\sum_{t=1}^T \mathbb E\big[\|\nabla F(x^t)\|^{g,*}\big].
\]
Assume in addition that $1-\beta_1=\eta$.
Summing \eqref{eq:descent_core_new} over $t=1,\dots,T$, taking expectation, and telescoping yields
\[
\mathcal G_T
\le
\frac{\Delta_F}{T\eta K}
+\frac{L\bar\rho^2}{2}(K\eta)
+\frac{1}{T}\sum_{t=1}^T \frac{\mathbb E[E_t]}{\eta K},
\qquad
\Delta_F:=F(x^1)-F^\star.
\]
Using \eqref{eq:Et_split_new} together with
\eqref{eq:Et_a_over_etaK_unit}, \eqref{eq:Etb_over_etaK_new}, and \eqref{eq:Etc_total_over_etaK_new},
and then invoking Lemmas~\ref{lem:Bt_app_new}, \ref{lem:Rt_app_new}, and \ref{lem:Vt_app_new}, we obtain
\begin{align}
\mathcal G_T
&\le
\frac{\Delta_F}{T\eta K}
+\frac{5L\bar\rho^2}{2}(K\eta)
+\Big(2+\frac{1}{2K}\Big)\bar\rho^2\eta
+L^2\bar\rho^2\eta^3
+2\sigma^2\eta
+\Big(3-\frac1K\Big)L\bar\rho^2K\eta^2
+2\bar\rho\sigma\eta \notag\\
&\quad
+2\bar\rho\eta\left[
 \frac{8\sqrt{2}LK\bar\rho\eta}{p(1-\beta_2)}
 + \frac{4\sqrt{2}\sigma\sqrt{1-\beta_2}}{\sqrt{Kp}}
 + \frac{8\sigma}{p(1-\beta_2)T}
\right] \notag\\
&\quad
+2\bar\rho\left[\frac{8\sqrt{2}LK\bar\rho\eta}{p(1-\beta_2)}\right]
+2\bar\rho\left[
\frac{4\sigma}{\sqrt{N(1-\beta_2)T}}
+\frac{4\sigma\sqrt{1-\beta_2}}{\sqrt{KM}}
\right].
\label{eq:GT_master_fl_inst_new}
\end{align}
This generic bound will be specialized below to a well-tuned regime with fixed $\gamma=1-\beta_2\in(0,1]$ and to the problem-parameter-free regime.

\subsection{Proof of Theorem~\ref{thm:fl_welltuned}}
\label{app:fl_welltuned_proof}

\begin{proof}
Set $1-\beta_2=\gamma$ in \eqref{eq:GT_master_fl_inst_new}.
Since $\eta\le 1$, we have
\[
L^2\bar\rho^2\eta^3\le L^2\bar\rho^2\eta,
\qquad
\Big(3-\frac1K\Big)L\bar\rho^2K\eta^2\le 3L\bar\rho^2K\eta,
\qquad
\frac{16\sqrt2\,L\bar\rho^2K\eta^2}{p\gamma}\le \frac{16\sqrt2\,L\bar\rho^2K\eta}{p\gamma},
\]
and
\[
\frac{16\bar\rho\sigma\eta}{p\gamma T}\le \frac{16\bar\rho\sigma}{p\gamma T},
\qquad
2+\frac{1}{2K}\le \frac{5}{2}.
\]
Therefore,
\begin{align}
\mathcal G_T
&\le
\frac{\Delta_F}{T\eta K}
+A_{\mathrm{wt}}(\gamma)\,\eta
+\frac{16\bar\rho\sigma}{p\gamma T}
+\frac{8\bar\rho\sigma}{\sqrt{N\gamma T}}
+\frac{8\bar\rho\sigma\sqrt{\gamma}}{\sqrt{KM}}.
\label{eq:GT_welltuned_reduced}
\end{align}
Let
\[
a:=\frac{\Delta_F}{KT},
\qquad
b:=A_{\mathrm{wt}}(\gamma).
\]
By the choice $\eta=\min\{1,\sqrt{a/b}\}$, if $a\le b$, then
\[
\frac{a}{\eta}+b\eta=2\sqrt{ab}.
\]
If $a>b$, then $\eta=1$ and
\[
\frac{a}{\eta}+b\eta=a+b\le a+2\sqrt{ab}.
\]
Hence, in all cases,
\[
\frac{\Delta_F}{T\eta K}+A_{\mathrm{wt}}(\gamma)\eta
\le
\frac{\Delta_F}{KT}
+2\sqrt{\frac{A_{\mathrm{wt}}(\gamma)\Delta_F}{KT}}.
\]
Substituting this into \eqref{eq:GT_welltuned_reduced} proves \eqref{eq:rate_fl_welltuned}.
\end{proof}

\subsection{Proof of Theorem~\ref{thm:fl_final}}
\begin{proof}
Recall \eqref{eq:descent_core_new}:
\[
\eta K\|\nabla F(x^t)\|^{g,*}
\le
F(x^t)-F(x^{t+1})
+\frac{L}{2}(K\bar\rho\eta)^2
+E_t.
\]
Summing over $t=1,\dots,T$, taking expectation, and telescoping gives
\begin{align}
\mathcal G_T
:=\frac{1}{T}\sum_{t=1}^T \mathbb E\big[\|\nabla F(x^t)\|^{g,*}\big]
&\le
\frac{\Delta_F}{T\eta K}
+\frac{L}{2}\cdot\frac{(K\bar\rho\eta)^2}{\eta K}
+\frac{1}{T}\sum_{t=1}^T \frac{\mathbb E[E_t]}{\eta K} \notag\\
&=
\frac{\Delta_F}{T\eta K}
+\frac{L\bar\rho^2}{2}\,(K\eta)
+\frac{1}{T}\sum_{t=1}^T \frac{\mathbb E[E_t]}{\eta K},
\label{eq:GT_start_new}
\end{align}
where $\Delta_F:=F(x^1)-F^\star$.

\paragraph{Step 1: substitute the $E_t$ decomposition bound.}
Using \eqref{eq:Et_split_new} together with
\eqref{eq:Et_a_over_etaK_unit}, \eqref{eq:Etb_over_etaK_new}, and \eqref{eq:Etc_total_over_etaK_new}, we have for each $t$,
\begin{align}
\frac{\mathbb E[E_t]}{\eta K}
&\le
\Big(2+\frac{1}{2K}\Big)\bar\rho^2\eta
+(1-\beta_1)^2L^2\bar\rho^2\eta
+\frac{2(1-\beta_1)^2\sigma^2}{\eta} \notag\\
&\quad
+2\bar\rho(1-\beta_1)\sqrt{B_t}
+\Big(3-\frac1K\Big)\bar\rho^2L(1-\beta_1)(K\eta)
+2\bar\rho(1-\beta_1)\sigma \notag\\
&\quad
+2L\bar\rho^2(K\eta)
+2\bar\rho\sqrt{R_t}
+2\bar\rho\sqrt{V_t}.
\label{eq:Et_over_etaK_master_new}
\end{align}

\paragraph{Step 2: choose hyperparameters and simplify deterministic terms.}
Assume $T\ge K$ and $NK\le T$. Set
\begin{equation}
\eta := \frac{N^{1/4}}{K^{3/4}T^{3/4}},
\qquad
1-\beta_1 := \eta,
\qquad
1-\beta_2 := \sqrt{\frac{NK}{T+K}}.
\label{eq:param_choice_new}
\end{equation}
Note that $NK\le T$ implies
\begin{equation}
K\eta=\frac{(NK)^{1/4}}{T^{3/4}}\le \frac{1}{\sqrt T},
\qquad
\eta\le \frac{1}{\sqrt T}\le 1.
\label{eq:Keta_le_sqrtT}
\end{equation}
Plugging $1-\beta_1=\eta$ into \eqref{eq:Et_over_etaK_master_new} gives
\begin{align}
\frac{\mathbb E[E_t]}{\eta K}
&\le
\Big(2+\frac{1}{2K}\Big)\bar\rho^2\eta
+L^2\bar\rho^2\eta^3
+2\sigma^2\eta
+\Big(3-\frac1K\Big)L\bar\rho^2K\eta^2
+2\bar\rho\sigma\eta \notag\\
&\quad
+2\bar\rho\eta\sqrt{B_t}
+2L\bar\rho^2(K\eta)
+2\bar\rho\sqrt{R_t}
+2\bar\rho\sqrt{V_t}.
\label{eq:Et_over_etaK_simplified_newchoice}
\end{align}
Averaging over $t=1,\dots,T$ and substituting into \eqref{eq:GT_start_new} gives
\begin{align}
\mathcal G_T
&\le
\frac{\Delta_F}{T\eta K}
+\frac{L\bar\rho^2}{2}(K\eta)
+\Big(2+\frac{1}{2K}\Big)\bar\rho^2\eta
+L^2\bar\rho^2\eta^3
+2\sigma^2\eta
+\Big(3-\frac1K\Big)L\bar\rho^2K\eta^2
+2\bar\rho\sigma\eta \notag\\
&\quad
+2\bar\rho\eta\cdot \frac1T\sum_{t=1}^T\sqrt{B_t}
+2L\bar\rho^2(K\eta)
+2\bar\rho\cdot \frac1T\sum_{t=1}^T\sqrt{R_t}
+2\bar\rho\cdot \frac1T\sum_{t=1}^T\sqrt{V_t}.
\label{eq:GT_before_BRV_new}
\end{align}

\paragraph{Step 3: plug in Lemmas~\ref{lem:Bt_app_new}, \ref{lem:Rt_app_new}, and \ref{lem:Vt_app_new}.}
By Lemmas~\ref{lem:Bt_app_new}, \ref{lem:Rt_app_new}, and \ref{lem:Vt_app_new},
\begin{align}
\frac{1}{T}\sum_{t=1}^T \sqrt{B_t}
&\le
 \frac{8\sqrt{2}LK\bar\rho\eta}{p(1-\beta_2)}
 + \frac{4\sqrt{2}\sigma\sqrt{1-\beta_2}}{\sqrt{Kp}}
 + \frac{8\sigma}{p(1-\beta_2)T},
\label{eq:avg_sqrtBt_lemma_newchoice}\\
\frac{1}{T}\sum_{t=1}^T \sqrt{R_t}
&\le
\frac{8\sqrt{2}LK\bar\rho\eta}{p(1-\beta_2)},
\label{eq:avg_sqrtRt_lemma_newchoice}\\
\frac{1}{T}\sum_{t=1}^T \sqrt{V_t}
&\le
\frac{4\sigma}{\sqrt{N(1-\beta_2)T}}
+\frac{4\sigma\sqrt{1-\beta_2}}{\sqrt{KM}}.
\label{eq:avg_sqrtVt_lemma_newchoice}
\end{align}
Substituting these into \eqref{eq:GT_before_BRV_new} gives
\begin{align}
\mathcal G_T
&\le
\frac{\Delta_F}{T\eta K}
+\frac{5L\bar\rho^2}{2}(K\eta)
+\Big(2+\frac{1}{2K}\Big)\bar\rho^2\eta
+L^2\bar\rho^2\eta^3
+2\sigma^2\eta
+\Big(3-\frac1K\Big)L\bar\rho^2K\eta^2
+2\bar\rho\sigma\eta \notag\\
&\quad
+2\bar\rho\eta\left[
 \frac{8\sqrt{2}LK\bar\rho\eta}{p(1-\beta_2)}
 + \frac{4\sqrt{2}\sigma\sqrt{1-\beta_2}}{\sqrt{Kp}}
 + \frac{8\sigma}{p(1-\beta_2)T}
\right] \notag\\
&\quad
+2\bar\rho\left[\frac{8\sqrt{2}LK\bar\rho\eta}{p(1-\beta_2)}\right]
+2\bar\rho\left[
\frac{4\sigma}{\sqrt{N(1-\beta_2)T}}
+\frac{4\sigma\sqrt{1-\beta_2}}{\sqrt{KM}}
\right].
\label{eq:GT_expand_all_newchoice}
\end{align}

\paragraph{Step 4: evaluate the dominant terms under \eqref{eq:param_choice_new}.}
Using $T\ge K$ so that $T+K\le 2T$, we have
\[
\frac{1}{1-\beta_2}=\sqrt{\frac{T+K}{NK}}\le \sqrt{\frac{2T}{NK}},
\qquad
\sqrt{1-\beta_2}
=\left(\frac{NK}{T+K}\right)^{1/4}
\le \left(\frac{NK}{T}\right)^{1/4}.
\]
Also, by \eqref{eq:Keta_le_sqrtT},
\[
K\eta\le T^{-1/2}.
\]
For the drift contribution coming from $\sqrt{R_t}$,
\begin{align}
\frac{LK\eta}{p(1-\beta_2)}
&=
\frac{L}{p}\cdot \frac{(NK)^{1/4}}{T^{3/4}}\cdot \sqrt{\frac{T+K}{NK}}
\le
\frac{\sqrt2\,L}{p(NKT)^{1/4}}.
\label{eq:key_L_term_order_new}
\end{align}
For the averaged stochastic contribution coming from $\sqrt{V_t}$,
\begin{align}
\frac{\sigma}{\sqrt{N(1-\beta_2)T}}
&=
\sigma\sqrt{\frac{1}{NT}}\left(\frac{T+K}{NK}\right)^{1/4}
\le
\frac{2^{1/4}\sigma}{(NKT)^{1/4}},
\label{eq:key_sigma_term_order_1_new}\\
\frac{\sigma\sqrt{1-\beta_2}}{\sqrt{KM}}
&\le
\sigma\cdot \frac{(NK/T)^{1/4}}{\sqrt{KM}}
=
\frac{\sigma\sqrt p}{(NKT)^{1/4}}
\le
\frac{\sigma}{(NKT)^{1/4}}.
\label{eq:key_sigma_term_order_2_new}
\end{align}
Hence the $\sqrt{V_t}$ contribution produces a stochastic leading term of order $(NKT)^{-1/4}$ with constant independent of $M$.

\paragraph{Step 5: lower-order terms.}
The two dominant contributions are therefore
\[
\frac{\Delta_F}{T\eta K}=
\frac{\Delta_F}{(NKT)^{1/4}},
\qquad
2\bar\rho\cdot \frac{8\sqrt2 LK\bar\rho\eta}{p(1-\beta_2)}
\lesssim
\frac{L\bar\rho^2/p}{(NKT)^{1/4}},
\]
and, by \eqref{eq:key_sigma_term_order_1_new}--\eqref{eq:key_sigma_term_order_2_new},
\[
2\bar\rho\cdot \frac{1}{T}\sum_{t=1}^T\sqrt{V_t}
\lesssim
\frac{\sigma\bar\rho}{(NKT)^{1/4}}.
\]
Moreover,
\[
\frac{5L\bar\rho^2}{2}(K\eta)\le \frac{5L\bar\rho^2}{2\sqrt T},
\]
which yields the $T^{-1/2}$ term.
All remaining terms are lower-order in $(T,K)$ under the choice \eqref{eq:param_choice_new}.
Indeed,
\[
\eta = \frac{N^{1/4}}{K^{3/4}T^{3/4}},
\qquad
K\eta^2=\frac{N^{1/2}}{K^{1/2}T^{3/2}},
\qquad
\eta^3=\frac{N^{3/4}}{K^{9/4}T^{9/4}},
\]
In particular, the query-wise noise term is $2\bar\rho\sigma\eta$, which is a constant multiple of the first term of $R_{\mathrm{lo}}(T,K)$. It therefore enlarges only $C_0$, leaving $C_1$ and $C_2$ unchanged.
and the terms coming from $\eta\sum_t\sqrt{B_t}$ are also lower-order:
\[
\eta\cdot \frac{LK\eta}{p(1-\beta_2)}
\lesssim \frac{1}{KT},
\qquad
\eta\cdot \frac{\sigma\sqrt{1-\beta_2}}{\sqrt{Kp}}
\lesssim \eta,
\qquad
\eta\cdot \frac{\sigma}{p(1-\beta_2)T}
\lesssim \frac{1}{N^{1/4}K^{5/4}T^{5/4}},
\]
all of which are absorbed by $R_{\mathrm{lo}}(T,K)$ after enlarging the constant $C_0$.
Therefore, one may take
\begin{align}
C_1
&:=\frac{5L\bar\rho^2}{2},\\
C_2
&:=\Delta_F
+64\cdot \frac{L\bar\rho^2}{p}
+32\cdot \sigma\bar\rho,
\end{align}
and obtain
\begin{equation}
\mathcal G_T
\le
\frac{C_1}{\sqrt T}
+\frac{C_2}{(NKT)^{1/4}}
+C_0\,R_{\mathrm{lo}}(T,K),
\label{eq:GT_final_bound_newchoice}
\end{equation}
where the lower-order remainder can be taken explicitly as
\begin{align}
R_{\mathrm{lo}}(T,K)
:=
&\;
\frac{N^{1/4}}{K^{3/4}T^{3/4}}
+\frac{N^{1/2}}{K^{1/2}T^{3/2}}
+\frac{N^{3/4}}{K^{9/4}T^{9/4}}
+\frac{1}{N^{1/4}K^{5/4}T^{3/4}}
+\frac{1}{\sqrt{NKT}}.
\label{eq:Rlo_explicit_newchoice}
\end{align}
Here $C_0$ depends only on $(L,\bar\rho,\sigma,p)$ (and absolute numerical constants), and is independent of $(T,K)$.
This completes the proof.
\end{proof}

\subsection{Proof of Theorem~\ref{thm:fl_inexact_projection}}
\begin{proof}
We show that inexact projection only adds the perturbation term stated in Theorem~\ref{thm:fl_inexact_projection}; the complexity statement then follows by specializing this perturbation to the hyperparameter choices in Theorem~\ref{thm:fl_final} and the well-tuned guarantee in Theorem~\ref{thm:fl_welltuned}.

The proof is identical to the exact-projection proof except for the predict-and-correct estimate.
Because the inexact solver always returns feasible update variables, the domain bound \eqref{eq:D_bound_new} and the round-movement bound \eqref{eq:move_round_new} remain valid.
Consequently, the smoothness descent inequality \eqref{eq:descent_core_new}, the bounds for $E_t^{(b)}$ and $E_t^{(c)}$, and all control-variate recursions are unchanged. In particular, the corrected query-wise noise bound \eqref{eq:Etb_noise_norm_new} uses only feasibility and hence also applies to inexact projections; it invokes no independence between a projected update and its current noise.

The only affected term is $E_t^{(a)}$.
Let
\begin{align*}
\hat\varepsilon_i^{t,k}
&:=
\operatorname{res}_{\mathcal D}(\hat z_i^{t,k},\hat\Delta_i^{t,k-1}),
\qquad
\varepsilon_i^{t,k}
:=
\operatorname{res}_{\mathcal D}(z_i^{t,k},\Delta_i^{t,k-1}).
\end{align*}
Applying Lemma~\ref{lem:inexact_pc_from_single_call} with $u=u_t$ and averaging over active clients gives
\begin{align}
E_t^{(a)}
\le\;&
\frac{(\bar\rho\eta)^2}{2}
+
\frac{2}{N}\sum_{i\in S_t}\sum_{k=2}^K
\|g_i^{t,k}-g_i^{t,k-1}\|_2^2
+
2K(\bar\rho\eta)^2
\notag\\
&\quad+
\frac{5}{N}\sum_{i\in S_t}\sum_{k=1}^K
\left(\hat\varepsilon_i^{t,k}+\varepsilon_i^{t,k}\right).
\label{eq:inexact_Eta_bound_single_call}
\end{align}
Taking expectation and using Lemma~\ref{lem:gvar_app_new} as in the exact-projection proof yields the same collection of non-residual terms as in the proofs of Theorem~\ref{thm:fl_final} and the well-tuned guarantee in Theorem~\ref{thm:fl_welltuned}, up to universal numerical constants.
The extra contribution, after summing over $t=1,\ldots,T$ and normalizing by $T\eta K$, is
\begin{align}
\frac{5}{T\eta K}
\sum_{t=1}^T
\mathbb E\!\left[
\frac1N\sum_{i\in S_t}\sum_{k=1}^K
\left(\hat\varepsilon_i^{t,k}+\varepsilon_i^{t,k}\right)
\right]
=
\frac{5\overline{\mathfrak R}_T}{\eta K}.
\label{eq:inexact_residual_extra_term}
\end{align}
This proves that the exact-projection bounds acquire only the perturbation term $\mathcal O(\overline{\mathfrak R}_T/(\eta K))$, with all other constants enlarged only by universal factors.

It remains to check that this perturbation does not alter the main complexity exponents under the residual budget $\overline{\mathfrak R}_T=\mathcal O(T^{-1})$.
For the well-tuned regime in Theorem~\ref{thm:fl_welltuned}, the update-budget radius is
\[
\eta=\min\left\{1,\sqrt{\frac{\Delta_F}{A_{\rm wt}(\gamma)KT}}\right\}.
\]
If $\eta=1$, then \eqref{eq:inexact_residual_extra_term} is $\mathcal O((KT)^{-1})$ under $\overline{\mathfrak R}_T=\mathcal O(T^{-1})$.
Otherwise, the perturbation is $\mathcal O\!\left(\sqrt{A_{\rm wt}(\gamma)/(\Delta_F KT)}\right)$. For fixed positive $\Delta_F$, it has the same $(T,K)$ dependence as the term $\sqrt{A_{\rm wt}(\gamma)\Delta_F/(KT)}$ in the well-tuned bound; here $A_{\rm wt}(\gamma)$ retains its stated dependence on $K$.
Thus the well-tuned communication and sample-complexity exponents are unchanged.

For the problem-parameter-free regime of Theorem~\ref{thm:fl_final},
\[
\eta=\frac{N^{1/4}}{K^{3/4}T^{3/4}},
\qquad
T\eta K=(NKT)^{1/4}.
\]
Hence $\overline{\mathfrak R}_T=\mathcal O(T^{-1})$ implies
\[
\frac{\overline{\mathfrak R}_T}{\eta K}
=
\mathcal O\!\left(\frac{1}{T\eta K}\right)
=
\mathcal O\!\left((NKT)^{-1/4}\right),
\]
which has the same order as the second leading term in Theorem~\ref{thm:fl_final}.
Therefore the problem-parameter-free communication and sample-complexity exponents are unchanged as well.
\end{proof}

\section{Decentralized-learning analysis}
\label{app:dl-analysis}

\paragraph{Exact initialization and its cost.}
\label{app:dl-initialization}
All nodes start from the common model $x_{\mathrm{init}}$ and form $c_i^1=\nabla f_i(x_{\mathrm{init}},\xi_i^{-1})$ using one stochastic-gradient query each. To initialize $m_i^1=\bar c^1:=M^{-1}\sum_j c_j^1$ exactly, we use a one-time averaging primitive before the $T$ training rounds. On the static connected support graph of $W$, a reduce-and-broadcast on a precomputed breadth-first spanning tree of depth $h$ takes $2h$ communication rounds and transmits $2(M-1)$ parameter-sized vectors. This is separate from repeated $W$-mixing, which need not reach exact consensus in finitely many rounds. Under FTC, one $\tau$-round mixing block gives the exact average; the subsequent valid mixing sequence is indexed from $t=1$. Both procedures ensure $D_1=0$, and Assumption~\ref{assump:stoch} gives $B_1\le\sigma^2$ with one query per node. The overall counts are $T+2h$ or $T+\tau$ communication rounds and $KT+1$ gradient queries per node, hence $MKT+M$ in total. These additive initialization costs do not change the accuracy exponents for a fixed network.

\subsection{Formal DL guarantees}
\label{app:dl-formal-guarantees}

\begin{theorem}[NBCLT-DL under static gossip matrix]
\label{thm:dl_final_gap}
Suppose that Assumptions~\ref{assump:smooth}--\ref{assump:norm} hold, and that the mixing matrix $\mathbf W\in\mathbb{R}^{M\times M}$ is fixed, symmetric, and doubly stochastic.
Define $\mathbf J=\frac1M\mathbf 1\mathbf 1^\top$, and
$\lambda:=\|\mathbf W-\mathbf J\|_2\in[0,1)$.
Use the one-query consensual initialization and exact tree averaging described in Appendix~\ref{app:dl-initialization}. Run Algorithm~\ref{alg:NBCLT-DL} for $T$ training rounds with $K$ local updates per round, and choose
\begin{equation}
\eta=\frac{M^{1/4}}{K^{3/4}T^{3/4}},
\qquad
1-\beta_1=\eta,
\qquad
1-\beta_2=\sqrt{\frac{MK}{T+K}},
\label{eq:dl_param_choice_gap}
\end{equation}
Define $\bar x^t:=\frac1M\sum_{i=1}^M x_i^t$ and $\Delta_F:=F(\bar x^1)-F^\star$.
Then for all $T\ge MK$, we have
\begin{equation}
\frac{1}{T}\sum_{t=1}^T
\mathbb E\!\left[\|\nabla F(\bar x^t)\|^{g,*}\right]
\;\le\;
\frac{C_1}{\sqrt T}
+
\frac{C_2}{(MKT)^{1/4}}
+
C_0\,R_{\mathrm{lo}}(T,K),
\label{eq:dl_rate_final_gap}
\end{equation}
where
\begin{align*}
C_1
&=
L\bar\rho^2
\Bigg(
\frac{5}{2}
+\frac{2}{1-\lambda}
+\frac{32\sqrt3+4\sqrt{15}}{(1-\lambda)^2}
\Bigg)
+
\frac{2(\sqrt6+\sqrt3)\bar\rho\sigma\sqrt M}{1-\lambda},\\
C_2
&=
\Delta_F
+\Bigl(\frac{5}{2}+L^2\Bigr)\bar\rho^2
+2\sigma^2
+\frac{32\sqrt2\,L\bar\rho^2}{1-\lambda}
+2\sqrt2\,\bar\rho\sigma\sqrt M.
\end{align*}
The lower-order remainder terms are
\begin{equation}
R_{\mathrm{lo}}(T,K)
=
\frac{1}{K\sqrt{T}}
+\frac{1}{\sqrt{KT}}
+\frac{1}{K^{5/4}T^{3/4}}
+\frac{1}{KT}
+\frac{1}{T},
\end{equation}
and the constant $C_0$ may depend on
$(L,\sigma,\bar\rho,M,\lambda,\Delta_F)$
but is independent of $(T,K)$.
\end{theorem}

\begin{theorem}[Convergence of NBCLT-DL under FTC]
\label{thm:dl_final_FTC}
Suppose Assumptions~\ref{assump:smooth}--\ref{assump:norm} hold.
Let $\{W^t\}_{t\ge 1}$ be doubly stochastic and satisfy the finite-time consensus (FTC) property~\eqref{eq:ftc}
with horizon $\tau\ge 1$.
Assume consensual initialization $x_i^1=\bar x^1$ for all $i$, and initialize trackers by
$c_i^1=\nabla f_i(x_i^1,\xi_i^{-1})$ using one query per node, followed by one $\tau$-round FTC block to obtain $m_i^1=M^{-1}\sum_{j=1}^M c_j^1$ exactly, as in Appendix~\ref{app:dl-initialization}.
Run Algorithm~\ref{alg:NBCLT-DL} for $T$ rounds with $K$ local steps and choose
\begin{equation}
\eta=\frac{M^{1/4}}{K^{3/4}T^{3/4}},
\qquad
1-\beta_1=\eta,
\qquad
1-\beta_2=\sqrt{\frac{MK}{T+K}}.
\label{eq:dl_param_choice_linearbeta2}
\end{equation}
Let $\Delta_F:=F(\bar x^1)-F^\star$, and let
$B_1:=\frac1M\sum_{i=1}^M \mathbb E\|c_i^1-\nabla f_i(x_i^1)\|_2^2$.
Under the stated initialization, $B_1\le \sigma^2$.
For any \(T \ge MK\), we have
\begin{equation}
\frac1T\sum_{t=1}^T \mathbb E\|\nabla F(\bar x^t)\|^{g,*}
\;\le\;
\frac{C_1}{\sqrt{T}}
+\frac{C_2}{(MKT)^{1/4}}
+ C_0 \mathcal R_{\mathrm{lo}},
\label{eq:dl_final_bound_linearbeta2}
\end{equation}
where
\begin{align}
C_1
&=
L\bar\rho^2
\Bigl(
\frac{5}{2}
+2\tau
+(32\sqrt3+4\sqrt{15})\tau\sqrt{1+\tau^2}
\Bigr)
+2(\sqrt6+\sqrt3)\tau\,\bar\rho\,\sigma\sqrt M, \notag\\
C_2
&=
\Delta_F
+\Bigl(\frac{5}{2}+L^2\Bigr)\bar\rho^2
+2\sigma^2
+32\sqrt2\,L\bar\rho^2\sqrt{1+\tau^2}
+2\sqrt2\,\bar\rho\sigma\sqrt M.
\label{thm3_c1_c2}
\end{align}
The lower-order remainder terms are
\begin{align}
\mathcal R_{\mathrm{lo}}
&:=
\frac{M^{1/4}}{K^{3/4}T^{3/4}}
\;+\;
\frac{\sqrt M}{\sqrt K\,T^{3/2}}
\;+\;
\frac{\sqrt M}{KT}
\;+\;
\frac{1}{M^{1/4}K^{5/4}T^{3/4}}
\;+\;
\frac{1}{\sqrt{MKT}},
\label{eq:Rlo_linearbeta2}
\end{align}
and the constant $C_0$ is independent of $(T,K,M)$.
\end{theorem}

\subsection{Notation and preliminaries}
\label{app:dl-notation}

Let $F(x)=\frac1M\sum_{i=1}^M f_i(x)$.
Algorithm~\ref{alg:NBCLT-DL} maintains $(x_i^t,c_i^t,m_i^t)\in(\mathbb R^d)^3$ for each node $i\in[M]$.
Define the averages
\[
\bar x^t:=\frac1M\sum_{i=1}^M x_i^t,
\qquad
\bar c^t:=\frac1M\sum_{i=1}^M c_i^t,
\qquad
\bar m^t:=\frac1M\sum_{i=1}^M m_i^t.
\]

\paragraph{Local iterates.}
For convenience, set $x_i^0:=x_i^1$.
In round $t$, define $x_i^{t,0}:=x_i^t$ and for $k=1,\dots,K$,
\[
\Delta_i^{t,k-1}:=x_i^{t,k}-x_i^{t,k-1},
\qquad
\Delta_i^t:=\sum_{k=1}^K\Delta_i^{t,k-1}=x_i^{t,k}-x_i^{t,0}.
\]

\paragraph{Mixing and spectral gap.}
Assume $\mathbf W\in\mathbb R^{M\times M}$ is fixed, symmetric and doubly stochastic.
Let $\mathbf J:=\frac1M\mathbf 1\mathbf 1^\top$ and $\lambda:=\|\mathbf W-\mathbf J\|_2\in[0,1)$.
The spectral gap is $1-\lambda$.
Mixing is
\[
x_i^{t+1}=\sum_{j=1}^M w_{ij}x_j^{t,K},
\qquad
m_i^{t+1}=\sum_{j=1}^M w_{ij}\big(m_j^t+c_j^{t+1}-c_j^t\big).
\]
By double stochasticity, mixing preserves the average:
\begin{equation}
\bar x^{t+1}=\frac1M\sum_{i=1}^M x_i^{t+1}
=\frac1M\sum_{j=1}^M x_j^{t,K}
=\bar x^t+\frac1M\sum_{j=1}^M\Delta_j^t.
\label{eq:dl_avg_update}
\end{equation}
Similarly,
\begin{equation}
\bar m^{t+1}=\bar m^t+\bar c^{t+1}-\bar c^t.
\label{eq:dl_avg_tracker}
\end{equation}

\paragraph{Projection operator and domain bound.}
Let $\Proj(\cdot)$ be Euclidean projection onto $\mathcal D=\{\Delta:\|\Delta\|^g\le \eta\}$.
By Assumption~\ref{assump:norm}, for any $\Delta\in\mathcal D$,
\begin{equation}
\|\Delta\|_2\le \bar\rho\|\Delta\|^g\le \bar\rho\eta.
\label{eq:dl_D_bound}
\end{equation}
Hence $\|\Delta_i^{t,k-1}\|_2\le \bar\rho\eta$ and
\begin{equation}
\|\bar x^{t+1}-\bar x^t\|_2
\le \frac1M\sum_{i=1}^M\|\Delta_i^t\|_2
\le \frac1M\sum_{i=1}^M\sum_{k=1}^K\|\Delta_i^{t,k-1}\|_2
\le K\bar\rho\eta.
\label{eq:dl_move_round}
\end{equation}

\paragraph{Stochastic gradients.}
Assume conditional unbiasedness and bounded conditional variance as in Assumption~\ref{assump:stoch},
and conditional independence across gradient queries.

\paragraph{Error measures.}
We quantify the accuracy of the local control variates, the disagreement of the gradient trackers,
and the one-step model transition by
\begin{align}
B_t
&:=\frac1M\sum_{i=1}^M \mathbb E\big\|c_i^t-\nabla f_i(x_i^{t-1})\big\|_2^2,
\notag\\
D_t
&:=\frac1M\sum_{i=1}^M \mathbb E\big\|m_i^t-\bar m^t\big\|_2^2,
\notag\\
R_t
&:=\frac1M\sum_{i=1}^M \mathbb E\big\|x_i^t-x_i^{t-1}\big\|_2^2.
\label{eq:Bt_Dt_defs_dl}
\end{align}
where $\bar m^t:=\frac1M\sum_{i=1}^M m_i^t$. Note that, by the convention $x_i^0=x_i^1$,
this definition agrees with the usual initialization quantity $B_1$.

\subsection{Auxiliary lemmas}

\begin{lemma}[$L$-smoothness descent on $\bar x^t$]
\label{lem:dl_smooth}
If $F$ is $L$-smooth, then for all $t$,
\begin{align}
F(\bar x^{t+1})
&\le
F(\bar x^t)+\langle \nabla F(\bar x^t),\,\bar x^{t+1}-\bar x^t\rangle
+\frac{L}{2}\|\bar x^{t+1}-\bar x^t\|_2^2 \notag\\
&\le
F(\bar x^t)+\langle \nabla F(\bar x^t),\,\bar x^{t+1}-\bar x^t\rangle
+\frac{L}{2}(K\bar\rho\eta)^2.
\label{eq:dl_smooth}
\end{align}
\end{lemma}

\begin{lemma}[Support function of $\mathcal D$]
\label{lem:dl_support}
Let $\|\cdot\|^{g,*}$ be the dual norm and define
\[
u_t\in\arg\min_{u\in\mathcal D}\langle \nabla F(\bar x^t),u\rangle.
\]
Then
\begin{equation}
\langle \nabla F(\bar x^t),u_t\rangle=-\eta\|\nabla F(\bar x^t)\|^{g,*}.
\label{eq:dl_support}
\end{equation}
\end{lemma}

\begin{lemma}[Projection three-point inequality]
\label{lem:dl_proj3}
Let $\Proj(\cdot)$ be Euclidean projection onto a nonempty closed convex set $\mathcal D$.
Let $y=\Proj(x-g)$. Then for any $u\in\mathcal D$,
\[
\langle g, y-u\rangle
\le
\frac12\Big(\|u-x\|_2^2-\|u-y\|_2^2-\|y-x\|_2^2\Big).
\]
\end{lemma}

\begin{lemma}[One-gradient predict-and-correct inequality]
\label{lem:dl_pc}
Fix $(t,i)$ and set $\hat\Delta_i^{t,-1}=0$.
For $k=1,\dots,K$, Algorithm~\ref{alg:NBCLT-DL} generates
\[
\hat\Delta_i^{t,k-1}=\Proj(\hat\Delta_i^{t,k-2}-g_i^{t,k}),
\qquad
\Delta_i^{t,k-1}=\Proj(\hat\Delta_i^{t,k-1}-g_i^{t,k}).
\]
Then for any $u\in\mathcal D$,
\begin{equation}
\sum_{k=1}^K \langle g_i^{t,k}, \Delta_i^{t,k-1}-u\rangle
\le
\frac{\|u\|_2^2}{2}
+\frac12\sum_{k=2}^K\|g_i^{t,k}-g_i^{t,k-1}\|_2^2
+2K(\bar\rho\eta)^2.
\label{eq:dl_pc}
\end{equation}
\end{lemma}

\begin{lemma}[Bounding gradient variation]
\label{lem:dl_gvar}
In Algorithm~\ref{alg:NBCLT-DL}, for any $(t,i,k\ge2)$,
\begin{equation}
\mathbb E\|g_i^{t,k}-g_i^{t,k-1}\|_2^2
\le
2(1-\beta_1)^2L^2\,\mathbb E\|x_i^{t,k-1}-x_i^{t,k-2}\|_2^2
+4(1-\beta_1)^2\sigma^2.
\label{eq:dl_gvar}
\end{equation}
\end{lemma}

\paragraph{Proofs omitted.}
The proofs of Lemmas~\ref{lem:dl_proj3}, \ref{lem:dl_pc}, and \ref{lem:dl_gvar} are identical to their federated-learning counterparts
(Lemmas~\ref{lem:proj_app_new}, \ref{lem:pc_app_new}, and \ref{lem:gvar_app_new}, respectively), and are therefore omitted.

\begin{lemma}[Static consensus bound for $C_t$]\label{lem:dl_Ct_simple}
Let
\[
C_t:=\frac1M\sum_{i=1}^M \mathbb E\|x_i^t-\bar x^t\|_2^2,
\qquad
\bar x^t:=\frac1M\sum_{i=1}^M x_i^t.
\]
Assume the mixing matrix $\mathbf W$ is fixed, symmetric, doubly stochastic, and satisfies
$\lambda:=\|\mathbf W-\mathbf J\|_2<1$ where $\mathbf J:=\frac1M\mathbf 1\mathbf 1^\top$.
Assume also that each round satisfies the per-node movement bound
\[
\|x_i^{t,k}-x_i^t\|_2 \le K\bar\rho\eta,\qquad \forall i\in[M].
\]
Then for all $t\ge 1$,
\[
\sqrt{C_{t+1}}
\le
\lambda \sqrt{C_t}+\lambda K\bar\rho\eta,
\]
and consequently,
\[
\sqrt{C_t}
\le
\lambda^{t-1}\sqrt{C_1}+\frac{\lambda}{1-\lambda}K\bar\rho\eta
\le
\sqrt{C_1}+\frac{K\bar\rho\eta}{1-\lambda}.
\]
In particular, if the initialization is consensual (e.g., $x_i^1=x^1$ for all $i$ so that $C_1=0$), then
\[
\sqrt{C_t}\le \frac{\lambda}{1-\lambda}K\bar\rho\eta \le \frac{K\bar\rho\eta}{1-\lambda}.
\]
\end{lemma}

\begin{proof}
Stack the node variables row-wise into the matrix
$\mathbf X^t\in\mathbb R^{M\times d}$ whose $i$-th row is $(x_i^t)^\top$.
Define the averaging projector $\mathbf J:=\frac1M\mathbf 1\mathbf 1^\top$
and the centered (disagreement) component
\[
\tilde{\mathbf X}^t := (\mathbf I-\mathbf J)\mathbf X^t.
\]
Then
\[
C_t=\frac1M\,\mathbb E\|\tilde{\mathbf X}^t\|_F^2.
\]

\paragraph{Step 1: relate $\tilde{\mathbf X}^{t+1}$ to $\tilde{\mathbf X}^{t,K}$.}
In the DL algorithm, after $K$ local updates, mixing gives
\[
\mathbf X^{t+1}=\mathbf W\,\mathbf X^{t,K}.
\]
Using double stochasticity, we have $\mathbf J\mathbf W=\mathbf W\mathbf J=\mathbf J$,
hence $(\mathbf I-\mathbf J)\mathbf W=\mathbf W-\mathbf J$. Therefore,
\[
\tilde{\mathbf X}^{t+1}
=(\mathbf I-\mathbf J)\mathbf X^{t+1}
=(\mathbf I-\mathbf J)\mathbf W\mathbf X^{t,K}
=(\mathbf W-\mathbf J)\mathbf X^{t,K}
=(\mathbf W-\mathbf J)\tilde{\mathbf X}^{t,K},
\]
where the last equality uses $(\mathbf W-\mathbf J)\mathbf J=\mathbf 0$.

\paragraph{Step 2: contract under mixing.}
By definition of $\lambda=\|\mathbf W-\mathbf J\|_2$ and the fact that
$\|\mathbf A\mathbf Z\|_F\le \|\mathbf A\|_2\|\mathbf Z\|_F$, we obtain
\[
\|\tilde{\mathbf X}^{t+1}\|_F
=
\|(\mathbf W-\mathbf J)\tilde{\mathbf X}^{t,K}\|_F
\le
\|\mathbf W-\mathbf J\|_2\cdot \|\tilde{\mathbf X}^{t,K}\|_F
\le
\lambda \|\tilde{\mathbf X}^{t,K}\|_F.
\]
Taking expectation and using Jensen's inequality,
\[
\sqrt{\mathbb E\|\tilde{\mathbf X}^{t+1}\|_F^2}
\le
\lambda \sqrt{\mathbb E\|\tilde{\mathbf X}^{t,K}\|_F^2}.
\]
Dividing by $\sqrt M$ yields
\begin{equation}
\sqrt{C_{t+1}}
\le
\lambda \sqrt{\frac1M\mathbb E\|\tilde{\mathbf X}^{t,K}\|_F^2}.
\label{eq:C_rec_step1}
\end{equation}

\paragraph{Step 3: bound the injected disagreement from local steps.}
Write the local-update increment matrix
$\mathbf \Delta^t:=\mathbf X^{t,K}-\mathbf X^t$ (row $i$ equals $(x_i^{t,k}-x_i^t)^\top$).
Then
\[
\tilde{\mathbf X}^{t,K}
=
(\mathbf I-\mathbf J)\mathbf X^{t,K}
=
(\mathbf I-\mathbf J)(\mathbf X^t+\mathbf \Delta^t)
=
\tilde{\mathbf X}^t+(\mathbf I-\mathbf J)\mathbf \Delta^t.
\]
Hence by triangle inequality and $\|\mathbf I-\mathbf J\|_2\le 1$,
\[
\|\tilde{\mathbf X}^{t,K}\|_F
\le
\|\tilde{\mathbf X}^t\|_F+\|(\mathbf I-\mathbf J)\mathbf \Delta^t\|_F
\le
\|\tilde{\mathbf X}^t\|_F+\|\mathbf \Delta^t\|_F.
\]
Moreover, using the assumed per-node movement bound,
\[
\|\mathbf \Delta^t\|_F^2
=
\sum_{i=1}^M\|x_i^{t,k}-x_i^t\|_2^2
\le
\sum_{i=1}^M (K\bar\rho\eta)^2
=
M(K\bar\rho\eta)^2,
\]
so $\|\mathbf \Delta^t\|_F\le \sqrt M\,K\bar\rho\eta$.
Taking expectation and again using Jensen gives
\[
\sqrt{\mathbb E\|\tilde{\mathbf X}^{t,K}\|_F^2}
\le
\sqrt{\mathbb E\|\tilde{\mathbf X}^{t}\|_F^2}
+\sqrt{\mathbb E\|\mathbf \Delta^t\|_F^2}
\le
\sqrt{\mathbb E\|\tilde{\mathbf X}^{t}\|_F^2}
+\sqrt M\,K\bar\rho\eta.
\]
Divide by $\sqrt M$ to obtain
\begin{equation}
\sqrt{\frac1M\mathbb E\|\tilde{\mathbf X}^{t,K}\|_F^2}
\le
\sqrt{C_t}+K\bar\rho\eta.
\label{eq:C_rec_step2}
\end{equation}

\paragraph{Step 4: combine and unroll.}
Plugging \eqref{eq:C_rec_step2} into \eqref{eq:C_rec_step1} yields
\[
\sqrt{C_{t+1}}
\le
\lambda\big(\sqrt{C_t}+K\bar\rho\eta\big)
=
\lambda\sqrt{C_t}+\lambda K\bar\rho\eta.
\]
Unrolling this recursion gives
\[
\sqrt{C_t}
\le
\lambda^{t-1}\sqrt{C_1}
+\lambda K\bar\rho\eta\sum_{s=0}^{t-2}\lambda^s
=
\lambda^{t-1}\sqrt{C_1}+\frac{\lambda(1-\lambda^{t-1})}{1-\lambda}K\bar\rho\eta
\le
\lambda^{t-1}\sqrt{C_1}+\frac{\lambda}{1-\lambda}K\bar\rho\eta.
\]
Finally, since $\lambda^{t-1}\le 1$ and $\lambda\le 1$, we also have
\[
\sqrt{C_t}\le \sqrt{C_1}+\frac{K\bar\rho\eta}{1-\lambda}.
\]
This completes the proof.
\end{proof}

\begin{lemma}[One-step model transition under static gossip]
\label{lem:dl_Rt_static}
Under the assumptions of Lemma~\ref{lem:dl_Ct_simple}, for all $t\ge 2$,
\begin{equation}
R_t
=\frac1M\sum_{i=1}^M \mathbb E\|x_i^t-x_i^{t-1}\|_2^2
\le
\frac{9(K\bar\rho\eta)^2}{(1-\lambda)^2}.
\label{eq:dl_Rt_static}
\end{equation}
\end{lemma}

\begin{proof}
For each $i$, write
\[
x_i^t-x_i^{t-1}
=
(x_i^t-\bar x^t)+(\bar x^t-\bar x^{t-1})+(\bar x^{t-1}-x_i^{t-1}).
\]
Using $\|a+b+c\|_2^2\le 3(\|a\|_2^2+\|b\|_2^2+\|c\|_2^2)$, averaging over $i$, and taking expectation yield
\[
R_t
\le
3C_t+3C_{t-1}+3\,\mathbb E\|\bar x^t-\bar x^{t-1}\|_2^2.
\]
By \eqref{eq:dl_move_round},
\[
\mathbb E\|\bar x^t-\bar x^{t-1}\|_2^2\le (K\bar\rho\eta)^2.
\]
Moreover, Lemma~\ref{lem:dl_Ct_simple} with consensual initialization gives
\[
C_t\le \frac{(K\bar\rho\eta)^2}{(1-\lambda)^2},
\qquad
C_{t-1}\le \frac{(K\bar\rho\eta)^2}{(1-\lambda)^2}.
\]
Since $(1-\lambda)^2\le 1$, we conclude
\[
R_t
\le
6\frac{(K\bar\rho\eta)^2}{(1-\lambda)^2}+3(K\bar\rho\eta)^2
\le
\frac{9(K\bar\rho\eta)^2}{(1-\lambda)^2}.
\]
This proves \eqref{eq:dl_Rt_static}.
\end{proof}

\subsection{Descent decomposition}

Define $u_t\in\arg\min_{u\in\mathcal D}\langle \nabla F(\bar x^t),u\rangle$.
Using Lemma~\ref{lem:dl_smooth} and \eqref{eq:dl_avg_update},
\begin{align}
F(\bar x^{t+1})
&\le
F(\bar x^t)
+\Big\langle \nabla F(\bar x^t),\,\frac1M\sum_{i=1}^M\sum_{k=1}^K \Delta_i^{t,k-1}\Big\rangle
+\frac{L}{2}(K\bar\rho\eta)^2\notag\\
&=
F(\bar x^t)
+K\langle \nabla F(\bar x^t),u_t\rangle
+\underbrace{\frac1M\sum_{i=1}^M\sum_{k=1}^K\langle \nabla F(\bar x^t),\Delta_i^{t,k-1}-u_t\rangle}_{=:E_t}
+\frac{L}{2}(K\bar\rho\eta)^2.
\label{eq:dl_descent_start}
\end{align}
By Lemma~\ref{lem:dl_support}, $K\langle \nabla F(\bar x^t),u_t\rangle=-\eta K\|\nabla F(\bar x^t)\|^{g,*}$.
Rearranging gives
\begin{equation}
\eta K\|\nabla F(\bar x^t)\|^{g,*}
\le
F(\bar x^t)-F(\bar x^{t+1})
+\frac{L}{2}(K\bar\rho\eta)^2
+E_t.
\label{eq:dl_descent_core}
\end{equation}

\subsection{Bounding \texorpdfstring{$E_t$}{E\_t} via predict-and-correct}

Recall the local signal
\[
g_i^{t,k}=(1-\beta_1)\big(\nabla f_i(x_i^{t,k-1},\xi_i^{t,k-1})-c_i^t\big)+m_i^t.
\]
For each $(i,k)$,
\begin{align*}
\langle \nabla F(\bar x^t),\Delta_i^{t,k-1}-u_t\rangle
&=
\langle m_i^t,\Delta_i^{t,k-1}-u_t\rangle
+\langle \nabla F(\bar x^t)-m_i^t,\Delta_i^{t,k-1}-u_t\rangle\\
&=
\langle g_i^{t,k},\Delta_i^{t,k-1}-u_t\rangle
-(1-\beta_1)\langle \nabla f_i(x_i^{t,k-1},\xi_i^{t,k-1})-c_i^t,\Delta_i^{t,k-1}-u_t\rangle\\
&\quad
+\langle \nabla F(\bar x^t)-m_i^t,\Delta_i^{t,k-1}-u_t\rangle.
\end{align*}
Summing over $i\in[M],k\in[K]$ and dividing by $M$ yields
\[
E_t=E_t^{(a)}+E_t^{(b)}+E_t^{(c)},
\]
where
\begin{align}
E_t^{(a)}
&:=\frac{1}{M}\sum_{i=1}^M\sum_{k=1}^K\langle g_i^{t,k},\Delta_i^{t,k-1}-u_t\rangle,\label{eq:dl_Ea_def}\\
E_t^{(b)}
&:=-(1-\beta_1)\cdot\frac{1}{M}\sum_{i=1}^M\sum_{k=1}^K
\langle \nabla f_i(x_i^{t,k-1},\xi_i^{t,k-1})-c_i^t,\Delta_i^{t,k-1}-u_t\rangle,\label{eq:dl_Eb_def}\\
E_t^{(c)}
&:=\frac{1}{M}\sum_{i=1}^M\sum_{k=1}^K
\langle \nabla F(\bar x^t)-m_i^t,\Delta_i^{t,k-1}-u_t\rangle.\label{eq:dl_Ec_def}
\end{align}

\paragraph{Step 1: bound $E_t^{(a)}$ (predict-and-correct).}
We apply the one-gradient predict-and-correct inequality in Lemma~\ref{lem:dl_pc} to each local sequence
$\{\Delta_i^{t,k-1}\}_{k=1}^K$ with the reference point chosen as $u=u_t$.
This yields, for each node $i$,
an upper bound on the cumulative inner products
$\sum_{k=1}^K \langle g_i^{t,k},\Delta_i^{t,k-1}-u_t\rangle$
in terms of the radius of the domain $\mathcal D$ and the temporal variation of the local gradient signals
$\{g_i^{t,k}\}_{k=1}^K$.

Averaging over all nodes and using $\|u_t\|_2\le \bar\rho\eta$ together with the diameter bound
$\|\Delta_i^{t,k-1}\|_2\le \bar\rho\eta$, we obtain the deterministic radius contribution of order
$\bar\rho^2\eta$.
Moreover, the gradient-variation term
$\sum_{k=2}^K\|g_i^{t,k}-g_i^{t,k-1}\|_2^2$
is controlled by Lemma~\ref{lem:dl_gvar}, and the local iterate difference
$\|x_i^{t,k-1}-x_i^{t,k-2}\|_2$ appearing there is bounded by
$\|\Delta_i^{t,k-2}\|_2\le \bar\rho\eta$ due to the projection constraint.

Collecting these bounds and summing over $i=1,\ldots,M$, we obtain
\begin{equation}
\frac{\mathbb E[E_t^{(a)}]}{\eta K}
\le
\Big(2+\frac{1}{2K}\Big)\bar\rho^2\eta
+(1-\beta_1)^2L^2\bar\rho^2\eta
+\frac{2(1-\beta_1)^2\sigma^2}{\eta}.
\label{eq:dl_Ea_over}
\end{equation}

\paragraph{Step 2: bound $E_t^{(b)}$ (tracking + noise).}
Define the gradient noise $\zeta_i^{t,k-1}:=\nabla f_i(x_i^{t,k-1},\xi_i^{t,k-1})-\nabla f_i(x_i^{t,k-1})$.
Using $\|\Delta_i^{t,k-1}-u_t\|_2\le 2\bar\rho\eta$ and $\mathbb E\|\zeta_i^{t,k-1}\|_2\le\sigma$, apply Cauchy--Schwarz to each query exactly as in \eqref{eq:Etb_noise_norm_new}, now with full participation. No martingale cancellation of the noise-dependent projected multipliers is used. We obtain
\begin{align}
\frac{\mathbb E[E_t^{(b)}]}{\eta K}
&\le
\frac{2\bar\rho(1-\beta_1)}{MK}
\sum_{i=1}^M\sum_{k=1}^K
\mathbb E\|\nabla f_i(x_i^{t,k-1})-c_i^t\|_2
+2\bar\rho(1-\beta_1)\sigma.
\label{eq:dl_Eb_start_new}
\end{align}
Insert the stale target $\nabla f_i(x_i^{t-1})$:
\[
\|\nabla f_i(x_i^{t,k-1})-c_i^t\|_2
\le
\|\nabla f_i(x_i^{t-1})-c_i^t\|_2
+
L\|x_i^{t,k-1}-x_i^{t-1}\|_2.
\]
Moreover,
\[
\|x_i^{t,k-1}-x_i^{t-1}\|_2
\le
\|x_i^{t}-x_i^{t-1}\|_2+(k-1)\bar\rho\eta.
\]
Therefore,
\begin{align*}
\frac{1}{MK}
\sum_{i=1}^M\sum_{k=1}^K
\mathbb E\|\nabla f_i(x_i^{t,k-1})-c_i^t\|_2
&\le
\frac1M\sum_{i=1}^M \mathbb E\|\nabla f_i(x_i^{t-1})-c_i^t\|_2
+L\cdot \frac1M\sum_{i=1}^M \mathbb E\|x_i^{t}-x_i^{t-1}\|_2 \\
&\qquad
+L\bar\rho\eta\cdot \frac{K-1}{2}.
\end{align*}
Using Jensen and the definitions of $B_t$ and $R_t$,
\[
\frac1M\sum_{i=1}^M \mathbb E\|\nabla f_i(x_i^{t-1})-c_i^t\|_2
\le \sqrt{B_t},
\qquad
\frac1M\sum_{i=1}^M \mathbb E\|x_i^{t}-x_i^{t-1}\|_2
\le \sqrt{R_t}.
\]
Plugging these estimates into \eqref{eq:dl_Eb_start_new} yields
\begin{equation}
\frac{\mathbb E[E_t^{(b)}]}{\eta K}
\le
2\bar\rho(1-\beta_1)\sqrt{B_t}
+2L\bar\rho(1-\beta_1)\sqrt{R_t}
+\bar\rho^2L(1-\beta_1)\eta(K-1)
+2\bar\rho(1-\beta_1)\sigma.
\label{eq:dl_Eb_over}
\end{equation}

\paragraph{Step 3: bound $E_t^{(c)}$ (tracker mismatch).}
By \eqref{eq:dl_D_bound}, $\|\Delta_i^{t,k-1}-u_t\|_2\le 2\bar\rho\eta$. Hence
\[
|E_t^{(c)}|
\le
\frac{1}{M}\sum_{i=1}^M\sum_{k=1}^K
\|\nabla F(\bar x^t)-m_i^t\|_2\cdot 2\bar\rho\eta
=
2\bar\rho\eta K \cdot \frac1M\sum_{i=1}^M \|\nabla F(\bar x^t)-m_i^t\|_2.
\]
Dividing by $\eta K$ and taking expectation gives
\begin{equation}
\frac{\mathbb E[E_t^{(c)}]}{\eta K}
\le
2\bar\rho\,
\mathbb E\Big[\frac1M\sum_{i=1}^M\|\nabla F(\bar x^t)-m_i^t\|_2\Big].
\label{eq:dl_Ec_over_start}
\end{equation}
Insert and subtract $\bar m^t$:
\[
\|\nabla F(\bar x^t)-m_i^t\|_2
\le
\|\nabla F(\bar x^t)-\bar m^t\|_2 + \|\bar m^t-m_i^t\|_2.
\]
For the disagreement term, Jensen yields
\[
\mathbb E\Big[\frac1M\sum_{i=1}^M\|\bar m^t-m_i^t\|_2\Big]\le \sqrt{D_t}.
\]
For the bias term, using
$\nabla F(\bar x^t)=\frac1M\sum_{j=1}^M\nabla f_j(\bar x^t)$ and Jensen,
\[
\|\nabla F(\bar x^t)-\bar m^t\|_2
=
\Big\|\frac1M\sum_{j=1}^M\big(\nabla f_j(\bar x^t)-m_j^t\big)\Big\|_2
\le
\frac1M\sum_{j=1}^M\|\nabla f_j(\bar x^t)-m_j^t\|_2.
\]
Now insert and subtract $c_j^t$ and then the stale target $\nabla f_j(x_j^{t-1})$:
\begin{align*}
\frac1M\sum_{j=1}^M\|\nabla f_j(\bar x^t)-m_j^t\|_2
&\le
\frac1M\sum_{j=1}^M\|\nabla f_j(\bar x^t)-c_j^t\|_2
+\frac1M\sum_{j=1}^M\|c_j^t-m_j^t\|_2 \\
&\le
\frac1M\sum_{j=1}^M\|\nabla f_j(x_j^{t-1})-c_j^t\|_2
+\frac{L}{M}\sum_{j=1}^M\|\bar x^t-x_j^{t-1}\|_2
+\sqrt{D_t}.
\end{align*}
Using Jensen, the definition of $B_t$, and
\[
\|\bar x^t-x_j^{t-1}\|_2
\le
\|\bar x^t-\bar x^{t-1}\|_2+\|x_j^{t-1}-\bar x^{t-1}\|_2,
\]
we obtain
\[
\|\nabla F(\bar x^t)-\bar m^t\|_2
\le
\sqrt{B_t}+L\|\bar x^t-\bar x^{t-1}\|_2+L\sqrt{C_{t-1}}+\sqrt{D_t}.
\]
Finally, by \eqref{eq:dl_move_round}, $\|\bar x^t-\bar x^{t-1}\|_2\le K\bar\rho\eta$. Therefore,
\begin{equation}
\frac{\mathbb E[E_t^{(c)}]}{\eta K}
\le
2\bar\rho\big(\sqrt{D_t}+\sqrt{B_t}\big)
+2L\bar\rho^2K\eta
+2L\bar\rho\sqrt{C_{t-1}}.
\label{eq:dl_Ec_over_final_simpler}
\end{equation}

\subsection{Control-variates recursion}

\begin{lemma}[Control recursion in the decentralized setting]
\label{lem:dl_Bt_explicit}
Assume each $f_i$ is $L$-smooth and Assumption~\ref{assump:stoch} holds.
Define
\[
B_t:=\frac1M\sum_{i=1}^M \mathbb E\big\|c_i^t-\nabla f_i(x_i^{t-1})\big\|_2^2 .
\]
Assume the network mixing matrix $\mathbf W$ is fixed, symmetric and doubly stochastic,
and $\|\mathbf W-\mathbf J\|_2=\lambda<1$.
Then there exist absolute constants $c_B,c_B'>0$ such that for all $t\ge1$,
\begin{equation}
B_{t+1}
\le
\Big(1-\frac{1-\beta_2}{8}\Big) B_t
\;+
\;
c_B\frac{L^2(K\bar\rho\eta)^2}{(1-\beta_2)(1-\lambda)^2}
\;+
\;
\frac{8(1-\beta_2)^2\sigma^2}{K}.
\label{eq:dl_B_rec_explicit}
\end{equation}
Consequently, for any $T\ge1$,
\begin{align}
\frac{1}{T}\sum_{t=1}^T \sqrt{B_t}
\le\;
&c_B'\frac{LK\bar\rho\eta}{(1-\beta_2)(1-\lambda)}
+\frac{8\,\sigma\sqrt{1-\beta_2}}{\sqrt K}
+\frac{16\,\sqrt{B_1}}{(1-\beta_2)T}.
\label{eq:dl_avg_sqrtB_explicit}
\end{align}
If moreover the one-batch initialization is used (so $B_1\le\sigma^2$), then
\begin{equation}
\frac{1}{T}\sum_{t=1}^T \sqrt{B_t}
\le
c_B'\frac{LK\bar\rho\eta}{(1-\beta_2)(1-\lambda)}
+\frac{8\,\sigma\sqrt{1-\beta_2}}{\sqrt K}
+\frac{16\,\sigma}{(1-\beta_2)T}.
\label{eq:dl_avg_sqrtB_explicit_B1}
\end{equation}
\end{lemma}

\begin{proof}
Let $h_i^t:=\nabla f_i(x_i^{t-1})$ and $e_i^t:=c_i^t-h_i^t$. Then
$B_t=\frac1M\sum_{i=1}^M\mathbb E\|e_i^t\|_2^2$.
Recall the control update
\[
c_i^{t+1}=c_i^t+(1-\beta_2)(\hat c_i^t-c_i^t),
\qquad
\hat c_i^t=\frac1K\sum_{k=1}^K \nabla f_i(x_i^{t,k-1},\xi_i^{t,k-1}).
\]
Subtract $h_i^{t+1}=\nabla f_i(x_i^{t})$ to obtain
\begin{equation}
e_i^{t+1}
=
\beta_2 e_i^t+(1-\beta_2)(\hat c_i^t-h_i^t)-d_i^t,
\qquad
d_i^t:=h_i^{t+1}-h_i^t=\nabla f_i(x_i^t)-\nabla f_i(x_i^{t-1}).
\label{eq:dl_e_dyn_explicit}
\end{equation}

\paragraph{Step 1: decompose $\hat c_i^t-h_i^t$.}
Write
\[
\hat c_i^t-h_i^t=b_i^t+n_i^t,
\]
where
\[
b_i^t:=\frac1K\sum_{k=1}^K\big(\nabla f_i(x_i^{t,k-1})-\nabla f_i(x_i^{t-1})\big),
\qquad
n_i^t:=\frac1K\sum_{k=1}^K\zeta_i^{t,k-1}.
\]
By Assumption~\ref{assump:stoch} and independence across queries,
\begin{equation}
\mathbb E[n_i^t\mid \mathcal F_t]=0,
\qquad
\mathbb E\|n_i^t\|_2^2\le \frac{\sigma^2}{K}.
\label{eq:dl_n_bound_explicit}
\end{equation}
Moreover, by $L$-smoothness,
\[
\|b_i^t\|_2^2
\le
\frac1K\sum_{k=1}^K L^2\|x_i^{t,k-1}-x_i^{t-1}\|_2^2.
\]
Using
\[
\|x_i^{t,k-1}-x_i^{t-1}\|_2
\le
\|x_i^{t}-x_i^{t-1}\|_2+(k-1)\bar\rho\eta,
\]
and $\|u+v\|_2^2\le 2\|u\|_2^2+2\|v\|_2^2$, we obtain
\[
\|b_i^t\|_2^2
\le
2L^2\|x_i^{t}-x_i^{t-1}\|_2^2+2L^2(K\bar\rho\eta)^2.
\]
Averaging over $i$ and using Lemma~\ref{lem:dl_Rt_static} gives
\begin{equation}
\frac1M\sum_{i=1}^M \mathbb E\|b_i^t\|_2^2
\le
2L^2R_t+2L^2(K\bar\rho\eta)^2
\le
\frac{20L^2(K\bar\rho\eta)^2}{(1-\lambda)^2}.
\label{eq:dl_b_bound_explicit}
\end{equation}
Likewise, by $L$-smoothness and Lemma~\ref{lem:dl_Rt_static},
\begin{equation}
\frac1M\sum_{i=1}^M \mathbb E\|d_i^t\|_2^2
\le
L^2R_t
\le
\frac{9L^2(K\bar\rho\eta)^2}{(1-\lambda)^2}.
\label{eq:dl_d_bound_explicit}
\end{equation}

\paragraph{Step 2: one-step recursion for $\mathbb E\|e_i^{t+1}\|^2$.}
Condition on the history up to the beginning of round $t$. Since $d_i^t$ is measurable with respect to this history and $\mathbb E[n_i^t\mid\mathcal F_t]=0$, the cross term involving $n_i^t$ vanishes. Applying Young's inequality twice to \eqref{eq:dl_e_dyn_explicit}, and using $\beta_2^2\le 1$, yields the safe bound
\begin{equation}
\mathbb E\|e_i^{t+1}\|_2^2
\le
\Big(1-\frac{1-\beta_2}{8}\Big)\mathbb E\|e_i^t\|_2^2
+c_1(1-\beta_2)\,\mathbb E\|b_i^t\|_2^2
+\frac{c_2}{1-\beta_2}\,\mathbb E\|d_i^t\|_2^2
+\frac{8(1-\beta_2)^2\sigma^2}{K},
\label{eq:dl_pre_rec_explicit}
\end{equation}
for some absolute constants $c_1,c_2>0$.
Averaging \eqref{eq:dl_pre_rec_explicit} over $i=1,\dots,M$ and using \eqref{eq:dl_b_bound_explicit}--\eqref{eq:dl_d_bound_explicit} gives \eqref{eq:dl_B_rec_explicit} after enlarging the numerical constant to $c_B$.

\paragraph{Step 3: average $\sqrt{B_t}$.}
Let $\alpha:=\frac{1-\beta_2}{8}$ and
\[
C:=c_B\frac{L^2(K\bar\rho\eta)^2}{(1-\beta_2)(1-\lambda)^2}+\frac{8(1-\beta_2)^2\sigma^2}{K}.
\]
Then \eqref{eq:dl_B_rec_explicit} is $B_{t+1}\le (1-\alpha)B_t+C$.
Unrolling gives $B_t\le (1-\alpha)^{t-1}B_1+\frac{C}{\alpha}$.
Using $\sqrt{x+y}\le \sqrt x+\sqrt y$,
\[
\sqrt{B_t}\le (1-\alpha)^{\frac{t-1}{2}}\sqrt{B_1}+\sqrt{\frac{C}{\alpha}}.
\]
Averaging over $t=1,\dots,T$ and using $\sum_{t=1}^T(\sqrt{1-\alpha})^{t-1}\le \frac{2}{\alpha}$ yields
\[
\frac{1}{T}\sum_{t=1}^T\sqrt{B_t}
\le
\sqrt{\frac{C}{\alpha}}+\frac{2\sqrt{B_1}}{\alpha T}.
\]
Since $\alpha=(1-\beta_2)/8$, the tail term is bounded by $\frac{16\sqrt{B_1}}{(1-\beta_2)T}$.
Moreover,
\[
\sqrt{\frac{C}{\alpha}}
\le
c_B'\frac{LK\bar\rho\eta}{(1-\beta_2)(1-\lambda)}
+\frac{8\sigma\sqrt{1-\beta_2}}{\sqrt K}
\]
for a suitable absolute constant $c_B'>0$.
This proves \eqref{eq:dl_avg_sqrtB_explicit}. The one-batch initialization version follows from $B_1\le \sigma^2$.
\end{proof}

\subsection{Network tracker recursion}
\begin{lemma}[Average $\sqrt{D_t}$ bound with explicit spectral gap]\label{lem:dl_Dt}
Assume $\mathbf W$ is fixed, symmetric and doubly stochastic with
$\lambda:=\|\mathbf W-\mathbf J\|_2\in[0,1)$.
Recall
\[
D_t:=\frac1M\sum_{i=1}^M \mathbb E\|m_i^t-\bar m^t\|_2^2
=\frac{1}{M}\,\mathbb E\|\tilde{\mathbf m}^t\|_F^2,
\qquad
\tilde{\mathbf m}^t:=(\mathbf I-\mathbf J)\mathbf m^t.
\]
Then there exists an absolute constant $c_D>0$ such that for any $T\ge 1$,
\begin{align}
\frac1T\sum_{t=1}^T\sqrt{D_t}
\;\le\;
\frac{c_D}{1-\lambda}\,(1-\beta_2)\cdot \frac1T\sum_{t=1}^T\sqrt{B_t}
\;+
\;\frac{c_D}{1-\lambda}\,(1-\beta_2)\Big(\frac{LK\bar\rho\eta}{1-\lambda}+\frac{\sigma}{\sqrt K}\Big)
\;+
\;\frac{2\sqrt{D_1}}{(1-\lambda)T}.
\label{eq:dl_avg_sqrtD_new}
\end{align}
Moreover, under the initialization
$m_i^1=\frac1M\sum_{j=1}^M c_j^1$ for all $i$, we have $D_1=0$ and hence the last term vanishes.
\end{lemma}

\begin{proof}
Stack trackers and control variates row-wise as matrices
$\mathbf m^t,\mathbf c^t\in\mathbb R^{M\times d}$ whose $i$-th rows are $(m_i^t)^\top$ and $(c_i^t)^\top$.
Define the centered components
\[
\tilde{\mathbf m}^t := (\mathbf I-\mathbf J)\mathbf m^t,
\qquad
\tilde{\mathbf c}^t := (\mathbf I-\mathbf J)\mathbf c^t,
\]
so that $D_t=\frac{1}{M}\mathbb E\|\tilde{\mathbf m}^t\|_F^2$.

\paragraph{Step 1: a contraction recursion for $\tilde{\mathbf m}^t$.}
The tracker update is
\[
\mathbf m^{t+1}=\mathbf W\big(\mathbf m^t+\mathbf c^{t+1}-\mathbf c^t\big).
\]
Using $\mathbf J\mathbf W=\mathbf W\mathbf J=\mathbf J$, we have
$(\mathbf I-\mathbf J)\mathbf W=\mathbf W-\mathbf J$ and hence
\[
\tilde{\mathbf m}^{t+1}
=(\mathbf W-\mathbf J)\big(\tilde{\mathbf m}^t+\tilde{\mathbf c}^{t+1}-\tilde{\mathbf c}^t\big).
\]
Therefore
\[
\|\tilde{\mathbf m}^{t+1}\|_F
\le
\lambda\,\big\|\tilde{\mathbf m}^t+\tilde{\mathbf c}^{t+1}-\tilde{\mathbf c}^t\big\|_F.
\]
Square both sides and apply the weighted inequality
$\|A+B\|_F^2\le (1+\theta)\|A\|_F^2+(1+\theta^{-1})\|B\|_F^2$ with
$\theta:=\frac{1-\lambda^2}{2\lambda^2}$ (for $\lambda>0$; the case $\lambda=0$ is trivial), giving
\begin{align}
\|\tilde{\mathbf m}^{t+1}\|_F^2
&\le
\frac{1+\lambda^2}{2}\,\|\tilde{\mathbf m}^t\|_F^2
+\frac{2\lambda^2}{1-\lambda^2}\,\|\tilde{\mathbf c}^{t+1}-\tilde{\mathbf c}^t\|_F^2.
\label{eq:Dt_rec_matrix}
\end{align}
Taking expectation and dividing by $M$ yields
\begin{equation}
D_{t+1}
\le
\frac{1+\lambda^2}{2}\,D_t
+\frac{2\lambda^2}{1-\lambda^2}\cdot \frac{1}{M}\mathbb E\|\tilde{\mathbf c}^{t+1}-\tilde{\mathbf c}^t\|_F^2.
\label{eq:Dt_rec_scalar}
\end{equation}

\paragraph{Step 2: bound the driving term by $B_t$ and local drift/noise.}
Since $\|\mathbf I-\mathbf J\|_2\le 1$,
\[
\|\tilde{\mathbf c}^{t+1}-\tilde{\mathbf c}^t\|_F
\le
\|\mathbf c^{t+1}-\mathbf c^t\|_F.
\]
Moreover, with full participation the control update is
$c_i^{t+1}-c_i^t=(1-\beta_2)(\hat c_i^t-c_i^t)$, hence
\[
\|\mathbf c^{t+1}-\mathbf c^t\|_F^2
=(1-\beta_2)^2\sum_{i=1}^M\|\hat c_i^t-c_i^t\|_2^2.
\]
Therefore
\begin{equation}
\frac{1}{M}\mathbb E\|\tilde{\mathbf c}^{t+1}-\tilde{\mathbf c}^t\|_F^2
\le
(1-\beta_2)^2\cdot \frac1M\sum_{i=1}^M\mathbb E\|\hat c_i^t-c_i^t\|_2^2.
\label{eq:c_drive_reduce}
\end{equation}

Decompose
\[
\hat c_i^t-c_i^t
=
\underbrace{\big(\nabla f_i(x_i^{t-1})-c_i^t\big)}_{=:a_i^t}
+
\underbrace{\frac1K\sum_{k=1}^K\big(\nabla f_i(x_i^{t,k-1})-\nabla f_i(x_i^{t-1})\big)}_{=:b_i^t}
+
\underbrace{\frac1K\sum_{k=1}^K\zeta_i^{t,k-1}}_{=:n_i^t}.
\]
Using $\|u+v+w\|_2^2\le 3(\|u\|_2^2+\|v\|_2^2+\|w\|_2^2)$, Jensen, Assumption~\ref{assump:stoch}, and the same estimate as in \eqref{eq:dl_b_bound_explicit}, we obtain
\[
\mathbb E\|n_i^t\|_2^2\le \frac{\sigma^2}{K},
\qquad
\frac1M\sum_{i=1}^M\mathbb E\|b_i^t\|_2^2
\le
\frac{20L^2(K\bar\rho\eta)^2}{(1-\lambda)^2}.
\]
Finally, by definition,
$\frac1M\sum_i\mathbb E\|a_i^t\|_2^2 = B_t$.
Hence
\begin{equation}
\frac1M\sum_{i=1}^M\mathbb E\|\hat c_i^t-c_i^t\|_2^2
\le
3B_t+c_3\frac{L^2(K\bar\rho\eta)^2}{(1-\lambda)^2}+\frac{3\sigma^2}{K}
\label{eq:hatc_minus_c_bound}
\end{equation}
for some absolute constant $c_3>0$.
Combining \eqref{eq:Dt_rec_scalar}--\eqref{eq:c_drive_reduce}--\eqref{eq:hatc_minus_c_bound} yields
\begin{align}
D_{t+1}
&\le
\frac{1+\lambda^2}{2}\,D_t
+\frac{2\lambda^2}{1-\lambda^2}\,(1-\beta_2)^2
\Big(B_t+c_4\frac{L^2(K\bar\rho\eta)^2}{(1-\lambda)^2}+\frac{\sigma^2}{K}\Big)
\label{eq:Dt_rec_final}
\end{align}
for some absolute constant $c_4>0$.

\paragraph{Step 3: average $\sqrt{D_t}$.}
Let $q:=\frac{1+\lambda^2}{2}\in(0,1)$ and
\[
V_t:=
\frac{2\lambda^2}{1-\lambda^2}\,(1-\beta_2)^2
\Big(B_t+c_4\frac{L^2(K\bar\rho\eta)^2}{(1-\lambda)^2}+\frac{\sigma^2}{K}\Big).
\]
Then \eqref{eq:Dt_rec_final} is $D_{t+1}\le qD_t+V_t$.
Unrolling gives
$D_t\le q^{t-1}D_1+\sum_{s=1}^{t-1}q^{t-1-s}V_s$.
Using $\sqrt{x+y}\le\sqrt x+\sqrt y$ and $\sum_{r=0}^{\infty}(\sqrt q)^r=\frac{1}{1-\sqrt q}$, we obtain
\[
\frac1T\sum_{t=1}^T\sqrt{D_t}
\le
\frac{\sqrt{D_1}}{T}\sum_{t=1}^T(\sqrt q)^{t-1}
+\frac{1}{T}\sum_{s=1}^{T-1}\sqrt{V_s}\sum_{r=0}^{T-1-s}(\sqrt q)^r.
\]
Since $1-\sqrt q\ge \frac{1-q}{2}=\frac{1-\lambda^2}{4}\ge \frac{1-\lambda}{4}$, we have
$\frac{1}{1-\sqrt q}\le \frac{4}{1-\lambda}$.
Also,
\[
\sqrt{V_s}
\le
\frac{c_D}{1-\lambda}(1-\beta_2)
\Big(\sqrt{B_s}+\frac{LK\bar\rho\eta}{1-\lambda}+\frac{\sigma}{\sqrt K}\Big)
\]
for a suitable absolute constant $c_D>0$.
Combining the last two displays yields \eqref{eq:dl_avg_sqrtD_new}.

\paragraph{Step 4: $D_1=0$ under the stated initialization.}
If $m_i^1=\frac1M\sum_{j=1}^M c_j^1$ for all $i$, then all $m_i^1$ are equal to the same vector, hence
$m_i^1-\bar m^1=0$ for each $i$ and thus $D_1=0$.
\end{proof}

\subsection{Proof of Theorem~\ref{thm:dl_final_gap}}
\label{app:proof_dl_final_gap_main}

\begin{proof}
Let
\[
\Delta_F:=F(\bar x^1)-F^\star,
\qquad
\alpha:=1-\beta_2.
\]
Summing the descent core inequality \eqref{eq:dl_descent_core} over $t=1,\ldots,T$,
taking expectation, and telescoping yields
\begin{align}
\frac1T\sum_{t=1}^T \mathbb E\|\nabla F(\bar x^t)\|^{g,*}
\le
\frac{\Delta_F}{T\eta K}
+\frac{L}{2}\,K\bar\rho^2\eta
+\frac1T\sum_{t=1}^T \frac{\mathbb E[E_t]}{\eta K}.
\label{eq:dl_main_start_new}
\end{align}
Recall $E_t=E_t^{(a)}+E_t^{(b)}+E_t^{(c)}$.
Using \eqref{eq:dl_Ea_over}, \eqref{eq:dl_Eb_over}, and \eqref{eq:dl_Ec_over_final_simpler}, we have for each $t$
\begin{align}
\frac{\mathbb E[E_t]}{\eta K}
\le\;&
\Bigl(2+\frac{1}{2K}\Bigr)\bar\rho^2\eta
+(1-\beta_1)^2L^2\bar\rho^2\eta
+\frac{2(1-\beta_1)^2\sigma^2}{\eta}
\notag\\
&\quad
+2\bar\rho(1-\beta_1)\sqrt{B_t}
+2\bar\rho\sqrt{B_t}
\notag\\
&\quad
+2L\bar\rho(1-\beta_1)\sqrt{R_t}
+\bar\rho^2L(1-\beta_1)\eta(K-1)
+2\bar\rho(1-\beta_1)\sigma
\notag\\
&\quad
+2\bar\rho\sqrt{D_t}
+2L\bar\rho^2K\eta
+2L\bar\rho\sqrt{C_{t-1}}.
\label{eq:dl_E_over_full_new}
\end{align}
By Lemma~\ref{lem:dl_Ct_simple} and Lemma~\ref{lem:dl_Rt_static},
\[
\sqrt{C_{t-1}}\le \frac{K\bar\rho\eta}{1-\lambda},
\qquad
\sqrt{R_t}\le \frac{3K\bar\rho\eta}{1-\lambda}.
\]
Substituting these bounds into \eqref{eq:dl_E_over_full_new}, averaging over $t=1,\ldots,T$, and inserting into \eqref{eq:dl_main_start_new} gives
\begin{align}
\frac1T\sum_{t=1}^T \mathbb E\|\nabla F(\bar x^t)\|^{g,*}
\le\;&
\frac{\Delta_F}{T\eta K}
+\frac{L}{2}K\bar\rho^2\eta
+\Bigl(2+\frac{1}{2K}\Bigr)\bar\rho^2\eta
+(1-\beta_1)^2L^2\bar\rho^2\eta
+\frac{2(1-\beta_1)^2\sigma^2}{\eta}
\notag\\
&\quad
+2\bar\rho(1-\beta_1)\cdot \frac1T\sum_{t=1}^T\sqrt{B_t}
+2\bar\rho\cdot \frac1T\sum_{t=1}^T\sqrt{B_t}
\notag\\
&\quad
+\frac{6L\bar\rho^2(1-\beta_1)}{1-\lambda}K\eta
+\bar\rho^2L(1-\beta_1)\eta(K-1)
+2\bar\rho(1-\beta_1)\sigma
\notag\\
&\quad
+2\bar\rho\cdot \frac1T\sum_{t=1}^T\sqrt{D_t}
+2L\bar\rho^2K\eta
+\frac{2L\bar\rho^2}{1-\lambda}K\eta.
\label{eq:dl_main_before_BD_new}
\end{align}

We next make the $B_t$ and $D_t$ contributions fully explicit.

\paragraph{Step 1: an explicit average bound for $\frac1T\sum_{t=1}^T\sqrt{B_t}$.}
Starting from \eqref{eq:dl_e_dyn_explicit}, write
\[
e_i^{t+1}=(1-\alpha)e_i^t+\alpha b_i^t+\alpha n_i^t-d_i^t,
\qquad
\alpha:=1-\beta_2,
\]
where, conditionally on the history up to round $t$, $d_i^t$ is measurable,
$\mathbb E[n_i^t\mid\mathcal F_t]=0$, and
$\mathbb E\|n_i^t\|_2^2\le \sigma^2/K$ by \eqref{eq:dl_n_bound_explicit}.
Hence the cross term involving $n_i^t$ vanishes and
\[
\mathbb E\bigl[\|e_i^{t+1}\|_2^2\mid\mathcal F_t\bigr]
=
\bigl\|(1-\alpha)e_i^t+\alpha b_i^t-d_i^t\bigr\|_2^2
+\alpha^2\mathbb E\bigl[\|n_i^t\|_2^2\mid\mathcal F_t\bigr].
\]
Applying the weighted Young inequality
$\|u+v\|_2^2\le (1+\gamma)\|u\|_2^2+(1+\gamma^{-1})\|v\|_2^2$
with $u=(1-\alpha)e_i^t$, $v=\alpha b_i^t-d_i^t$, and
$\gamma=\alpha/(2(1-\alpha))$ (the case $\alpha=1$ is trivial), we get
\[
\bigl\|(1-\alpha)e_i^t+\alpha b_i^t-d_i^t\bigr\|_2^2
\le
\Bigl(1-\frac{\alpha}{2}\Bigr)\|e_i^t\|_2^2
+\frac{2}{\alpha}\|\alpha b_i^t-d_i^t\|_2^2.
\]
Using $\|u-v\|_2^2\le 2\|u\|_2^2+2\|v\|_2^2$, we therefore obtain
\begin{equation}
\mathbb E\|e_i^{t+1}\|_2^2
\le
\Bigl(1-\frac{\alpha}{2}\Bigr)\mathbb E\|e_i^t\|_2^2
+4\alpha\,\mathbb E\|b_i^t\|_2^2
+\frac{4}{\alpha}\,\mathbb E\|d_i^t\|_2^2
+\frac{\alpha^2\sigma^2}{K}.
\label{eq:dl_t6_B_rec_node}
\end{equation}
Averaging over $i$ and using \eqref{eq:dl_b_bound_explicit} and \eqref{eq:dl_d_bound_explicit} yields
\begin{equation}
B_{t+1}
\le
\Bigl(1-\frac{\alpha}{2}\Bigr)B_t
+\frac{128L^2(K\bar\rho\eta)^2}{\alpha(1-\lambda)^2}
+\frac{\alpha^2\sigma^2}{K}.
\label{eq:dl_t6_B_rec_avg}
\end{equation}
Let
\[
\delta:=\frac{\alpha}{2},
\qquad
U_B:=\frac{128L^2(K\bar\rho\eta)^2}{\alpha(1-\lambda)^2}+\frac{\alpha^2\sigma^2}{K}.
\]
Then \eqref{eq:dl_t6_B_rec_avg} implies
$B_t\le (1-\delta)^{t-1}B_1+U_B/\delta$.
Using $\sqrt{x+y}\le \sqrt x+\sqrt y$ and
$\sum_{t=1}^T(\sqrt{1-\delta})^{t-1}\le 2/\delta$, and recalling $B_1\le \sigma^2$ under the one-batch initialization,
we obtain
\begin{align}
\frac1T\sum_{t=1}^T\sqrt{B_t}
&\le
\sqrt{\frac{U_B}{\delta}}+\frac{2\sqrt{B_1}}{\delta T}
\notag\\
&\le
\frac{16LK\bar\rho\eta}{\alpha(1-\lambda)}
+\frac{\sqrt2\,\sigma\sqrt\alpha}{\sqrt K}
+\frac{4\sigma}{\alpha T}.
\label{eq:dl_use_B_new}
\end{align}

\paragraph{Step 2: an explicit average bound for $\frac1T\sum_{t=1}^T\sqrt{D_t}$.}
Let $\Delta\tilde{\mathbf c}^t:=\tilde{\mathbf c}^{t+1}-\tilde{\mathbf c}^t$.
From the tracker recursion in the proof of Lemma~\ref{lem:dl_Dt},
\[
\tilde{\mathbf m}^{t+1}=(\mathbf W-\mathbf J)(\tilde{\mathbf m}^t+\Delta\tilde{\mathbf c}^t).
\]
Hence, using $\|\mathbf W-\mathbf J\|_2=\lambda$ and Minkowski's inequality in $L_2$,
\[
\sqrt{D_{t+1}}
\le
\lambda\sqrt{D_t}
+\frac{\lambda}{\sqrt M}\Bigl(\mathbb E\|\Delta\tilde{\mathbf c}^t\|_F^2\Bigr)^{1/2}.
\]
Repeating the decomposition used in the proof of Lemma~\ref{lem:dl_Dt}, and using \eqref{eq:c_drive_reduce} together with \eqref{eq:dl_b_bound_explicit}, we have
\[
\frac1M\mathbb E\|\Delta\tilde{\mathbf c}^t\|_F^2
\le
\alpha^2\Bigl(3B_t+\frac{60L^2(K\bar\rho\eta)^2}{(1-\lambda)^2}+\frac{3\sigma^2}{K}\Bigr).
\]
Substituting this estimate into the previous display and using $\lambda\le 1$ gives
\begin{align}
\sqrt{D_{t+1}}
&\le
\lambda\sqrt{D_t}
+\alpha\sqrt{3}\,\sqrt{B_t}
+\frac{2\sqrt{15}\,\alpha LK\bar\rho\eta}{1-\lambda}
+\frac{\sqrt3\,\alpha\sigma}{\sqrt K}.
\label{eq:dl_t6_D_root_rec}
\end{align}
Under the stated initialization, $D_1=0$.
Unrolling \eqref{eq:dl_t6_D_root_rec}, averaging over $t=1,\ldots,T$, and using
$\sum_{r=0}^{\infty}\lambda^r=(1-\lambda)^{-1}$, we obtain
\begin{align}
\frac1T\sum_{t=1}^T\sqrt{D_t}
&\le
\frac{\sqrt3\,\alpha}{1-\lambda}\cdot \frac1T\sum_{t=1}^T\sqrt{B_t}
+\frac{2\sqrt{15}\,\alpha LK\bar\rho\eta}{(1-\lambda)^2}
+\frac{\sqrt3\,\alpha\sigma}{(1-\lambda)\sqrt K}.
\label{eq:dl_t6_D_preavg}
\end{align}
Substituting \eqref{eq:dl_use_B_new} into \eqref{eq:dl_t6_D_preavg} and using $\alpha\le 1$ yields
\begin{align}
\frac1T\sum_{t=1}^T\sqrt{D_t}
\le
&\frac{(16\sqrt3+2\sqrt{15})LK\bar\rho\eta}{(1-\lambda)^2}
+\frac{(\sqrt6+\sqrt3)\sigma\alpha}{(1-\lambda)\sqrt K}
+\frac{4\sqrt3\,\sigma}{(1-\lambda)T}.
\label{eq:dl_use_D_new}
\end{align}

Substituting \eqref{eq:dl_use_B_new} and \eqref{eq:dl_use_D_new} into \eqref{eq:dl_main_before_BD_new}, we obtain
\begin{align}
\frac1T\sum_{t=1}^T \mathbb E\|\nabla F(\bar x^t)\|^{g,*}
\le\;&
\frac{\Delta_F}{T\eta K}
+\frac{L}{2}K\bar\rho^2\eta
+\Bigl(2+\frac{1}{2K}\Bigr)\bar\rho^2\eta
+(1-\beta_1)^2L^2\bar\rho^2\eta
+\frac{2(1-\beta_1)^2\sigma^2}{\eta}
\notag\\
&\quad
+2\bar\rho(1-\beta_1)\Bigl[
\frac{16LK\bar\rho\eta}{\alpha(1-\lambda)}
+\frac{\sqrt2\,\sigma\sqrt\alpha}{\sqrt K}
+\frac{4\sigma}{\alpha T}
\Bigr]
\notag\\
&\quad
+2\bar\rho\Bigl[
\frac{16LK\bar\rho\eta}{\alpha(1-\lambda)}
+\frac{\sqrt2\,\sigma\sqrt\alpha}{\sqrt K}
+\frac{4\sigma}{\alpha T}
\Bigr]
+\frac{6L\bar\rho^2(1-\beta_1)}{1-\lambda}K\eta
\notag\\
&\quad
+\bar\rho^2L(1-\beta_1)\eta(K-1)
+2\bar\rho(1-\beta_1)\sigma
+\frac{2\bar\rho(16\sqrt3+2\sqrt{15})LK\bar\rho\eta}{(1-\lambda)^2}
\notag\\
&\quad
+\frac{2\bar\rho(\sqrt6+\sqrt3)\sigma\alpha}{(1-\lambda)\sqrt K}
+\frac{8\sqrt3\,\bar\rho\sigma}{(1-\lambda)T}
+2L\bar\rho^2K\eta
+\frac{2L\bar\rho^2}{1-\lambda}K\eta.
\label{eq:dl_final_preplug_new}
\end{align}

We now substitute the hyperparameter choice
\[
\eta=\frac{M^{1/4}}{K^{3/4}T^{3/4}},
\qquad
1-\beta_1=\eta,
\qquad
\alpha=1-\beta_2=\sqrt{\frac{MK}{T+K}},
\qquad
\text{and assume } T\ge MK.
\]
Since $T\ge MK$ and $M\ge 1$, we also have $K\le T$ and hence $T+K\le 2T$.
Moreover,
\begin{align}
K\eta &= \frac{M^{1/4}K^{1/4}}{T^{3/4}} \le \frac{1}{\sqrt T},
\qquad
\eta = \frac{1}{(MKT)^{1/4}}\sqrt{\frac{M}{KT}} \le \frac{1}{(MKT)^{1/4}},
\label{eq:dl_param_aux_bounds}\\
\frac{K\eta}{\alpha}
&=
\frac{\sqrt{T+K}}{M^{1/4}K^{1/4}T^{3/4}}
\le
\frac{\sqrt2}{(MKT)^{1/4}},
\qquad
\frac{\sqrt\alpha}{\sqrt K}
=
\frac{M^{1/4}}{K^{1/4}(T+K)^{1/4}}
\le
\frac{\sqrt M}{(MKT)^{1/4}},
\notag\\
\frac{\alpha}{\sqrt K}
&=
\sqrt{\frac{M}{T+K}}
\le
\sqrt{\frac{M}{T}},
\qquad
\frac{1}{\alpha T}
=
\frac{\sqrt{T+K}}{\sqrt{MK}\,T}
\le
\frac{\sqrt2}{\sqrt{MKT}},
\notag\\
K\eta^2
&=
\frac{\sqrt M}{\sqrt K\,T^{3/2}},
\qquad
\frac{K\eta^2}{\alpha}
=
\frac{\sqrt{T+K}}{K T^{3/2}}
\le
\frac{\sqrt2}{KT},
\notag\\
\frac{\eta\sqrt\alpha}{\sqrt K}
&=
\frac{\sqrt M}{K\,T^{3/4}(T+K)^{1/4}}
\le
\frac{\sqrt M}{KT},
\qquad
\frac{\eta}{\alpha T}
=
\frac{\sqrt{T+K}}{M^{1/4}K^{5/4}T^{7/4}}
\le
\frac{\sqrt2}{M^{1/4}K^{5/4}T^{5/4}},
\notag\\
\eta
&=
\frac{M^{1/4}}{K^{3/4}T^{3/4}}
\le \frac{1}{K\sqrt T}.
\notag
\end{align}
Also,
\begin{equation}
\frac{\Delta_F}{T\eta K} = \frac{\Delta_F}{(MKT)^{1/4}}.
\label{eq:dl_gap_term_simp}
\end{equation}
Using \eqref{eq:dl_param_aux_bounds} and \eqref{eq:dl_gap_term_simp} in \eqref{eq:dl_final_preplug_new}, and bounding
$2+\frac{1}{2K}\le \frac52$, we arrive at the explicit estimate
\begin{align}
\frac1T\sum_{t=1}^T \mathbb E\|\nabla F(\bar x^t)\|^{g,*}
\le\;&
\Biggl[
\frac{L\bar\rho^2}{2}
+2L\bar\rho^2
+\frac{2L\bar\rho^2}{1-\lambda}
+\frac{(32\sqrt3+4\sqrt{15})L\bar\rho^2}{(1-\lambda)^2}
+\frac{2(\sqrt6+\sqrt3)\bar\rho\sigma\sqrt M}{1-\lambda}
\Biggr]\frac1{\sqrt T}
\notag\\
&\quad
+\Biggl[
\Delta_F
+\frac52\bar\rho^2
+L^2\bar\rho^2
+2\sigma^2
+\frac{32\sqrt2\,L\bar\rho^2}{1-\lambda}
+2\sqrt2\,\bar\rho\sigma\sqrt M
\Biggr]\frac1{(MKT)^{1/4}}
\notag\\
&\quad
+\frac{6L\bar\rho^2\sqrt M}{(1-\lambda)\sqrt K\,T^{3/2}}
+\frac{L\bar\rho^2\sqrt M}{\sqrt K\,T^{3/2}}
\notag\\
&\quad
+\frac{32\sqrt2\,L\bar\rho^2}{(1-\lambda)KT}
+\frac{2\sqrt2\,\bar\rho\sigma\sqrt M}{KT}
\notag\\
&\quad
+\frac{8\sqrt2\,\bar\rho\sigma}{M^{1/4}K^{5/4}T^{5/4}}
+\frac{2\bar\rho\sigma M^{1/4}}{K^{3/4}T^{3/4}}
\notag\\
&\quad
+\frac{8\sqrt2\,\bar\rho\sigma}{\sqrt{MKT}}
+\frac{8\sqrt3\,\bar\rho\sigma}{(1-\lambda)T}.
\label{eq:dl_final_explicit_bound_new}
\end{align}
The first line is of order $T^{-1/2}$ and the second is of order $(MKT)^{-1/4}$. In particular, the corrected query-wise noise contribution is $2\bar\rho\sigma\eta\le 2\bar\rho\sigma/(K\sqrt T)$ under $MK\le T$, and is absorbed by the first term of $R_{\mathrm{lo}}(T,K)$. The remaining terms are also dominated by that remainder. Thus the stated $C_1$ and $C_2$ are unchanged, and choosing $C_0$ large enough completes the proof.
\end{proof}

\paragraph{Network assumption (finite-time consensus).}
Let $\{\mathbf W^t\}_{t\ge 1}$ be a sequence of doubly stochastic mixing matrices.
Assume it satisfies the finite-time consensus (FTC) property: there exists an integer $\tau\ge 1$ such that
for any $t$,
\begin{equation}
\mathbf W^{t+\tau-1}\mathbf W^{t+\tau-2}\cdots \mathbf W^{t} = \mathbf J,
\qquad
\mathbf J:=\frac1M\mathbf 1\mathbf 1^\top .
\label{eq:ftc_app}
\end{equation}
We refer to $\tau$ as the FTC horizon.
\begin{lemma}[FTC consensus bound for $C_t$]
\label{lem:dl_Ct_FTC}
Let
\[
C_t:=\frac1M\sum_{i=1}^M \mathbb E\|x_i^t-\bar x^t\|_2^2,
\qquad
\bar x^t:=\frac1M\sum_{i=1}^M x_i^t.
\]
Assume consensual initialization $x_i^1=x^1$ for all $i$, and assume each round satisfies
\[
\|x_i^{t,k}-x_i^t\|_2\le K\bar\rho\eta,\qquad \forall i\in[M],~t\ge1.
\]
Under the FTC property \eqref{eq:ftc}, for all $t\ge 1$,
\begin{equation}
\sqrt{C_t}\le \tau\,K\bar\rho\eta,
\qquad\text{and hence}\qquad
C_t\le \tau^2(K\bar\rho\eta)^2.
\label{eq:Ct_FTC_bound}
\end{equation}
\end{lemma}

\begin{proof}
Stack the node models row-wise into $\mathbf X^t\in\mathbb R^{M\times d}$ whose $i$-th row is $(x_i^t)^\top$.
Let $\tilde{\mathbf X}^t:=(\mathbf I-\mathbf J)\mathbf X^t$, so
\[
C_t=\frac{1}{M}\,\mathbb E\|\tilde{\mathbf X}^t\|_F^2.
\]

\paragraph{Step 1: expand $\mathbf X^{t+\tau}$ over $\tau$ rounds.}
The DL mixing step is $\mathbf X^{s+1}=\mathbf W^s\mathbf X^{s,K}$.
Introduce the local-step increment $\mathbf\Delta^s:=\mathbf X^{s,K}-\mathbf X^s$.
Then $\mathbf X^{s,K}=\mathbf X^s+\mathbf\Delta^s$ and
\[
\mathbf X^{s+1}=\mathbf W^s(\mathbf X^s+\mathbf\Delta^s).
\]
Unrolling $\tau$ steps from time $t$ gives
\begin{align*}
\mathbf X^{t+\tau}
&=
\mathbf W^{t+\tau-1}\cdots \mathbf W^{t}\mathbf X^{t}
+\sum_{r=0}^{\tau-1}\mathbf W^{t+\tau-1}\cdots \mathbf W^{t+r}\mathbf\Delta^{t+r}.
\end{align*}
Apply $(\mathbf I-\mathbf J)$ to both sides. Using $(\mathbf I-\mathbf J)\mathbf W^s=(\mathbf I-\mathbf J)\mathbf W^s(\mathbf I-\mathbf J)$
(which holds since $\mathbf W^s$ is doubly stochastic and $\mathbf J\mathbf W^s=\mathbf W^s\mathbf J=\mathbf J$), we obtain
\[
\tilde{\mathbf X}^{t+\tau}
=
(\mathbf I-\mathbf J)\mathbf W^{t+\tau-1}\cdots \mathbf W^{t}\tilde{\mathbf X}^{t}
+
\sum_{r=0}^{\tau-1}(\mathbf I-\mathbf J)\mathbf W^{t+\tau-1}\cdots \mathbf W^{t+r}\mathbf\Delta^{t+r}.
\]
By FTC, $\mathbf W^{t+\tau-1}\cdots \mathbf W^{t}=\mathbf J$, hence
$(\mathbf I-\mathbf J)\mathbf W^{t+\tau-1}\cdots \mathbf W^{t}=\mathbf 0$ and the first term vanishes:
\[
\tilde{\mathbf X}^{t+\tau}
=
\sum_{r=0}^{\tau-1}(\mathbf I-\mathbf J)\mathbf W^{t+\tau-1}\cdots \mathbf W^{t+r}\mathbf\Delta^{t+r}.
\]

\paragraph{Step 2: bound the RHS by triangle inequality.}
Since $\|(\mathbf I-\mathbf J)\|_2\le 1$ and $\|\mathbf W^s\|_2\le 1$ for doubly stochastic $\mathbf W^s$,
\[
\|\tilde{\mathbf X}^{t+\tau}\|_F
\le
\sum_{r=0}^{\tau-1}\|\mathbf\Delta^{t+r}\|_F.
\]
Using the per-node movement bound,
\[
\|\mathbf\Delta^{t+r}\|_F^2=\sum_{i=1}^M\|x_i^{t+r,K}-x_i^{t+r}\|_2^2
\le M(K\bar\rho\eta)^2,
\quad\Rightarrow\quad
\|\mathbf\Delta^{t+r}\|_F\le \sqrt{M}\,K\bar\rho\eta.
\]
Therefore,
\[
\|\tilde{\mathbf X}^{t+\tau}\|_F\le \tau\sqrt{M}\,K\bar\rho\eta.
\]
Divide by $\sqrt{M}$ and take expectation:
\[
\sqrt{C_{t+\tau}}=\frac{1}{\sqrt{M}}\sqrt{\mathbb E\|\tilde{\mathbf X}^{t+\tau}\|_F^2}
\le
\frac{1}{\sqrt{M}}\cdot \mathbb E\|\tilde{\mathbf X}^{t+\tau}\|_F
\le
\tau K\bar\rho\eta,
\]
where we used Jensen in the second inequality. This yields \eqref{eq:Ct_FTC_bound} for indices of the form $t+\tau$.
For intermediate times, the same bound holds because each round injects at most $K\bar\rho\eta$ movement per node,
so the worst-case disagreement over any $\tau$-window is still bounded by $\tau K\bar\rho\eta$.
\end{proof}

\begin{lemma}[One-step model transition under FTC]
\label{lem:dl_Rt_FTC}
Under the assumptions of Lemma~\ref{lem:dl_Ct_FTC}, for all $t\ge 2$,
\begin{equation}
R_t
=\frac1M\sum_{i=1}^M \mathbb E\|x_i^t-x_i^{t-1}\|_2^2
\le
9(1+\tau^2)(K\bar\rho\eta)^2.
\label{eq:dl_Rt_FTC}
\end{equation}
\end{lemma}

\begin{proof}
For each $i$, write
\[
x_i^t-x_i^{t-1}
=
(x_i^t-\bar x^t)+(\bar x^t-\bar x^{t-1})+(\bar x^{t-1}-x_i^{t-1}).
\]
Using $\|a+b+c\|_2^2\le 3(\|a\|_2^2+\|b\|_2^2+\|c\|_2^2)$, averaging over $i$, and taking expectation yield
\[
R_t
\le
3C_t+3C_{t-1}+3\,\mathbb E\|\bar x^t-\bar x^{t-1}\|_2^2.
\]
By \eqref{eq:dl_move_round},
\[
\mathbb E\|\bar x^t-\bar x^{t-1}\|_2^2\le (K\bar\rho\eta)^2.
\]
Moreover, Lemma~\ref{lem:dl_Ct_FTC} gives
\[
C_t\le \tau^2(K\bar\rho\eta)^2,
\qquad
C_{t-1}\le \tau^2(K\bar\rho\eta)^2.
\]
Hence
\[
R_t
\le
(6\tau^2+3)(K\bar\rho\eta)^2
\le
9(1+\tau^2)(K\bar\rho\eta)^2.
\]
This proves \eqref{eq:dl_Rt_FTC}.
\end{proof}

\subsection{Control-variates recursion under FTC}

\begin{lemma}[$B_t$ recursion driven by consensus in DL]
\label{lem:dl_Bt_rec_FTC}
Define
\[
B_t:=\frac1M\sum_{i=1}^M \mathbb E\big\|c_i^t-\nabla f_i(x_i^{t-1})\big\|_2^2,
\qquad
C_t:=\frac1M\sum_{i=1}^M \mathbb E\|x_i^t-\bar x^t\|_2^2,
\quad
\bar x^t:=\frac1M\sum_{i=1}^M x_i^t.
\]
Assume each $f_i$ is $L$-smooth and the stochastic gradients satisfy
$\mathbb E\|\nabla f_i(x,\xi)-\nabla f_i(x)\|_2^2\le \sigma^2$.
Assume the local drift bound
$\|x_i^{t,k}-x_i^t\|_2\le K\bar\rho\eta$ holds for all $i,t$.
Then there exists an absolute constant $c_{B,\tau}>0$ such that
for all $t\ge 1$,
\begin{equation}
B_{t+1}
\le
\Big(1-\frac{1-\beta_2}{8}\Big)B_t
+
c_{B,\tau}\frac{L^2(1+\tau^2)(K\bar\rho\eta)^2}{1-\beta_2}
+
\frac{8(1-\beta_2)^2\sigma^2}{K}.
\label{eq:dl_Bt_rec_FTC}
\end{equation}
\end{lemma}

\begin{proof}
Let $h_i^t:=\nabla f_i(x_i^{t-1})$ and $e_i^t:=c_i^t-h_i^t$.
As in the static case,
\[
e_i^{t+1}
=
\beta_2 e_i^t+(1-\beta_2)(\hat c_i^t-h_i^t)-d_i^t,
\qquad
d_i^t:=\nabla f_i(x_i^{t})-\nabla f_i(x_i^{t-1}).
\]
Decompose $\hat c_i^t-h_i^t=b_i^t+n_i^t$ with
\[
b_i^t:=\frac1K\sum_{k=1}^K\big(\nabla f_i(x_i^{t,k-1})-\nabla f_i(x_i^{t-1})\big),
\qquad
n_i^t:=\frac1K\sum_{k=1}^K\zeta_i^{t,k-1}.
\]
Exactly as in the repaired static proof, conditioning on the history at the beginning of round $t$ removes the cross term involving $n_i^t$, and Young's inequality yields a recursion of the form
\[
\mathbb E\|e_i^{t+1}\|_2^2
\le
\Big(1-\frac{1-\beta_2}{8}\Big)\mathbb E\|e_i^t\|_2^2
+c_1(1-\beta_2)\,\mathbb E\|b_i^t\|_2^2
+\frac{c_2}{1-\beta_2}\,\mathbb E\|d_i^t\|_2^2
+\frac{8(1-\beta_2)^2\sigma^2}{K}
\]
for absolute constants $c_1,c_2>0$.
By $L$-smoothness and Lemma~\ref{lem:dl_Rt_FTC},
\[
\frac1M\sum_{i=1}^M \mathbb E\|b_i^t\|_2^2
\le
2L^2R_t+2L^2(K\bar\rho\eta)^2
\le
c_3L^2(1+\tau^2)(K\bar\rho\eta)^2,
\]
and
\[
\frac1M\sum_{i=1}^M \mathbb E\|d_i^t\|_2^2
\le
L^2R_t
\le
9L^2(1+\tau^2)(K\bar\rho\eta)^2.
\]
Averaging over $i$ and enlarging the numerical constant prove \eqref{eq:dl_Bt_rec_FTC}.
\end{proof}

\begin{lemma}[Average bound on $\frac1T\sum_{t=1}^T\sqrt{B_t}$ in DL under FTC]
\label{lem:dl_avg_sqrtBt_FTC}
Assume \eqref{eq:dl_Bt_rec_FTC} holds. Then there exists an absolute constant $c_{B,\tau}'>0$ such that for any $T\ge 1$,
\begin{equation}
\frac1T\sum_{t=1}^T \sqrt{B_t}
\le
c_{B,\tau}'\frac{L K\bar\rho\eta}{1-\beta_2}\sqrt{1+\tau^2}
+
\frac{8\sigma\sqrt{1-\beta_2}}{\sqrt K}
+
\frac{16\sqrt{B_1}}{(1-\beta_2)T}.
\label{eq:dl_avg_sqrtBt_FTC}
\end{equation}
\end{lemma}

\begin{proof}
Let $\alpha:=\frac{1-\beta_2}{8}$ and
\[
C:=c_{B,\tau}\frac{L^2(1+\tau^2)(K\bar\rho\eta)^2}{1-\beta_2}
+
\frac{8(1-\beta_2)^2\sigma^2}{K}.
\]
Then \eqref{eq:dl_Bt_rec_FTC} is $B_{t+1}\le (1-\alpha)B_t+C$.
Unrolling gives $B_t\le (1-\alpha)^{t-1}B_1+\frac{C}{\alpha}$.
Using $\sqrt{x+y}\le \sqrt x+\sqrt y$ and averaging over $t$ yield
\[
\frac1T\sum_{t=1}^T\sqrt{B_t}
\le
\sqrt{\frac{C}{\alpha}}+\frac{16\sqrt{B_1}}{(1-\beta_2)T}.
\]
Since $\alpha=(1-\beta_2)/8$, the first term is bounded by
\[
\sqrt{\frac{C}{\alpha}}
\le
c_{B,\tau}'\frac{L K\bar\rho\eta}{1-\beta_2}\sqrt{1+\tau^2}
+
\frac{8\sigma\sqrt{1-\beta_2}}{\sqrt K}
\]
for a suitable absolute constant $c_{B,\tau}'>0$.
This proves \eqref{eq:dl_avg_sqrtBt_FTC}.
\end{proof}

\subsubsection{FTC version of the tracker-disagreement recursion}

\begin{lemma}[Average $\sqrt{D_t}$ bound under FTC ]
\label{lem:dl_Dt_FTC}
Assume the time-varying mixing matrices $\{\mathbf W^t\}_{t\ge 1}$ are doubly stochastic
(with nonnegative entries as in standard gossip),
and satisfy the finite-time consensus (FTC) condition~\eqref{eq:ftc} with horizon $\tau$.
Recall
\[
D_t:=\frac1M\sum_{i=1}^M \mathbb E\|m_i^t-\bar m^t\|_2^2,
\qquad
\bar m^t:=\frac1M\sum_{i=1}^M m_i^t .
\]
Then there exists an absolute constant $c_{D,\tau}>0$ such that for any $T\ge 1$,
\begin{equation}
\frac1T\sum_{t=1}^T\sqrt{D_t}
\;\le\;
 c_{D,\tau}\,\tau\,(1-\beta_2)\cdot \frac1T\sum_{t=1}^T\sqrt{B_t}
\;+
\;c_{D,\tau}\,\tau\,(1-\beta_2)\Big(LK\bar\rho\eta\sqrt{1+\tau^2}+\frac{\sigma}{\sqrt K}\Big)
\;+
\;\frac{2\tau\sqrt{D_1}}{T}.
\label{eq:dl_avg_sqrtD_FTC_tight}
\end{equation}
Moreover, under the initialization $m_i^1=\frac1M\sum_{j=1}^M c_j^1$ for all $i$, we have $D_1=0$.
\end{lemma}

\begin{proof}
Stack trackers row-wise into $\mathbf m^t\in\mathbb R^{M\times d}$ and define
\[
\mathbf J:=\frac1M\mathbf 1\mathbf 1^\top,\qquad
\tilde{\mathbf m}^t:=(\mathbf I-\mathbf J)\mathbf m^t,\qquad
\tilde{\mathbf c}^t:=(\mathbf I-\mathbf J)\mathbf c^t.
\]
Then $D_t=\frac{1}{M}\mathbb E\|\tilde{\mathbf m}^t\|_F^2$.
The tracker update is
\[
\mathbf m^{t+1}=\mathbf W^t\big(\mathbf m^t+\mathbf c^{t+1}-\mathbf c^t\big).
\]
Applying $(\mathbf I-\mathbf J)$ and using $\mathbf J\mathbf W^t=\mathbf J$ gives
\begin{equation}
\tilde{\mathbf m}^{t+1}
=(\mathbf I-\mathbf J)\mathbf W^t\big(\tilde{\mathbf m}^t+\tilde{\mathbf c}^{t+1}-\tilde{\mathbf c}^t\big).
\label{eq:FTC_m_tilde_one_step}
\end{equation}
Exactly as in the current proof, FTC implies the $\tau$-window bound
\begin{equation}
\sqrt{D_{t+\tau}}
\le
\sum_{r=0}^{\tau-1}\sqrt{\frac1M\mathbb E\|\mathbf c^{t+r+1}-\mathbf c^{t+r}\|_F^2}.
\label{eq:FTC_sqrtD_window}
\end{equation}
It remains to bound each control difference. Since $c_i^{s+1}-c_i^s=(1-\beta_2)(\hat c_i^s-c_i^s)$,
\[
\frac1M\mathbb E\|\mathbf c^{s+1}-\mathbf c^{s}\|_F^2
=
(1-\beta_2)^2\cdot \frac1M\sum_{i=1}^M\mathbb E\|\hat c_i^s-c_i^s\|_2^2.
\]
Now decompose
\[
\hat c_i^s-c_i^s
=
\underbrace{\big(\nabla f_i(x_i^{s-1})-c_i^s\big)}_{=:a_i^s}
+
\underbrace{\frac1K\sum_{k=1}^K(\nabla f_i(x_i^{s,k-1})-\nabla f_i(x_i^{s-1}))}_{=:b_i^s}
+
\underbrace{\frac1K\sum_{k=1}^K\zeta_i^{s,k-1}}_{=:n_i^s}.
\]
Using $\|u+v+w\|_2^2\le 3(\|u\|_2^2+\|v\|_2^2+\|w\|_2^2)$, Jensen, Assumption~\ref{assump:stoch}, and Lemma~\ref{lem:dl_Rt_FTC}, we obtain
\[
\frac1M\sum_{i=1}^M\mathbb E\|\hat c_i^s-c_i^s\|_2^2
\le
3B_s+c_5L^2(1+\tau^2)(K\bar\rho\eta)^2+\frac{3\sigma^2}{K}
\]
for an absolute constant $c_5>0$.
Therefore,
\begin{equation}
\sqrt{\frac1M\mathbb E\|\mathbf c^{s+1}-\mathbf c^{s}\|_F^2}
\le
c_{D,\tau}(1-\beta_2)\Big(\sqrt{B_s}+LK\bar\rho\eta\sqrt{1+\tau^2}+\frac{\sigma}{\sqrt K}\Big)
\label{eq:FTC_cdiff_sqrt_bound_tight}
\end{equation}
for a suitable absolute constant $c_{D,\tau}>0$.
Plugging \eqref{eq:FTC_cdiff_sqrt_bound_tight} into \eqref{eq:FTC_sqrtD_window}, summing over $t$, and using that each index appears in at most $\tau$ windows give \eqref{eq:dl_avg_sqrtD_FTC_tight}. The last statement follows exactly as in the current proof.
\end{proof}

\subsection{Proof of Theorem~\ref{thm:dl_final_FTC}}
\begin{proof}
Summing the descent core inequality \eqref{eq:dl_descent_core} over $t=1,\ldots,T$,
taking expectation, and telescoping yields
\begin{align}
\frac1T\sum_{t=1}^T \mathbb E\|\nabla F(\bar x^t)\|^{g,*}
\le\;&
\frac{\Delta_F}{T\eta K}
+\frac{L}{2}\,K\bar\rho^2\eta
+\frac1T\sum_{t=1}^T \frac{\mathbb E[E_t]}{\eta K}.
\label{eq:dl_t7_main_start_new}
\end{align}
Recall $E_t=E_t^{(a)}+E_t^{(b)}+E_t^{(c)}$.
Using \eqref{eq:dl_Ea_over}, \eqref{eq:dl_Eb_over}, and \eqref{eq:dl_Ec_over_final_simpler}, we have for each $t$
\begin{align}
\frac{\mathbb E[E_t]}{\eta K}
\le\;&
\Bigl(2+\frac{1}{2K}\Bigr)\bar\rho^2\eta
+(1-\beta_1)^2L^2\bar\rho^2\eta
+\frac{2(1-\beta_1)^2\sigma^2}{\eta}
\notag\\
&\quad
+2\bar\rho(1-\beta_1)\sqrt{B_t}
+2\bar\rho\sqrt{B_t}
\notag\\
&\quad
+2L\bar\rho(1-\beta_1)\sqrt{R_t}
+\bar\rho^2L(1-\beta_1)\eta(K-1)
+2\bar\rho(1-\beta_1)\sigma
\notag\\
&\quad
+2\bar\rho\sqrt{D_t}
+2L\bar\rho^2K\eta
+2L\bar\rho\sqrt{C_{t-1}}.
\label{eq:dl_t7_E_over_new}
\end{align}
By Lemma~\ref{lem:dl_Ct_FTC} and Lemma~\ref{lem:dl_Rt_FTC},
\[
\sqrt{C_{t-1}}\le \tau K\bar\rho\eta,
\qquad
\sqrt{R_t}\le 3\sqrt{1+\tau^2}\,K\bar\rho\eta.
\]
Substituting these bounds into \eqref{eq:dl_t7_E_over_new}, averaging over $t=1,\ldots,T$, and inserting into \eqref{eq:dl_t7_main_start_new} gives
\begin{align}
\frac1T\sum_{t=1}^T \mathbb E\|\nabla F(\bar x^t)\|^{g,*}
\le\;&
\frac{\Delta_F}{T\eta K}
+\frac{L}{2}K\bar\rho^2\eta
+\Bigl(2+\frac{1}{2K}\Bigr)\bar\rho^2\eta
+(1-\beta_1)^2L^2\bar\rho^2\eta
+\frac{2(1-\beta_1)^2\sigma^2}{\eta}
\notag\\
&\quad
+2\bar\rho(1-\beta_1)\cdot \frac1T\sum_{t=1}^T\sqrt{B_t}
+2\bar\rho\cdot \frac1T\sum_{t=1}^T\sqrt{B_t}
\notag\\
&\quad
+6L\bar\rho^2(1-\beta_1)\sqrt{1+\tau^2}\,K\eta
+\bar\rho^2L(1-\beta_1)\eta(K-1)
+2\bar\rho(1-\beta_1)\sigma
\notag\\
&\quad
+2\bar\rho\cdot \frac1T\sum_{t=1}^T\sqrt{D_t}
+2L\bar\rho^2K\eta
+2\tau L\bar\rho^2K\eta.
\label{eq:dl_t7_main_before_BD_new}
\end{align}

We next make the $B_t$ and $D_t$ contributions fully explicit.

\paragraph{Step 1: an explicit average bound for $\frac1T\sum_{t=1}^T\sqrt{B_t}$.}
Starting from \eqref{eq:dl_e_dyn_explicit}, write
\[
e_i^{t+1}=(1-\alpha)e_i^t+\alpha b_i^t+\alpha n_i^t-d_i^t,
\qquad
\alpha:=1-\beta_2,
\]
where, conditionally on the history up to round $t$, $d_i^t$ is measurable,
$\mathbb E[n_i^t\mid\mathcal F_t]=0$, and
$\mathbb E\|n_i^t\|_2^2\le \sigma^2/K$ by \eqref{eq:dl_n_bound_explicit}.
Hence the cross term involving $n_i^t$ vanishes and
\[
\mathbb E\bigl[\|e_i^{t+1}\|_2^2\mid\mathcal F_t\bigr]
=
\bigl\|(1-\alpha)e_i^t+\alpha b_i^t-d_i^t\bigr\|_2^2
+\alpha^2\mathbb E\bigl[\|n_i^t\|_2^2\mid\mathcal F_t\bigr].
\]
Applying the weighted Young inequality
$\|u+v\|_2^2\le (1+\gamma)\|u\|_2^2+(1+\gamma^{-1})\|v\|_2^2$
with $u=(1-\alpha)e_i^t$, $v=\alpha b_i^t-d_i^t$, and
$\gamma=\alpha/(2(1-\alpha))$ (the case $\alpha=1$ is trivial), we get
\[
\bigl\|(1-\alpha)e_i^t+\alpha b_i^t-d_i^t\bigr\|_2^2
\le
\Bigl(1-\frac{\alpha}{2}\Bigr)\|e_i^t\|_2^2
+\frac{2}{\alpha}\|\alpha b_i^t-d_i^t\|_2^2.
\]
Using $\|u-v\|_2^2\le 2\|u\|_2^2+2\|v\|_2^2$, we therefore obtain
\begin{equation}
\mathbb E\|e_i^{t+1}\|_2^2
\le
\Bigl(1-\frac{\alpha}{2}\Bigr)\mathbb E\|e_i^t\|_2^2
+4\alpha\,\mathbb E\|b_i^t\|_2^2
+\frac{4}{\alpha}\,\mathbb E\|d_i^t\|_2^2
+\frac{\alpha^2\sigma^2}{K}.
\label{eq:dl_t7_B_rec_node}
\end{equation}
By $L$-smoothness and Lemma~\ref{lem:dl_Rt_FTC},
\begin{equation}
\frac1M\sum_{i=1}^M \mathbb E\|b_i^t\|_2^2
\le
2L^2R_t+2L^2(K\bar\rho\eta)^2
\le
20L^2(1+\tau^2)(K\bar\rho\eta)^2,
\label{eq:dl_t7_b_bound_explicit}
\end{equation}
and
\begin{equation}
\frac1M\sum_{i=1}^M \mathbb E\|d_i^t\|_2^2
\le
L^2R_t
\le
9L^2(1+\tau^2)(K\bar\rho\eta)^2.
\label{eq:dl_t7_d_bound_explicit}
\end{equation}
Averaging \eqref{eq:dl_t7_B_rec_node} over $i$ and using \eqref{eq:dl_t7_b_bound_explicit} and \eqref{eq:dl_t7_d_bound_explicit} yield
\begin{equation}
B_{t+1}
\le
\Bigl(1-\frac{\alpha}{2}\Bigr)B_t
+\frac{128L^2(1+\tau^2)(K\bar\rho\eta)^2}{\alpha}
+\frac{\alpha^2\sigma^2}{K},
\label{eq:dl_t7_B_rec_avg}
\end{equation}
where we used $\alpha\le 1$ to absorb the term proportional to $\alpha$ into the $\alpha^{-1}$ term.
Let
\[
\delta:=\frac{\alpha}{2},
\qquad
U_B:=\frac{128L^2(1+\tau^2)(K\bar\rho\eta)^2}{\alpha}+\frac{\alpha^2\sigma^2}{K}.
\]
Then \eqref{eq:dl_t7_B_rec_avg} implies
$B_t\le (1-\delta)^{t-1}B_1+U_B/\delta$.
Using $\sqrt{x+y}\le \sqrt x+\sqrt y$ and
$\sum_{t=1}^T(\sqrt{1-\delta})^{t-1}\le 2/\delta$, and recalling that the initialization in the theorem statement implies $B_1\le \sigma^2$, we obtain
\begin{align}
\frac1T\sum_{t=1}^T\sqrt{B_t}
&\le
\sqrt{\frac{U_B}{\delta}}+\frac{2\sqrt{B_1}}{\delta T}
\notag\\
&\le
\frac{16LK\bar\rho\eta\sqrt{1+\tau^2}}{\alpha}
+\frac{\sqrt2\,\sigma\sqrt\alpha}{\sqrt K}
+\frac{4\sigma}{\alpha T}.
\label{eq:dl_t7_use_B_new}
\end{align}

\paragraph{Step 2: an explicit average bound for $\frac1T\sum_{t=1}^T\sqrt{D_t}$.}
Exactly as in the proof of Lemma~\ref{lem:dl_Dt_FTC}, the same $\tau$-window unrolling (with shorter initial windows when $t\le \tau$) and the initialization $D_1=0$ imply that for every $t\ge 1$,
\[
\sqrt{D_t}
\le
\sum_{r=\max\{1,t-\tau\}}^{t-1}
\sqrt{\frac1M\mathbb E\|\mathbf c^{r+1}-\mathbf c^{r}\|_F^2},
\]
where the sum is interpreted as $0$ when $t=1$.
Since $c_i^{s+1}-c_i^s=\alpha(\hat c_i^s-c_i^s)$,
\[
\frac1M\mathbb E\|\mathbf c^{s+1}-\mathbf c^{s}\|_F^2
=
\alpha^2\cdot \frac1M\sum_{i=1}^M\mathbb E\|\hat c_i^s-c_i^s\|_2^2.
\]
Now decompose
\[
\hat c_i^s-c_i^s
=
\underbrace{\big(\nabla f_i(x_i^{s-1})-c_i^s\big)}_{=:a_i^s}
+
\underbrace{\frac1K\sum_{k=1}^K(\nabla f_i(x_i^{s,k-1})-\nabla f_i(x_i^{s-1}))}_{=:b_i^s}
+
\underbrace{\frac1K\sum_{k=1}^K\zeta_i^{s,k-1}}_{=:n_i^s}.
\]
Using $\|u+v+w\|_2^2\le 3(\|u\|_2^2+\|v\|_2^2+\|w\|_2^2)$, Jensen, Assumption~\ref{assump:stoch}, and \eqref{eq:dl_t7_b_bound_explicit}, we obtain
\[
\frac1M\sum_{i=1}^M\mathbb E\|\hat c_i^s-c_i^s\|_2^2
\le
3B_s+60L^2(1+\tau^2)(K\bar\rho\eta)^2+\frac{3\sigma^2}{K}.
\]
Therefore,
\begin{equation}
\sqrt{\frac1M\mathbb E\|\mathbf c^{s+1}-\mathbf c^{s}\|_F^2}
\le
\alpha\Bigl(
\sqrt3\,\sqrt{B_s}
+2\sqrt{15}\,LK\bar\rho\eta\sqrt{1+\tau^2}
+\frac{\sqrt3\,\sigma}{\sqrt K}
\Bigr).
\label{eq:dl_t7_cdiff_sqrt}
\end{equation}
Summing the FTC window bound over $t=1,\ldots,T$ and using that each index appears in at most $\tau$ windows yields
\begin{align}
\frac1T\sum_{t=1}^T\sqrt{D_t}
&\le
\sqrt3\,\tau\alpha\cdot \frac1T\sum_{t=1}^T\sqrt{B_t}
+2\sqrt{15}\,\tau\alpha LK\bar\rho\eta\sqrt{1+\tau^2}
+\frac{\sqrt3\,\tau\alpha\sigma}{\sqrt K}.
\label{eq:dl_t7_D_preavg}
\end{align}
Substituting \eqref{eq:dl_t7_use_B_new} into \eqref{eq:dl_t7_D_preavg} and using $\alpha\le 1$ yields
\begin{align}
\frac1T\sum_{t=1}^T\sqrt{D_t}
\le\;&
(16\sqrt3+2\sqrt{15})\tau LK\bar\rho\eta\sqrt{1+\tau^2}
+\frac{(\sqrt6+\sqrt3)\tau\alpha\sigma}{\sqrt K}
+\frac{4\sqrt3\,\tau\sigma}{T}.
\label{eq:dl_t7_use_D_new}
\end{align}

Substituting \eqref{eq:dl_t7_use_B_new} and \eqref{eq:dl_t7_use_D_new} into \eqref{eq:dl_t7_main_before_BD_new}, we obtain
\begin{align}
\frac1T\sum_{t=1}^T \mathbb E\|\nabla F(\bar x^t)\|^{g,*}
\le\;&
\frac{\Delta_F}{T\eta K}
+\frac{L}{2}K\bar\rho^2\eta
+\Bigl(2+\frac{1}{2K}\Bigr)\bar\rho^2\eta
+(1-\beta_1)^2L^2\bar\rho^2\eta
+\frac{2(1-\beta_1)^2\sigma^2}{\eta}
\notag\\
&\quad
+2\bar\rho(1-\beta_1)\Bigl[
\frac{16LK\bar\rho\eta\sqrt{1+\tau^2}}{\alpha}
+\frac{\sqrt2\,\sigma\sqrt\alpha}{\sqrt K}
+\frac{4\sigma}{\alpha T}
\Bigr]
\notag\\
&\quad
+2\bar\rho\Bigl[
\frac{16LK\bar\rho\eta\sqrt{1+\tau^2}}{\alpha}
+\frac{\sqrt2\,\sigma\sqrt\alpha}{\sqrt K}
+\frac{4\sigma}{\alpha T}
\Bigr]
+6L\bar\rho^2(1-\beta_1)\sqrt{1+\tau^2}\,K\eta
\notag\\
&\quad
+\bar\rho^2L(1-\beta_1)\eta(K-1)
+2\bar\rho(1-\beta_1)\sigma
\notag\\
&\quad
+2\bar\rho(16\sqrt3+2\sqrt{15})\tau LK\bar\rho\eta\sqrt{1+\tau^2}
\notag\\
&\quad
+\frac{2\bar\rho(\sqrt6+\sqrt3)\tau\alpha\sigma}{\sqrt K}
+\frac{8\sqrt3\,\bar\rho\tau\sigma}{T}
\notag\\
&\quad
+2L\bar\rho^2K\eta
+2\tau L\bar\rho^2K\eta.
\label{eq:dl_t7_final_preplug_new}
\end{align}

We now substitute the hyperparameter choice
\[
\eta=\frac{M^{1/4}}{K^{3/4}T^{3/4}},
\qquad
1-\beta_1=\eta,
\qquad
\alpha=1-\beta_2=\sqrt{\frac{MK}{T+K}},
\qquad
\text{and assume } T\ge MK.
\]
Since $T\ge MK$ and $M\ge 1$, we also have $K\le T$ and hence $T+K\le 2T$.
Moreover,
\begin{align}
K\eta &= \frac{M^{1/4}K^{1/4}}{T^{3/4}} \le \frac{1}{\sqrt T},
\qquad
\eta = \frac{1}{(MKT)^{1/4}}\sqrt{\frac{M}{KT}} \le \frac{1}{(MKT)^{1/4}},
\label{eq:dl_t7_param_aux_bounds}\\
\frac{K\eta}{\alpha}
&=
\frac{\sqrt{T+K}}{M^{1/4}K^{1/4}T^{3/4}}
\le
\frac{\sqrt2}{(MKT)^{1/4}},
\qquad
\frac{\sqrt\alpha}{\sqrt K}
=
\frac{M^{1/4}}{K^{1/4}(T+K)^{1/4}}
\le
\frac{\sqrt M}{(MKT)^{1/4}},
\notag\\
\frac{\alpha}{\sqrt K}
&=
\sqrt{\frac{M}{T+K}}
\le
\sqrt{\frac{M}{T}},
\qquad
\frac{1}{\alpha T}
=
\frac{\sqrt{T+K}}{\sqrt{MK}\,T}
\le
\frac{\sqrt2}{\sqrt{MKT}},
\notag\\
K\eta^2
&=
\frac{\sqrt M}{\sqrt K\,T^{3/2}},
\qquad
\frac{K\eta^2}{\alpha}
=
\frac{\sqrt{T+K}}{K T^{3/2}}
\le
\frac{\sqrt2}{KT},
\notag\\
\frac{\eta\sqrt\alpha}{\sqrt K}
&=
\frac{\sqrt M}{K\,T^{3/4}(T+K)^{1/4}}
\le
\frac{\sqrt M}{KT},
\qquad
\frac{\eta}{\alpha T}
=
\frac{\sqrt{T+K}}{M^{1/4}K^{5/4}T^{7/4}}
\le
\frac{\sqrt2}{M^{1/4}K^{5/4}T^{5/4}},
\notag\\
\eta
&=
\frac{M^{1/4}}{K^{3/4}T^{3/4}}
\le \frac{1}{K\sqrt T}.
\notag
\end{align}
Also,
\begin{equation}
\frac{\Delta_F}{T\eta K} = \frac{\Delta_F}{(MKT)^{1/4}}.
\label{eq:dl_t7_gap_term_simp}
\end{equation}
Using \eqref{eq:dl_t7_param_aux_bounds} and \eqref{eq:dl_t7_gap_term_simp} in \eqref{eq:dl_t7_final_preplug_new}, and bounding
$2+\frac{1}{2K}\le \frac52$, we arrive at the explicit estimate
\begin{align}
\frac1T\sum_{t=1}^T \mathbb E\|\nabla F(\bar x^t)\|^{g,*}
\le\; &\left[
\begin{aligned}
&\frac{L\bar\rho^2}{2}
+2L\bar\rho^2
+2\tau L\bar\rho^2 \\
&\quad +(32\sqrt3+4\sqrt{15})\tau L\bar\rho^2\sqrt{1+\tau^2}
+2(\sqrt6+\sqrt3)\tau\bar\rho\sigma\sqrt M
\end{aligned}
\right]\frac1{\sqrt T}
\notag\\
&\quad
+\left[
\begin{aligned}
&\Delta_F
+\frac52\bar\rho^2
+L^2\bar\rho^2
+2\sigma^2 \\
&\quad +32\sqrt2\,L\bar\rho^2\sqrt{1+\tau^2}
+2\sqrt2\,\bar\rho\sigma\sqrt M
\end{aligned}
\right]\frac1{(MKT)^{1/4}}
\notag\\
&\quad
+\frac{6L\bar\rho^2\sqrt{1+\tau^2}\sqrt M}{\sqrt K\,T^{3/2}}
+\frac{L\bar\rho^2\sqrt M}{\sqrt K\,T^{3/2}}
\notag\\
&\quad
+\frac{32\sqrt2\,L\bar\rho^2\sqrt{1+\tau^2}}{KT}
+\frac{2\sqrt2\,\bar\rho\sigma\sqrt M}{KT}
\notag\\
&\quad
+\frac{8\sqrt2\,\bar\rho\sigma}{M^{1/4}K^{5/4}T^{5/4}}
+\frac{2\bar\rho\sigma M^{1/4}}{K^{3/4}T^{3/4}}
\notag\\
&\quad
+\frac{8\sqrt2\,\bar\rho\sigma}{\sqrt{MKT}}
+\frac{8\sqrt3\,\tau\bar\rho\sigma}{T}.
\label{eq:dl_t7_final_explicit_bound_new}
\end{align}
The first line is of order $T^{-1/2}$ and the second is of order $(MKT)^{-1/4}$. The corrected query-wise noise contribution is $2\bar\rho\sigma\eta$, which is explicitly retained through the first term of $\mathcal R_{\mathrm{lo}}$. Its coefficient is independent of $(T,K,M)$, and $\eta\le 1/(K\sqrt T)$ under $MK\le T$. The remaining terms are dominated by the other remainder terms. Hence the stated $C_1$ and $C_2$ are unchanged, and a sufficiently large $C_0$, with the stated parameter dependence, completes the proof.
\end{proof}

\section{PolyGram-RCClip: Degree--Residual Certificates for Spectral Clipping}
\label{app:polygram_clip}

This appendix develops the SVD-free spectral clipping oracle used for matrix blocks with spectral-norm geometry. The key point is that a Gram-polynomial map can be made feasible for the spectral-norm ball and can be assigned a residual certificate compatible with Theorem~\ref{thm:fl_inexact_projection}.

\subsection{Setup and notation}

For a matrix $Z\in\mathbb R^{m\times n}$ and a clipping radius $\tau>0$, define the spectral-norm ball $\mathcal B_\tau:=\{Y\in\mathbb R^{m\times n}:\|Y\|_2\le \tau\}$. If $\|Z\|_2\le\tau$, the oracle returns $Y=Z$ and the residual is zero. We therefore consider the nontrivial case $0<\tau<\|Z\|_2\le s$, where $s$ is any safe upper scale. Set $\alpha:=\tau/s\in(0,1)$ and define the exact clipping transfer in the Gram variable by
\[
\phi_\alpha(\lambda):=\min\left\{1,\frac{\alpha}{\sqrt\lambda}\right\},
\qquad
\phi_\alpha(0):=1,
\qquad \lambda\in[0,1].
\]
For a degree-$d$ polynomial $p_d$, PolyGram-RCClip uses the map
\[
Y_{d,s}=Zp_d(Z^\top Z/s^2) \quad (m\ge n),
\qquad
Y_{d,s}=p_d(ZZ^\top/s^2)Z \quad (m<n).
\]
If $Z=U\operatorname{diag}(\sigma_i)V^\top$ and $\lambda_i=\sigma_i^2/s^2$, both forms preserve the singular vectors and transform singular values as
\[
\widetilde\sigma_i=\sigma_i p_d(\lambda_i).
\]
For any feasible output $Y\in\mathcal B_\tau$, the single-call variational residual is
\[
\operatorname{res}_{\mathcal B_\tau}(Z,Y)
:=
\left[\sup_{U\in\mathcal B_\tau}\langle Z-Y,U-Y\rangle_F\right]_+.
\]
Using duality between the spectral and nuclear norms,
\[
\operatorname{res}_{\mathcal B_\tau}(Z,Y)
=
\left[\tau\|Z-Y\|_* - \langle Z-Y,Y\rangle_F\right]_+.
\]

\subsection{Existence and degree--scale residual bound}
\label{app:polygram_existence}

Let $\mathcal P^{\rm feas}_{d,\alpha}$ be the class of degree-$d$ polynomials satisfying
\[
 p_d(\lambda)\ge0,
 \qquad
 \sqrt\lambda\,p_d(\lambda)\le\alpha,
 \qquad \forall \lambda\in[0,1],
\]
and define the optimal feasible fitting radius
\[
\delta_d^\star(\alpha)
:=
\inf_{p\in\mathcal P^{\rm feas}_{d,\alpha}}
\|p-\phi_\alpha\|_{\infty,[0,1]}.
\]

\begin{proposition}[Residual controlled by feasible fitting radius]
\label{prop:polygram_residual_bound}
Let $p_d\in\mathcal P^{\rm feas}_{d,\alpha}$ and define $Y_{d,s}$ by the Gram-polynomial map above. If $\|p_d-\phi_\alpha\|_{\infty,[0,1]}\le\delta$ with $0\le\delta\le1$, then $Y_{d,s}\in\mathcal B_\tau$ and
\begin{align}
\operatorname{res}_{\mathcal B_\tau}(Z,Y_{d,s})
\le
\delta\left(\tau\|Z\|_*+2\|Z\|_F^2\right).
\label{eq:polygram_residual_bound}
\end{align}
Consequently, with $q=\min\{m,n\}$,
\begin{align}
\operatorname{res}_{\mathcal B_\tau}(Z,Y_{d,s})
\le
\delta\left(\tau\sqrt q\|Z\|_F+2\|Z\|_F^2\right).
\label{eq:polygram_residual_bound_frob}
\end{align}
\end{proposition}

\begin{proof}
Let $q_i=p_d(\lambda_i)$. Feasibility gives $\sigma_iq_i=s\sqrt{\lambda_i}p_d(\lambda_i)\le s\alpha=\tau$, so $Y_{d,s}\in\mathcal B_\tau$. The support-function form of the residual gives
\[
\|Z-Y_{d,s}\|_*=
\sum_i \sigma_i|1-q_i|,
\qquad
\langle Z-Y_{d,s},Y_{d,s}\rangle_F=
\sum_i \sigma_i^2 q_i(1-q_i).
\]
If $\sigma_i\le\tau$, then $\phi_\alpha(\lambda_i)=1$, so $|1-q_i|\le\delta$. Since $q_i\ge0$ and $\delta\le1$, also $q_i\le2$, and the contribution of this index is at most $\delta\tau\sigma_i+2\delta\sigma_i^2$. If $\sigma_i>\tau$, then $\phi_\alpha(\lambda_i)=\tau/\sigma_i$. Feasibility implies $q_i\le\tau/\sigma_i$, so $q_i=\tau/\sigma_i-e_i$ for some $e_i\in[0,\delta]$. Hence $\tau-\sigma_iq_i=\sigma_ie_i$ and $1-q_i\ge0$, so this index contributes at most $\delta\sigma_i^2$. Summing over $i$ proves \eqref{eq:polygram_residual_bound}, and \eqref{eq:polygram_residual_bound_frob} follows from $\|Z\|_*\le\sqrt q\|Z\|_F$.
\end{proof}

\begin{lemma}[Jackson direct theorem]
\label{lem:jackson_direct}
For a continuous function $f$ on $[0,1]$, define the best degree-$d$ uniform approximation error and the ordinary modulus of continuity by
\[
E_d(f;[0,1])
:=
\inf_{\deg q\le d}\|f-q\|_{\infty,[0,1]},
\qquad
\omega(f,t;[0,1])
:=
\sup_{\substack{x,y\in[0,1]\\ |x-y|\le t}}
|f(x)-f(y)|.
\]
There exists a universal constant $C_J^{\rm Jac}>0$ such that, for every $d\ge1$,
\[
E_d(f;[0,1])
\le
C_J^{\rm Jac}\,\omega(f,1/d;[0,1]).
\]
Consequently, if $f$ is $L_f$-Lipschitz on $[0,1]$, then $E_d(f;[0,1])\le C_J^{\rm Jac}L_f/d$. This is the classical Jackson direct theorem for algebraic polynomial approximation~\citep{jackson1912approximation,devore1993constructive,ditzian1987moduli}.
\end{lemma}

\begin{theorem}[Degree--scale residual certificate]
\label{thm:polygram_degree_residual}
There exists a universal constant $C_J>0$ such that the following holds. In the nontrivial case $0<\tau<\|Z\|_2\le s$, let $\alpha=\tau/s$ and
\[
 d_0(\alpha):=\left\lceil \max\left\{1,\frac{C_J}{\alpha^3}\right\}\right\rceil.
\]
For every $d\ge d_0(\alpha)$, there exists a degree-$d$ polynomial $p_d$ satisfying $0\le p_d(\lambda)\le\phi_\alpha(\lambda)$ for all $\lambda\in[0,1]$ and
\[
\|p_d-\phi_\alpha\|_{\infty,[0,1]}
\le
\frac{C_J}{\alpha^2d}
=
\frac{C_Js^2}{\tau^2d}.
\]
Consequently, there exists a feasible degree-$d$ Gram-polynomial clipping map satisfying
\begin{align}
\operatorname{res}_{\mathcal B_\tau}(Z,Y_{d,s})
\le
\frac{C_Js^2}{\tau^2d}
\left(\tau\|Z\|_*+2\|Z\|_F^2\right).
\label{eq:polygram_degree_residual}
\end{align}
\end{theorem}

\begin{proof}
The clipping transfer $\phi_\alpha$ is constant on $[0,\alpha^2]$ and has derivative $-\alpha/(2\lambda^{3/2})$ on $(\alpha^2,1]$. Hence $\phi_\alpha$ is Lipschitz on $[0,1]$ with $\operatorname{Lip}(\phi_\alpha)\le1/(2\alpha^2)$. Applying Lemma~\ref{lem:jackson_direct} with $f=\phi_\alpha$, and absorbing universal numerical factors into $C_J$, gives a degree-$d$ polynomial $q_d$ such that
\[
\|q_d-\phi_\alpha\|_{\infty,[0,1]}
\le
\frac{C_J}{2\alpha^2d}
=:
\varepsilon_d.
\]
Define the shifted polynomial $p_d:=q_d-\varepsilon_d$. Then $p_d\le\phi_\alpha$ pointwise and
\[
\|p_d-\phi_\alpha\|_{\infty,[0,1]}
\le
2\varepsilon_d
=
\frac{C_J}{\alpha^2d}.
\]
Since $\phi_\alpha(\lambda)\ge\alpha$ for all $\lambda\in[0,1]$, the condition $d\ge d_0(\alpha)$ implies $2\varepsilon_d\le\alpha$ and therefore $p_d(\lambda)\ge\phi_\alpha(\lambda)-2\varepsilon_d\ge0$. Thus $0\le p_d\le\phi_\alpha$, which also implies $\sqrt\lambda p_d(\lambda)\le\sqrt\lambda\phi_\alpha(\lambda)\le\alpha$. Proposition~\ref{prop:polygram_residual_bound} then yields \eqref{eq:polygram_degree_residual}.
\end{proof}

\paragraph{Cheap scales.}
The bound is valid for every $s\ge\|Z\|_2$. Thus a tighter scale improves the certificate through the factor $s^2/\tau^2$. The deterministic choice $s=\|Z\|_F$ is always valid and gives
\[
\operatorname{res}_{\mathcal B_\tau}(Z,Y_{d,\|Z\|_F})
\le
\frac{C_J\|Z\|_F^2}{\tau^2d}
\left(\tau\|Z\|_*+2\|Z\|_F^2\right).
\]
Using $\|Z\|_*\le\sqrt q\|Z\|_F$ and $\tau<\|Z\|_2\le\|Z\|_F$ yields the Frobenius-only certificate
\[
\operatorname{res}_{\mathcal B_\tau}(Z,Y_{d,\|Z\|_F})
\le
\frac{C_J(\sqrt q+2)}{d}\frac{\|Z\|_F^4}{\tau^2}.
\]

\subsection{Coefficient design as a semi-infinite LP}
\label{app:polygram_lp}

After scaling by $s$, all singular values enter through $\lambda=\sigma^2/s^2\in[0,1]$, and the exact clipping transfer is determined only by the scalar $\alpha=\tau/s$. Thus, for a fixed degree $d$, coefficient design is a one-dimensional family indexed by $\alpha$, rather than by $Z,\tau,s$ separately. Write $p_{d,\alpha}(\lambda)=\sum_{j=0}^d c_j(\alpha)T_j(2\lambda-1)$ in a Chebyshev basis. The continuous coefficient-design problem is
\begingroup\small
\begin{equation}
\begin{alignedat}{3}
\min_{c_0,\ldots,c_d,\delta}\quad &\delta\qquad\\[-0.2em]
\text{s.t.}\quad
& p_{d,\alpha}(\lambda)\ge0, &\qquad& \sqrt\lambda\,p_{d,\alpha}(\lambda)\le\alpha, &\qquad& \forall\lambda\in[0,1],\\[-0.1em]
& p_{d,\alpha}(\lambda)-\phi_\alpha(\lambda)\le\delta, && \phi_\alpha(\lambda)-p_{d,\alpha}(\lambda)\le\delta, && \forall\lambda\in[0,1],\\[-0.1em]
& \delta\ge0.&&&&
\end{alignedat}
\label{eq:polygram_feasible_lp}
\end{equation}
\endgroup
The value of this semi-infinite LP is $\delta_d^\star(\alpha)$, and Theorem~\ref{thm:polygram_degree_residual} gives a degree--scale upper bound on this value.

\subsection{Table-based PolyGram-RCClip oracle}
\label{app:polygram_table_oracle}

Because the LP in Appendix~\ref{app:polygram_lp} is indexed only by $\alpha$, the runtime oracle does not solve an LP. We precompute a table
$\mathcal T=\{(\alpha_j,c_0^{(j)},\ldots,c_d^{(j)})\}_{j=1}^J$.
In the experiments, $d=20$ and $J=1000$, so the table contains only $J(d+1)=21{,}000$ scalar coefficients.

Algorithm~\ref{alg:polygram_rcclip} gives the practical fixed-degree runtime oracle used in the experiments. The safe scale is chosen deterministically. Let $G_0=Z^\top Z$ if $m\ge n$ and $G_0=ZZ^\top$ otherwise. We use $s_{\rm safe}:=\min\{\|Z\|_F,\sqrt{\|Z\|_1\|Z\|_\infty},\sqrt{\|G_0\|_\infty}\}$ and $s:=1.0001\,s_{\rm safe}$, where $\|Z\|_1$ is the maximum absolute column sum and $\|Z\|_\infty$ is the maximum absolute row sum. Each term in the minimum upper bounds $\|Z\|_2$, hence $s\ge\|Z\|_2$. The identity-return check uses the strict bound $s_{\rm safe}\le\tau$, in which case $Z\in\mathcal B_\tau$.

The scale bound alone does not certify the entire oracle. The residual theorem additionally requires continuous-interval feasibility and fitting-error bounds for the selected polynomial. Floor lookup and low-$\alpha$ rescaling preserve feasibility in exact arithmetic when the table entries are feasible on the full interval, but their approximation errors must be included in the residual budget. A discretized LP and floating-point evaluation require additional verification for a rigorous certificate. We do not claim that the experimental fixed choice $d=20$ satisfies the accumulated budget of Theorem~\ref{thm:fl_inexact_projection} for arbitrary horizons.
\begin{algorithm}[H]
   \caption{Table-based PolyGram-RCClip for one matrix block}
   \label{alg:polygram_rcclip}
\small
\begin{algorithmic}[1]
   \STATE {\bfseries Input:} matrix $Z\in\mathbb R^{m\times n}$, radius $\tau$, safe scale $s\ge\|Z\|_2$, coefficient table $\mathcal T=\{(\alpha_j,c^{(j)})\}$, degree $d$.
   \IF{$s\le\tau$} \STATE \textbf{return} $Z$ \ENDIF
   \STATE Set $\alpha_{\rm clip}\leftarrow \tau/s$ and let $\alpha_{\min}$ be the smallest grid value in $\mathcal T$.
   \IF{$\alpha_{\rm clip}<\alpha_{\min}$}
      \STATE Load coefficients for $\alpha_{\min}$ and set $r\leftarrow \alpha_{\rm clip}/\alpha_{\min}$.
   \ELSE
      \STATE Choose the largest grid value $\tilde\alpha\le\alpha_{\rm clip}$, load its coefficients, and set $r\leftarrow1$.
   \ENDIF
   \STATE Write the selected coefficients as $c_0,\ldots,c_d$ in the basis $T_j(2\lambda-1)$.
   \IF{$m\ge n$}
      \STATE Form $\widetilde G\leftarrow 2Z^\top Z/s^2-I_n$ and set $V_0\leftarrow Z$, $Y\leftarrow c_0V_0$.
      \IF{$d\ge1$}
         \STATE $V_1\leftarrow V_0\widetilde G$, $Y\leftarrow Y+c_1V_1$.
         \FOR{$j=1,\ldots,d-1$}
            \STATE $V_{j+1}\leftarrow 2V_j\widetilde G-V_{j-1}$; \quad $Y\leftarrow Y+c_{j+1}V_{j+1}$.
         \ENDFOR
      \ENDIF
   \ELSE
      \STATE Form $\widetilde H\leftarrow 2ZZ^\top/s^2-I_m$ and set $V_0\leftarrow Z$, $Y\leftarrow c_0V_0$.
      \IF{$d\ge1$}
         \STATE $V_1\leftarrow \widetilde H V_0$, $Y\leftarrow Y+c_1V_1$.
         \FOR{$j=1,\ldots,d-1$}
            \STATE $V_{j+1}\leftarrow 2\widetilde H V_j-V_{j-1}$; \quad $Y\leftarrow Y+c_{j+1}V_{j+1}$.
         \ENDFOR
      \ENDIF
   \ENDIF
   \STATE \textbf{return} $rY$.
\end{algorithmic}
\end{algorithm}
After the smaller Gram matrix is formed, Algorithm~\ref{alg:polygram_rcclip} uses exactly $d$ sequential Gram-action multiplications to evaluate a degree-$d$ transfer. Thus the table degree is both the approximation degree in Theorem~\ref{thm:polygram_degree_residual} and the main compute knob of the oracle.

\section{Additional preliminaries: uSCG template}
\label{app:uscg_template}
For completeness, we recall the Unconstrained Stochastic Conditional Gradient (uSCG) template~\citep{pethick2025training}, which provides the LMO-based abstraction used in Section~\ref{subsec:ball} to motivate norm-ball update geometries.

\begin{algorithm}[H]
   \caption{Unconstrained Stochastic Conditional Gradient (uSCG)~\citep{pethick2025training}}
   \label{alg:uscg}
\begin{algorithmic}
   \STATE {\bfseries Input:} Initial iterate $x^1 \in \mathbb{R}^d$, momentum factor $\beta \in [0,1]$, constraint set $\mathcal{D}$.
   \STATE Initialize momentum buffer $m^{-1}=\mathbf{0}$.
   \FOR{$t=1$ {\bfseries to} $T$}
   \STATE Compute stochastic gradient $g^t=\nabla f(x^t,\xi_t)$.
   \STATE Update momentum:
   \[
   m^t=(1-\beta)m^{t-1}+\beta g^t.
   \]
   \STATE Compute update direction via Linear Minimization Oracle (LMO):
   \[
   \Delta^t\in\argmin_{\Delta\in\mathcal D}\langle m^t,\Delta\rangle.
   \]
   \STATE Update parameters: $x^{t+1}=x^t+\Delta^t$.
   \ENDFOR
\end{algorithmic}
\end{algorithm}

\section{Experimental details and additional results}
\label{app:experiments}

\subsection{Model architecture}
\label{app:exp_model}
We use a CIFAR-style ResNet-18 for $32\times32$ inputs. The stem is a bias-free $3\times3$ convolution from 3 to 64 channels, followed by GroupNorm with 8 groups and a GELU activation; there is no ImageNet-style large stem or max-pooling layer. Each residual block contains two bias-free $3\times3$ convolutions, GroupNorm after each convolution, GELU after the first normalization, and GELU after the residual addition. Downsampling blocks use a bias-free $1\times1$ convolution followed by GroupNorm. The four stages have widths $64,128,256,512$, with two basic blocks per stage and strides $1,2,2,2$. The network ends with adaptive average pooling and a bias-free linear classifier $\mathbb R^{512}\to\mathbb R^{10}$. All GroupNorm layers keep the trainable scale parameter but remove the bias parameter. In the main NBCLT experiments, we use Scion initialization.

\subsection{Baselines and tuning}
\label{app:exp_baselines}
For each Dirichlet parameter, we tune each baseline over its own hyperparameter grid and select the best run by average test accuracy over the last 100 communication rounds. All baselines and NBCLT variants use the shared weight-decay grid $\mathrm{wd}\in\{0,10^{-3}\}$. We use $\mathrm{lr}$ for client-side local learning rates, and $\eta$ or $\eta_g$ for server-side learning rates following the notation of the corresponding methods. Momentum parameters are denoted by $\beta$, or by $(\beta_1,\beta_2)$ when the method uses two such parameters. For FedMuon~\citep{takezawa2025fedmuon}, Muon and SGD layers have separate stepsizes, and $\alpha$ denotes the momentum parameter; our implementation uses exact orthogonalization for LMO-based layers rather than a Newton--Schulz approximation.

\begin{table}[H]
\centering
\small
\setlength{\tabcolsep}{4pt}
\renewcommand{\arraystretch}{1.12}
\caption{Hyperparameter grids for CIFAR-10 FL experiments. All methods additionally sweep $\mathrm{wd}\in\{0,10^{-3}\}$.}
\label{tab:cifar10_tuning_grids}
\begin{tabular}{p{0.16\linewidth}p{0.78\linewidth}}
\toprule
Method & Grid \\
\midrule
FedAvg &
$\mathrm{lr}\in\{10^{-3},3\!\times\!10^{-3},10^{-2},3\!\times\!10^{-2},10^{-1}\}$; client momentum $\in\{0.9,0.95\}$. \\
SCAFFOLD-M &
local $\mathrm{lr}\in\{10^{-3},10^{-2},3\!\times\!10^{-2},10^{-1}\}$; server $\eta_g\in\{1,0.1\}$; $\beta\in\{0.9,0.95\}$. \\
FedAMS &
local $\mathrm{lr}\in\{10^{-3},10^{-2},3\!\times\!10^{-2},10^{-1}\}$; server $\eta\in\{1,0.1\}$; $\beta_1=0.9$; $\beta_2\in\{0.99,0.999\}$; $\epsilon=10^{-6}$. \\
FedMuon &
Muon stepsize $\in\{10^{-3},10^{-4}\}$; SGD stepsize $\in\{10^{-2},10^{-1}\}$; $\alpha\in\{0.05,0.1\}$. \\
PAdaMFed &
local $\mathrm{lr}\in\{10^{-2},3\!\times\!10^{-2},10^{-1}\}$; server $\eta_g\in\{1,0.1\}$; $\beta\in\{0.9,0.95\}$. \\
NBCLT tuned &
$\eta\in\{0.003,0.01\}$; $\beta_1\in\{0.9,0.95\}$; $\beta_2\in\{0.75,0.9,0.95\}$. \\
\bottomrule
\end{tabular}
\end{table}

For NBCLT-PF, we do not tune $\eta,\beta_1,\beta_2$; we use the theorem-suggested values reported in Section~\ref{sec:experiments} and only apply the shared weight-decay grid.

\subsection{PolyGram-RCClip implementation}
\label{app:exp_polygram_impl}
The exact-SVD implementation projects each spectral matrix block by clipping singular values to the block radius. PolyGram-RCClip replaces this training-time SVD by the table-based Gram-polynomial oracle in Appendix~\ref{app:polygram_clip}. For a matrix block $Z$ and radius $\tau$, the reported PolyGram-RCClip curves use the certified safe-min scale from Appendix~\ref{app:polygram_table_oracle}: with $G_0=Z^\top Z$ if $m\ge n$ and $G_0=ZZ^\top$ otherwise,
\[
 s=1.0001\cdot
 \min\left\{\|Z\|_F,\sqrt{\|Z\|_1\|Z\|_\infty},\sqrt{\|G_0\|_\infty}\right\}.
\]
The clipping ratio $\alpha_{\rm clip}=\tau/s$ is distinct from the Dirichlet parameter $\alpha_{\rm dir}$. We use a precomputed Chebyshev coefficient table over $\alpha_{\rm clip}\in[10^{-3},0.999]$ with degree $d=20$ and 1000 grid points; training-time calls use floor lookup and evaluate either $Zp_d(Z^\top Z/s^2)$ or $p_d(ZZ^\top/s^2)Z$, choosing the smaller Gram matrix. For the hardest split $\alpha_{\rm dir}=0.1$, we also ran additional experiments with different polynomial degrees $d$ while keeping the same certified scale rule.

\subsection{Norm geometry and block scales}
\label{app:exp_norm_geometry}
NBCLT uses block-wise norm-ball projection with layer-dependent radii $\tau_\ell=\eta r_\ell$. In the CIFAR-10 ResNet-18 experiments, $3\times3$ convolutions and $1\times1$ downsampling convolutions use spectral-norm clipping, while GroupNorm weights and the final classifier head use $\ell_2$ clipping. For a convolutional kernel $\Delta\in\mathbb R^{C_{\rm out}\times C_{\rm in}\times k\times k}$, we unfold it to $Z\in\mathbb R^{C_{\rm out}\times (C_{\rm in}k^2)}$ and use
\[
 r_{\rm conv,spec}=\sqrt{\frac{C_{\rm out}}{C_{\rm in}k^2}},
 \qquad
 \tau_{\rm conv}=\eta\sqrt{\frac{C_{\rm out}}{C_{\rm in}k^2}}.
\]
For a linear head $W\in\mathbb R^{d_{\rm out}\times d_{\rm in}}$ with $\ell_2$ geometry, we use $r_{{\rm linear},\ell_2}=1/\sqrt{d_{\rm out}d_{\rm in}}$; for a one-dimensional normalization weight of dimension $d$, we use $r_{{\rm norm},\ell_2}=1/\sqrt d$.

\subsection{Additional curves and projector ablations}
\label{app:exp_additional_results}
Figure~\ref{fig:cifar10-fl-full9} expands the main comparison by showing the five tuned baseline curves together with the four NBCLT-FL variants. The degree-$d$ projector ablation for the hardest split $\alpha_{\rm dir}=0.1$ is shown in the main text as Figure~\ref{fig:cifar10-degree-ablation}; it uses the same certified safe-min scale rule described in Appendix~\ref{app:polygram_table_oracle}.

\begin{figure}[H]
\centering
\includegraphics[width=\linewidth]{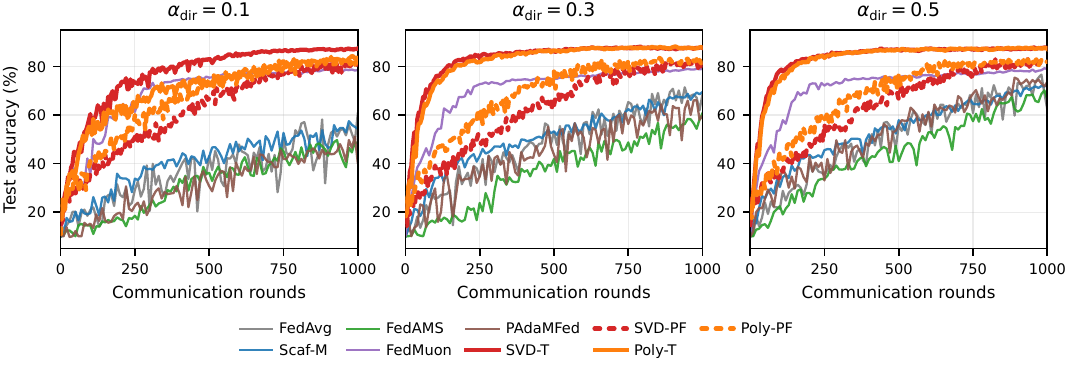}
\caption{Full CIFAR-10 federated-learning comparison under Dirichlet label-skew partitions.
Each panel shows all five tuned baselines--FedAvg, SCAFFOLD-M, FedAMS, FedMuon, and PAdaMFed--together with the four NBCLT-FL variants.
The main text reports the compressed comparison in which these five baselines are represented by the strongest tuned baseline curve.}
\label{fig:cifar10-fl-full9}
\end{figure}

\section{Limitations, Compute, and Societal Considerations}
\label{app:limitations}

\paragraph{Scale and comparison to Muon.}
The NBCLT update rule and the PolyGram-RCClip projector are not specific to federated aggregation and could in principle be instantiated as a centralized optimizer for large-scale neural-network training. This makes the method conceptually close in scope to Muon-type optimizers, while differing in the mathematical primitive: NBCLT clips tentative updates by projection, whereas Muon-type methods select LMO/orthogonalized directions. Our FL experiments include FedMuon as a federated Muon-type baseline and show favorable performance in that setting, but they do not replace a direct large-scale centralized comparison with Muon on language-model pretraining or other large neural-network benchmarks. Such comparisons are an important future direction.

\paragraph{Empirical scope.}
The experiments evaluate CIFAR-10 federated learning with Dirichlet label-skew partitions and a CIFAR-style ResNet-18. The best run for each method is selected by average test accuracy over the final 100 communication rounds. Appendix~\ref{app:experiments} provides the training setup, hyperparameter grids, and implementation details.

\paragraph{Choosing geometry at scale.}
NBCLT allows different update geometries for different parameter blocks, but the best norm choice for very large models is not yet understood. In particular, it remains open whether clipping-based flexible geometry can outperform LMO-based Muon geometry on large-scale language-model training, and which layerwise or blockwise norm choices are most effective in that regime. The current theory treats the chosen norm as fixed and analyzes the resulting local-training dynamics; automatically selecting or adapting the geometry is beyond the scope of this work.

\paragraph{Projector certification.}
The PolyGram-RCClip certificate requires a safe scale $s\ge\|Z\|_2$, a feasible polynomial with a certified fitting error, and an accumulated residual budget. The experiments use the deterministic safe-min scale in Appendix~\ref{app:polygram_table_oracle}, but the fixed-degree table implementation is a practical approximation; we do not certify its accumulated residual for arbitrary horizons. Heuristic scales without an upper-bound guarantee, table discretization, and floating-point effects also require additional verification.

\paragraph{Compute and assets.}
The CIFAR-10 experiments were run on a workstation with four NVIDIA RTX 4090 GPUs. We use the public CIFAR-10 dataset and standard open-source deep-learning libraries; no new dataset or human-subject data are introduced. NBCLT is a general optimization method for distributed training. It may reduce communication and make distributed training more efficient, but it does not by itself provide privacy, security, fairness, or robustness guarantees. In particular, federated learning should not be interpreted as formal privacy protection unless combined with mechanisms such as differential privacy or secure aggregation.

\end{document}